\documentclass[preprint,3p,12pt]{elsarticle}
\usepackage{lineno,hyperref}
\usepackage{amsmath,amssymb,epsfig}
\usepackage[utf8]{inputenc}
\usepackage[english]{babel}
\usepackage{mathtools}
\usepackage{helvet,enumerate}         
\usepackage{courier}        
\usepackage{type1cm}
\usepackage{csquotes}
\usepackage{makeidx}        
\usepackage{graphicx}
\usepackage{multicol}
\usepackage{appendix}
\usepackage{stix}
\usepackage[bottom]{footmisc}
\usepackage[dvipsnames]{xcolor}
\usepackage{multicol}
\usepackage[bottom]{footmisc}     
\usepackage{hyperref} 
\usepackage{tikz}
\usepackage{float}  
\usepackage{helvet,overpic} 
\usepackage{subfig}
\usepackage{soul}
\usepackage{amsthm}
\newtheorem{thm}{Theorem}
\newtheorem{lem}{Lemma}
\newtheorem{rmk}{Remark}

\newtheorem{limit}{Limitation}
\newtheorem{Counter}{Counter example}

\newcommand{\rectangleblack}{\scalebox{0.5}{\ensuremath \hrectangleblack}}

\let\mybibitem\bibitem
\renewcommand{\bibitem}[1]{%
  \ifstrequal{#1}{nature}
    {\color{blue}\mybibitem{#1}}
    {\color{black}\mybibitem{#1}}%
}
\usepackage{etoolbox}

\begin{document}

\begin{frontmatter}

\title{Singularity Distance Computations for 3-RPR Manipulators Using Extrinsic Metrics}

\author{Aditya Kapilavai and Georg Nawratil}
\address{Institute of Discrete Mathematics and Geometry \& 
Center for Geometry and Computational Design, TU Wien\\
Wiedner Hauptstrasse 8--10, Vienna 1040, Austria}


\cortext[mycorrespondingauthor]{Corresponding author: Aditya Kapilavai}

\ead{\{akapilavai,nawratil\}@geometrie.tuwien.ac.at}
 
\begin{abstract}
{
It is well-known that parallel manipulators are prone to singularities. However, there is still a lack of distance evaluation functions, referred to as metrics, for computing the distance between two 3-RPR configurations. 
The proposed extrinsic metrics take the combinatorial structure of the manipulator into account as well as different design options. Utilizing these extrinsic metrics, we formulate constrained optimization problems. These problems are aimed at identifying the closest singular configurations for a given non-singular configuration.
The solution to the associated system of polynomial equations relies on algorithms from numerical algebraic geometry implemented in the software package \texttt{Bertini}. 
Furthermore, we have developed a computational pipeline for determining the distance to singularity during a one-parametric motion of the manipulator. 
To facilitate these computations for the user, an open-source interface is developed between software packages \texttt{Maple}, \texttt{Bertini}, and \texttt{Paramotopy}. The effectiveness of the presented approach is demonstrated based on numerical examples and compared with existing indices evaluating the singularity closeness.}

\end{abstract}

\begin{keyword}
3-RPR planar parallel manipulator, singularity distance, extrinsic metric,  numerical homotopy continuation, constrained optimization, kinematic performance indices
\end{keyword}
\end{frontmatter}

\section{Introduction}

A 3-RPR manipulator (cf. Fig.~\ref{3rpr}) is a three Degree-of-Freedom (DoF) planar parallel manipulator with two translational DoFs and one rotational DoF. The base and platform are connected by three legs, where each leg consists of two passive revolute (R) joints connected by an actuated prismatic (P) joint  (cf.\ Fig.~\ref{3rpr}).  
Note that we will not consider any motion range limitations on the P and R joints.

Let $\mathbf{k}_i$ denote the 
coordinate vectors of the base anchor points with respect 
to the fixed frame, having coordinates $(x_i,y_i)^T$
for $i=1,2 ,3$. 
The coordinate vectors of the platform anchor points with respect to the moving frame are  $\mathbf{p}_{j}$, having coordinates $(x_j,y_j)^T$ 
for $j=4,5,6$. Their coordinate vectors with respect to the fixed frame can be computed as follows:
\begin{equation}
\mathbf{k}_{j} =
\mathbf{R} \mathbf{p}_j+ \mathbf{t},
\label{initial}
\end{equation}
\noindent where $\mathbf{R} $ is a $2\times2$ rotation matrix and $\mathbf{t}$ is the translation vector. 

Without loss of generality, we can assume that $\mathbf{k}_1$ is the origin of the fixed frame ($\Rightarrow x_1=y_1=0$), and  $\mathbf{k}_2$ is located on the $x$-axis of the fixed frame ($\Rightarrow y_2=0$). The same assumptions can be made for the moving frame, which implies $x_4=y_4=y_5=0$. 
The remaining six coordinates $x_2,x_3,y_3,x_5,x_6,y_6$ can be seen as the design parameters of the base and platform, respectively. 

In the remainder of the paper, we represent a 3-RPR configuration as $\mathbf{K}$, where $\mathbf{K}:= (\mathbf{k}_1, \ldots, \mathbf{k}_6)$. Hence, $\mathbb{R}^{12}$ denotes the 12-dimensional configuration space.

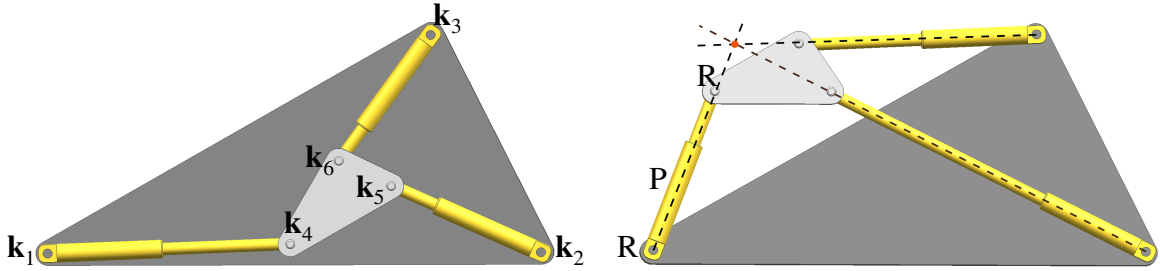
\begin{figure}[H]
\centering
\begin{overpic}[width=70mm]{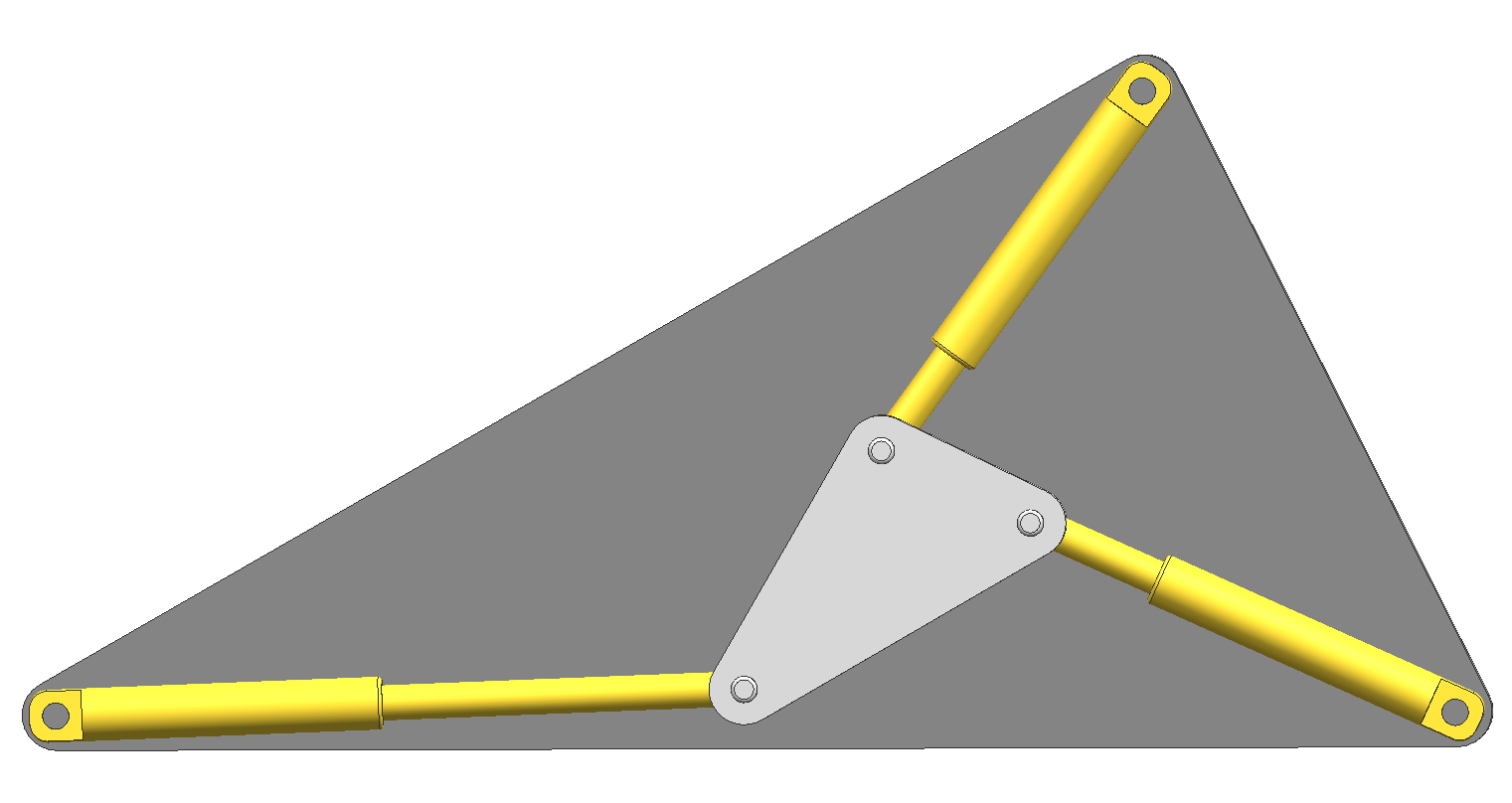}
\put (48,9.5){$\mathbf{k}_{4}$}
\put (-4.1,5){$\mathbf{k}_{1}$}
\put (99,5){$\mathbf{k}_{2}$}
\put (61.5,16.5){$\mathbf{k}_{5}$}
\put (52,22){$\mathbf{k}_{6}$}
\put (76,49){$\mathbf{k}_{3}$}

\end {overpic}
\qquad
 \begin{overpic}[width=70mm]{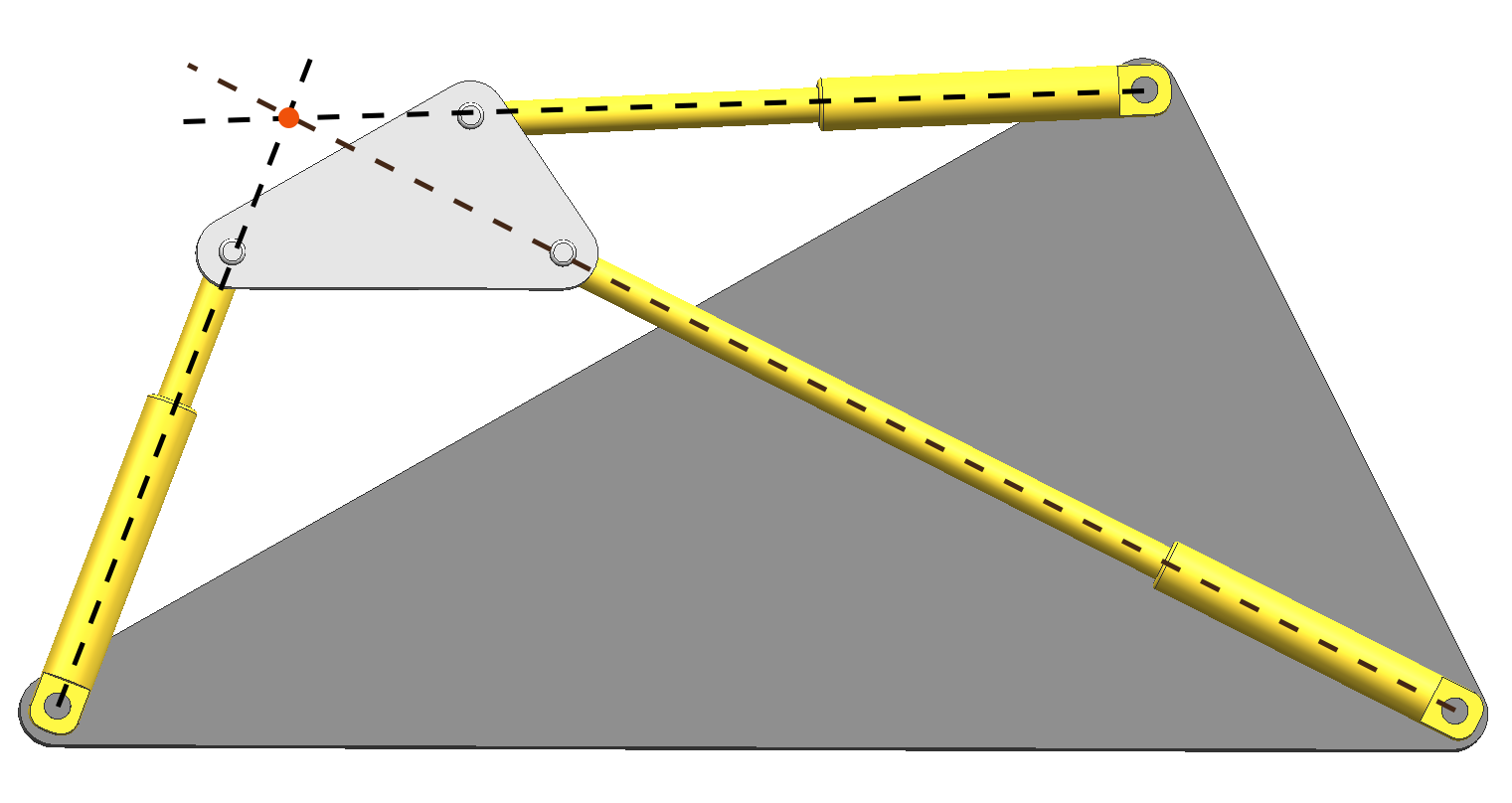}
 \put (-3,5){R}
  \put (3,18){P}
   \put (12,37.5){R}
  \end {overpic}
\caption{(left) 3-RPR planar parallel manipulator with base anchor points $\mathbf{k}_1,\mathbf{k}_2,\mathbf{k}_3$ and platform anchor points 
$\mathbf{k}_4,\mathbf{k}_5,\mathbf{k}_6$, where $\mathbf{k}_{i}$ 
denotes the coordinate vector of the $i$th attachment with respect to the fixed frame. (right) Singular configuration of a 3-RPR manipulator.}
\label{3rpr}
\end{figure}

The 3-RPR manipulator is the best-studied planar parallel mechanism, and their singularities are well understood from various points of view (e.g.\ \cite{merlet2006parallel}). 
It is well known that 3-RPR manipulators 
have at least one uncontrollable instantaneous DoF in singular (also known as shaky or infinitesimal flexible) configurations. Moreover in their neighborhood, minor geometry variations caused by, e.g., backlash in the joints and uncertainties in the actuation of P-joints, respectively,  can significantly affect the resulting manipulator pose.  Another phenomenon that appears close to singularities is that the actuator forces can become very large, which may result in a breakdown of the manipulator \cite{hubert}. 
Therefore, singularities and their vicinities should be avoided, which reasons the interest of the kinematic/robotic community in evaluating the singularity closeness. Even though a lot of performance indices are given in the literature, which can also be seen as closeness indices to singularities, there is still a lack of such distance metrics, where one can distinguish the following two kinds: 
\begin{enumerate}
    \item 
    Intrinsic metrics: The distance to the singularity is measured based on the inner metric of the manipulator, which is determined by the distances $\ell_{ij}$ between $\mathbf{k}_i$ and $\mathbf{k}_j$  with $i<j$ and $(i,j)\in\{(1, 2), (2, 3), (1, 3), (1, 4), (2, 5), (3, 6), (4, 5), (5, 6), (4, 6)\}$. Singularity distance computations for 3-RPR manipulators using intrinsic metrics is presented by the authors in \cite{intrinsic}.
    \item
    Extrinsic metrics: The distance to the singularity is measured based on the metric of the embedding space of the manipulator, which, in our case, is the Euclidean plane. $\mathbb{R}^2$. 
\end{enumerate}

\subsection{Review on singularity closeness indices}\label{sec:review}

Since the early days of robotics, numerous {\it kinematic performance indices}\footnote{A kinematic performance index of a robotic mechanical system converts the capability of the system to transmit motions (at the level of velocities) into a scalar (cf.\cite[page 171]{angeles})} (KPI) have been defined.

By denoting the part of the robotic system, for which manipulation of the mechanical device is built as the end-effector (EE),  one can distinguish the following two types of KPIs: If the index depends on the EE 
in some sense (e.g.\ relative position to the mechanical device, shape, and size of the EE, \ldots) then 
the index is called an {\it EE-dependent KPI}; otherwise it is an {\it EE-independent KPI}.

As a 3-RPR robot cannot transmit the motion standstill of the P-joints to the moving platform within a singular configuration, kinematic performance indices are also sometimes regarded as {\it closeness indices} to singularities. 
Some papers on parallel manipulators pointing out this property explicitly; e.g. \cite{bu2016closeness,bu2017closeness,hartley2001invariant, huang2014force, laryushkin2015estimation,liu2012new,mao2013new,nawratil2009new,takeda1996kinematic, Wolf, Chao, ebrahimi2007actuation,Gomez}.

This interpretation only makes sense for EE-independent KPIs, as the singularity 
variety is solely determined by the geometry of the 3-RPR robot. 
Therefore, the indices \cite {bu2016closeness, bu2017closeness, huang2014force, liu2012new, Chao}, which depend on the choice of a "{\it point of interest}" and therefore on the EE, are not appropriate to evaluate the closeness to singularities. For the same reason the most prominent class of KPIs, namely the {\it condition number indices} (cf.\ \cite{nawratil2009new}), are not suited to indicate the vicinity to singularities.

A further problem of many indices is that they cannot handle the dimensional inhomogeneity which arises for parallel manipulators with rotational and translational dofs. These mixed dimensions cause troubles for the normalization of instantaneous screws which is required for the approaches given in~\cite{laryushkin2015estimation, Wolf,pottmann1998approximation, hartley2001invariant}

\begin{rmk}
    Hubert and Merlet \cite{hubert} suggested using the maximal joint forces to indicate the closeness to a singularity without giving an index\footnote{In \cite{hubert} only regions within the constant-orientation workspace were computed for which the joint forces are lower than a fixed threshold for a given wrench.}.
Such a closeness index was proposed in 
\cite{laryushkin2015estimation} but its computation involves the problematic 
normalization of the wrench. \hfill $\diamond$
\end{rmk}

The following EE-independent KPIs remain as candidates to indicate singularity closeness:

\begin{enumerate}[$\bullet$]
     \item Manipulability (M): This index was introduced by Yoshikawa~ \cite{yoshikawa1985manipulability} and used for a detailed study of 3-RPR manipulators in~\cite{lee1996optimum}. For these planar robots it equals the absolute value of the determinant of the Jacobian matrix $\mathbf{J}$; i.e.\ 
     $M(\mathbf{K})=|\det\mathbf{J}|$. 
      \item Incircle radius (IR): In \cite{ebrahimi2007actuation} it is proposed to take the radius of the incircle of the triangle determined by the carrier lines of the three legs to indicate the closeness to singularities\footnote{Poses with parallel lines have to be treated in a special way (for details see  \cite[Section 6.2]{ebrahimi2007actuation}) which causes jump discontinuities.\label{jump}}.
\begin{rmk}\label{rmk:mani}
The {\it manipulability} and the {\it incircle radius} are only based on the carrier lines of the legs, but
ignore the location of the anchor points on these lines. \hfill $\diamond$
\end{rmk}
    \item Transmission index (TI): In the first step, Takeda and  Funabashi \cite{takeda1995motion} determined the so-called pressure angles,  which equal the angle $\alpha_i\in[0,\tfrac{\pi}{2}]$  between the carrier line of the $i$-th leg and the velocity vector of the corresponding platform anchor point when the $i$-th prismatic link is driven alone. 
   Based on these pressure angles $\alpha_i$ the {\it transmission index} is defined as 
	\begin{equation}\label{eq:TI}
	min(\cos\alpha_1, \ldots,\cos\alpha_3)\in [0;1],
	\end{equation}
 for the case of 3-RPR manipulators. 
 \item 
 In \cite{Gomez} an index similar to the transmission index was introduced for 3-RPR robots, which claims to determine the {\it distance to singularity} (DS) by  
\begin{equation}\label{eq:DPS}
 DS:=1-\frac{2max(\alpha_1,\ldots,\alpha_3)}{\pi}\in [0;1], 
\end{equation}
using the notation from Eq.\ (\ref{eq:TI}). 
 \begin{rmk}\label{rmk:transmission}
 The indices TI and DS do not only take the carrier lines of the legs into account but also the location of the platform anchor points on these lines.  \hfill $\diamond$    
 \end{rmk} 
     \item Control number (CN): The minimum $\mu_-$ and maximum $\mu_+$ of the sum of the squared angular velocities of the passive joints are computed under the side condition, that the sum of the squared velocities of the prismatic joints equals 1. 
	Then the control number is given by $\sqrt{\mu_-/\mu_+}\in[0,1]$ according to \cite{nawratil2009new}.
  \begin{rmk}\label{rmk:control}
In contrast to TI and DS (cf.\ Remark \ref{rmk:transmission}) the control number also takes the location of the base anchor points on the carrier lines of the legs into account and therefore the complete geometry of the parallel manipulator.  \hfill $\diamond$   
 \end{rmk} 
\end{enumerate}

Note that there are also other approaches \cite{Voglewede, Voglewede1} to evaluate 
the vicinity to singularities involving physical quantities (like kinetic energy, and stiffness), which 
are not within the scope of this review as we are only interested in singularity closeness from a pure kinematic/geometric point of view.

\subsection{{Motivation}}\label{sec:motivation}

All the approaches for evaluating the closeness to a singularity mentioned so far are referred to as {\it indices} as the resulting values are not based on a distance function; i.e.\ a metric.
Therefore, from these index values no conclusion can be drawn on the shape and size of a singularity-free region in the workspace around the given configuration. For the determination of this information, the following method is proposed in the literature:

\begin{enumerate}[$\bullet$]
\item
	Li et al \cite{LI20061157} determined the singularity-free zone around a non-singular configuration as follows: 
	They parameterized the 3-dimensional configuration space by $x,y,\zeta$, where 
	$x,y$ are the two position variables and  $\zeta$ the orientation angle. 	
	Then point $(x,y,\zeta)$ of the singularity variety which minimizes the function 
	\begin{equation}
	d =(x-x_0)^2+(y-y_0)^2,
        \label{compare1}
	\end{equation}
	where $(x_0,y_0)^T$ corresponds to the position of the given non-singular configuration. 
	Note that the orientation $\zeta_0$ of the given configuration is not taken into account thus $\sqrt{d}$ 
	is the radius of the circular directrix centered in  $(x_0,y_0)^T$ of the ``{\it singularity-free cylinder}''. 
	This concept was also used in \cite{Abbasnejad}.

 Note that $\sqrt{d}$ cannot give the distance to the singularity as for a given non-singular configuration also $\sqrt{d}=0$ can hold. 
 But this method implies a kind of closest singular configuration, which can be seen as the contact point of the singularity-free cylinder and the singularity variety. 
\end{enumerate}

By computing the singularity distance using an extrinsic metric one does not only get the distance to the singularity variety (which can also be used as a KPI) but also the radius of a singularity-free sphere in the configuration space. In addition, we receive the closest singular configuration, i.e.\ the contact point of the singularity-free sphere and the singularity variety.

Following the argumentation of Wolf and Shoham \cite{Wolf} the closest pencil of lines to the legs of the 3-RPR robot, which is determined by the legs of the closest singular configuration, provides additional information and a better physical understanding of the motion the manipulator tends to perform in a singular configuration or its neighborhood. 
Moreover, the singularity-free spheres and the associated closest singular configurations can be used for path-planning \cite{LI20061157,Abbasnejad} and path-optimization \cite{RASOULZADEH2020104002}.

Note that the computation of 
the distance to the next singularity (or rather the radius of the largest singularity-free sphere) for a fixed orientation or a fixed position, respectively, are further concepts known in kinematics \cite{li2007determination, nag2021singularity, mao2013new}. But 
from these two pieces of separated information, no conclusion about the distance to the next singularity (i.e.\ the radius of the largest singularity-free sphere) within the configuration space of robots with mixed dofs can be drawn.

Instead of determining the configuration having the largest singularity-free sphere over all constant-orientation workspaces \cite{kaloorazi2016determining}, the 
proposed extrinsic metrics can also be used to 
find the configuration possessing the largest singularity-free sphere in the manipulator's total workspace. The center of this sphere can be regarded as the robot's home configuration and the maximization of its radius can be used as an objective for optimizing the design of the manipulator (cf.\ \cite{jiang2006maximal}). This concludes the motivation for our research on singularity distances using extrinsic metrics, which are reviewed next.

\subsection{Review on singularity distance computation using an extrinsic metric}\label{review}

	The second author presented in \cite[Eq.\ (3)]{Nawratil_2019} the following extrinsic metric 
    to compute the distance between two configurations $\mathbf{K}$ and $\mathbf{K}'$:
    \begin{equation}
    D^{\bullet}_{\bullet}(\mathbf{K},\mathbf{K}')^{2} = {\frac{1}{6}\sum_{i=1}^{6}||{\mathbf{k}'_{i}}-\mathbf{k}_{i}||^2}.
    \label{eq:metric0}
    \end{equation}
The upper and lower bullets in Eq.~(\ref{eq:metric0}) indicate that the distance is measured between the corresponding platform and base anchor points of the two configurations. This preliminary extrinsic metric was already used for determining the singularity distance, which has the following physical interpretation according to  \cite[Theorem 1]{Nawratil_2019}: {\it If the radial clearance of the six passive R-joints is smaller than $D^{\bullet}_{\bullet}$ then the parallel manipulator is guaranteed to be not in a singular configuration.} 
In the following, we give details on the computation of this singularity distance.

From the line-geometric point of view a 3-RPR configuration $\mathbf{K}'$ is singular if and only if the carrier lines of the three legs intersect in a common point (cf.\ Fig.~ \ref{3rpr}) or are parallel. This is equivalent to the fact that the Pl\"ucker coordinates of these lines are linearly dependent, resulting in the algebraic characterization in the form of 
so-called {\it singularity variety} $V={0}$ with 
 \begin{equation}\label{variety}
\mathrm{V} = \det \mathbf{V}(\mathbf{K}'), \quad  {\mathbf{V}(\mathbf{K}')}:=\begin{pmatrix}
 \mathbf{k}'_{4}-\mathbf{k}'_{1} & \phantom{-}\mathbf{k}'_{5}-\mathbf{k}'_{2} & \phantom{-}\mathbf{k}'_{6}-\mathbf{k}'_{3} \\ 
\det\left(\mathbf{k}'_{1},\mathbf{k}'_{4}-\mathbf{k}'_{1}\right)& \phantom{-}\det\left(\mathbf{k}'_{2}, \mathbf{k}'_{4}-\mathbf{k}'_{2}\right)& \phantom{-}\det\left(\mathbf{k}'_{3}, \mathbf{k}'_{4}- \mathbf{k}'_{3}\right) 
 \end{pmatrix}.
\end{equation}

The {\it singularity polynomial} $V$ is of degree 4 in the coordinates  $(c_i,d_i)^T$ of the points $\mathbf{k}_i'$ with respect to the fixed frame (for $i=1,\ldots,6$). Therefore, it can be seen as a quartic hypersurface in the configuration space $\mathbb{R}^{12}$.

The minimization of $D^{\bullet}_{\bullet}(\mathbf{K},\mathbf{K}')$ given in   Eq.\ (\ref{eq:metric0})
under the side condition $V=0$ of Eq.\ (\ref{variety})  is a constrained optimization problem, which can be formulated as the Lagrange function $L$ with
\begin{equation}
L =D^{\bullet}_{\bullet}(\mathbf{K}, \mathbf{K}')^2+\lambda V,
 \label{eq:distance}
\end{equation}
\noindent where $\lambda$ is the Lagrange multiplier. 

To find the closest singular configuration, one needs to compute the critical points of $L$ as the zero sets of the partial derivatives of $L$ with respect to the 13 involved unknowns: $c_1, \ldots, c_6$, $d_1, \ldots, d_6$ and $\lambda$. In \cite{Nawratil_2019}, the resulting systems of square polynomial equations were solved using symbolic computations (Gr\"obner basis methods) implemented in a computer algebra system (software \texttt{Maple}). Subsequently, the configuration corresponding to the critical point that yields the smallest value for $D^{\bullet}_{\bullet}(\mathbf{K},\mathbf{K}')$ has been designated as the closest singular configuration (cf. Fig.~\ref{measure}).

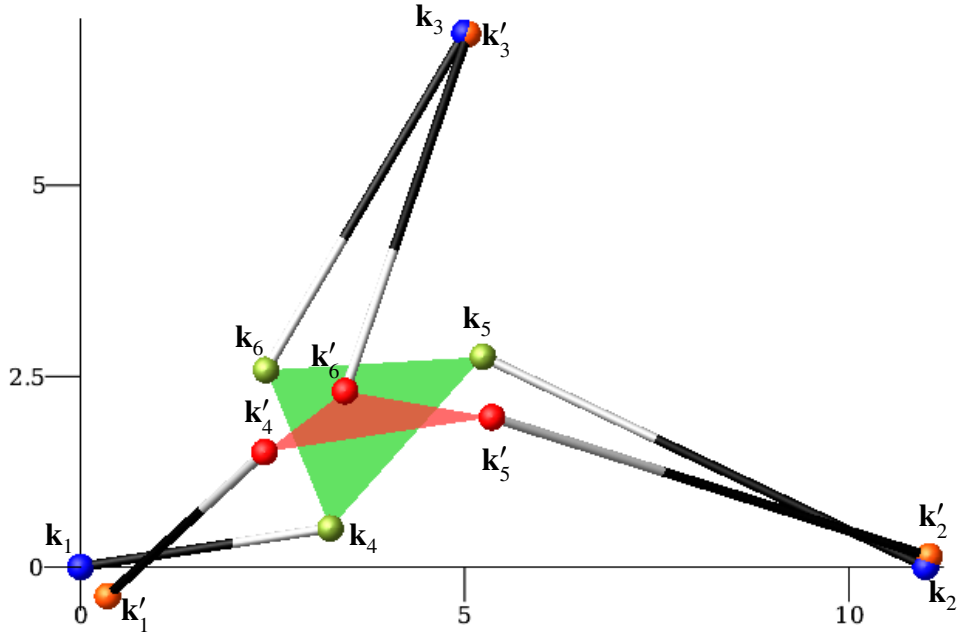
\begin{figure}[h!]
    \begin{center}
        \begin{overpic}[width=12.5cm]{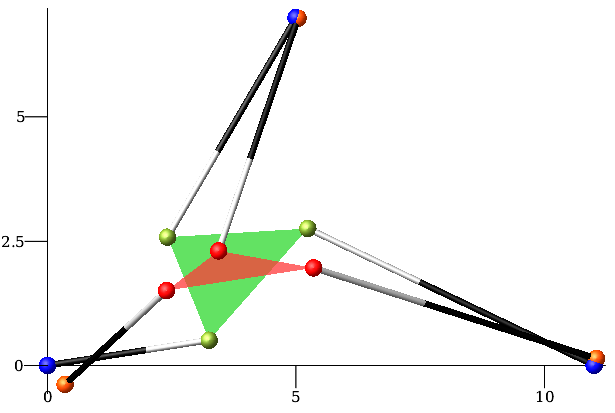}
       \put(4,10){$\mathbf{k}_{1}$}
        \put(12,2){{{$\mathbf{k}'_{1}$}}}
         \put(97,4){{$\mathbf{k}_{2}$}}
         \put(96,12){{$\mathbf{k}'_{2}$}}
           \put(50,18){{$\mathbf{k}'_{5}$}}
             \put(48,33){{$\mathbf{k}_{5}$}}
              \put(36,10){{$\mathbf{k}_{4}$}}
               \put(25,23){{$\mathbf{k}'_{4}$}}
              \put(24,31){{$\mathbf{k}_{6}$}}
              \put(32,29){{$\mathbf{k}'_{6}$}}
                \put(43,65){{$\mathbf{k}_{3}$}}
                  \put(50,63){{$\mathbf{k}_{3}'$}}
                   
        \end{overpic}
        \end{center}
 \caption{
Closest singular configuration $\mathbf{K}'$ (red) for a given non-singular 3-RPR manipulator pose $\mathbf{K}$ (green) with  $D^{\bullet}_{\bullet}(\mathbf{K}, \mathbf{K}')=0.7541454$ units. $\mathbf{K}$ corresponds to the configuration given in \cite[Section 3]{Nawratil_2019} for $\phi = 0.8471710528$ radians.}
\label{measure}
\end{figure}

\begin{rmk}
In~\cite{jaquier2021geometry} it is pointed out that a Riemannian distance between the manipulability ellipsoids of two configurations can be well-defined; apart from its dependence on the operation point and the dimensional inhomogeneity in the case of mixed dof robots.
Nevertheless, this metric cannot be used to compute a singularity distance as it is only defined for non-singular configurations (see also \cite{MARIC2021103865}). \hfill $\diamond$
\end{rmk}

\subsection{Contribution and outline}\label{sec:moti}
A limitation of the approach presented in \cite{Nawratil_2019} concerns the computational efficiency, caused by Gr\"obner basis calculations. As we aim to compute the singularity distance, (incl. the closest singularity), along a one-parametric motion of the manipulator, we use the numerical algebraic geometry tool of homotopy continuation, which is implemented in the freeware \texttt{Bertini}~\cite{BHSW06}. To simplify these computations for users, we have developed an open-source interface that connects the software packages \texttt{Maple}, \texttt{Bertini} \cite{bates2013numerically}, and \texttt{Paramotopy} \cite{bates2018paramotopy}.

Moreover, the paper at hand also fills the gap that singular points of the singular variety $V=0$ are excluded from the Lagrangian approach given in Eq.\ (\ref{eq:distance}) as pointed out in~\cite[ Section 3.2.2]{RASOULZADEH2020104002}. We identify these points, which are considered separately within the presented computational pipeline. 

A further limitation of the distance metric $D^{\bullet}_{\bullet}$ given in Eq.\ (\ref{eq:metric0}) is that it does not take into account how the six anchor points are connected combinatorially. Especially for the three platform/base anchor points one can distinguish two basic design options; namely if they are considered as the vertices of a triangular plate ($\blacktriangle$) or as the pin-joints of a triangular bar structure ($\vartriangle$). 
Note that in the latter case, additional shaky configurations arise because of the collinearity of the three base/platform points.

In the general case, the structural components (bars and plates) are assumed to be made of deformable material allowing not only a variation of the leg lengths\footnote{Note that the legs are also considered as bars. } but also a change of the platform/base geometry by an affine transformation. 
We also take the option into account that the  platform/base is made of a non-deformable material, which is indicated by the symbol 
$\hrectangleblack$, restricting affine transformations to Euclidean motions.
Thus we end up with three possibilities ($\blacktriangle, \vartriangle, \hrectangleblack$)
for the platform/base, which results in a total of nine interpretations illustrated in Fig.\ \ref{inter}.

\begin{figure}[t]
    \begin{center}
       \begin{overpic}[width=14cm]{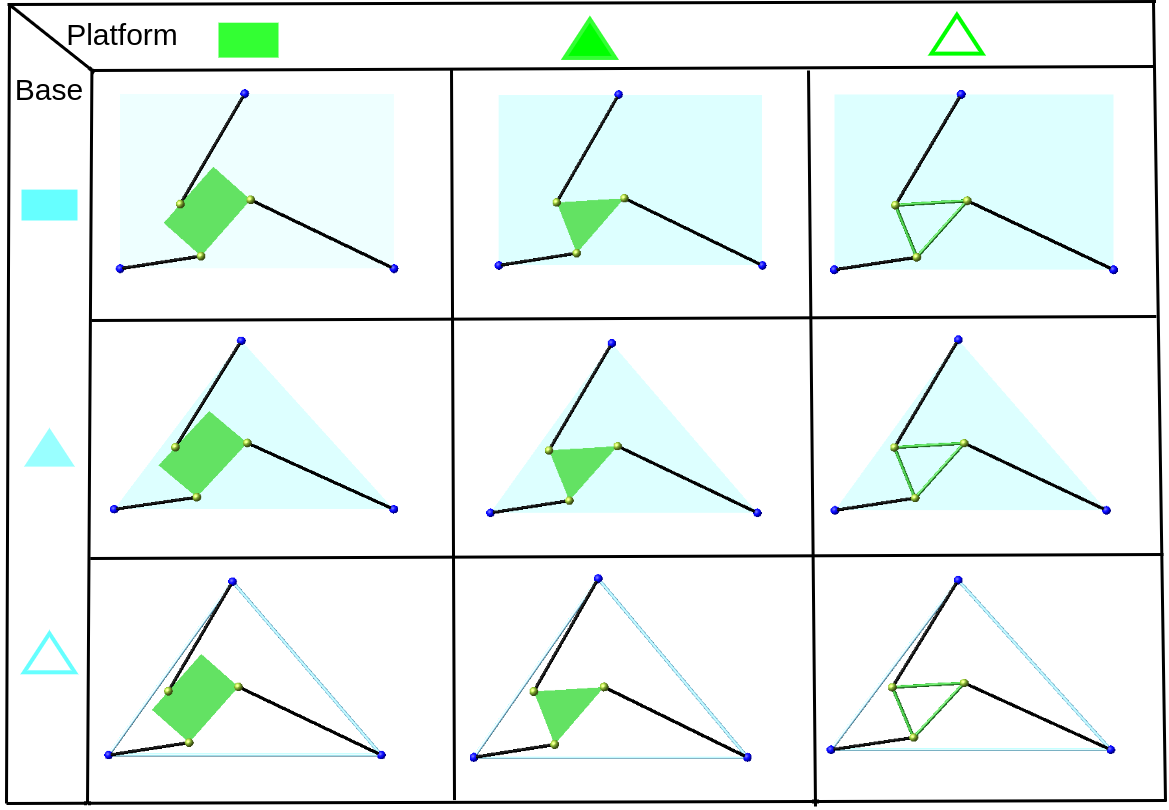}
     \end{overpic}
       \end{center}
     \caption{An illustration of the nine considered interpretations of 3-RPR manipulators arranged in a $3\times 3$ matrix.}
      \label{inter}  
\end{figure}

\begin{rmk}\label{rmk:sensitivity}
Note that we have distinguished the same nine interpretations in the already mentioned study \cite{intrinsic} using intrinsic metrics. In the paper at hand, we present the corresponding extrinsic metrics. 
These paired intrinsic and extrinsic metrics can be used for quantifying the change in the shape of the manipulator implied by variations of the inner geometry, which contributes to the 
topic of sensitivity analysis (e.g.~\cite{caro2009sensitivity,goldsztejn2016three}). 
\hfill $\diamond$ 
\end{rmk}

The combinatorial structure of the nine interpretations of Fig.\ \ref{inter}  is taken into account by constructing extrinsic metrics, which rely on the distance computation between corresponding structural components (bars or plates) outlined in Section \ref{sec:extrinsic}. 
The constrained optimization problems related to these extrinsic metrics are set up in Section \ref{generalcase}.
A geometric characterization of the singular points of the constrained varieties is obtained in Section \ref{sec:singpoints}. 
The developed computational pipeline for computing the singularity distance along a one-parametric motion of the manipulator is given in Section \ref{results}. In Section~\ref{numericalexample} we present a numerical example, and compare the obtained singularity distances with existing approaches already pointed out in Sections \ref{sec:review} and \ref{sec:motivation}
and discuss the results. 
We conclude the paper in Section \ref{end}.
\section{Extrinsic metric formulation}\label{sec:extrinsic}

We begin by setting up the distance function between the geometric elements (line segments and triangles) which can be associated with the structural components (bars and plates) of the manipulator.

According to~\cite{chen1999approximation}, the squared distance between two oriented line-segments $\vert_{ij}=(\mathbf{k}_{i},\mathbf{k}_{j})$  and $\vert_{ij}^{'}=(\mathbf{k}'_{i},\mathbf{k}'_{j})$ can be defined as as shown in (cf.~Fig.~\ref{fig:test1} left):

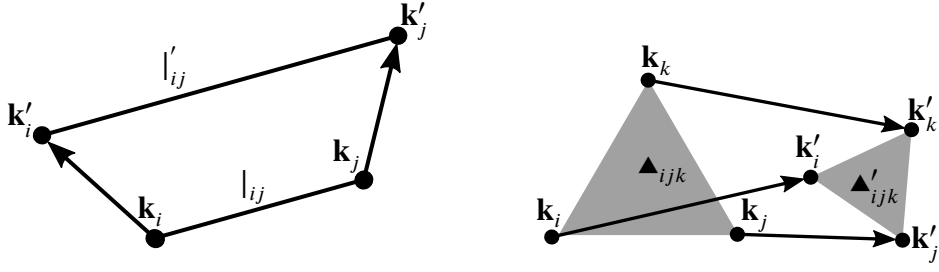
\begin{figure}
\centering
  \begin{overpic}[width=.3\linewidth]{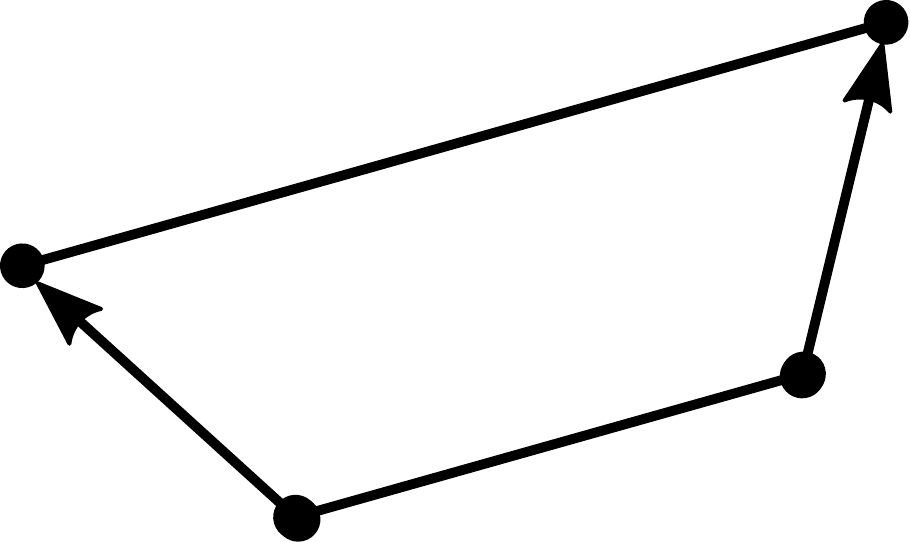}
  \put(55,15){$\vert_{ij}$}
  \put(28,8){$\mathbf{k}_{i}$}
  \put(-6.8,33){$\mathbf{k}'_{i}$}
  \put(80,23.5){$\mathbf{k}_{j}$}
  \put(98,60){$\mathbf{k}'_{j}$}
  \put(34,46.8){$\vert_{ij}^{'}$}
  \end {overpic}
\qquad \qquad
  \begin{overpic}[width=.3\linewidth]{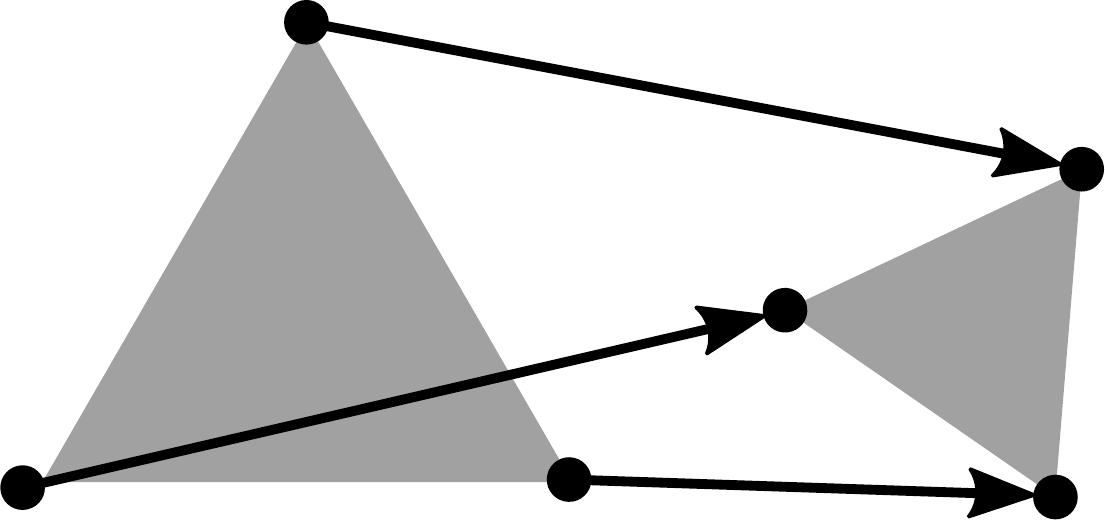}
  \put(52.5,7.5){$\mathbf{k}_{j}$}
  \put(-2,7.7){$\mathbf{k}_{i}$}
  \put(26,48.6){$\mathbf{k}_{k}$}
   \put(81.5,15){$\blacktriangle_{ijk}'$}
    \put(25,20){$\blacktriangle_{ijk}$}
  \put(98,0){$\mathbf{k}'_{j}$}
  \put(97,34.5){$\mathbf{k}'_{k}$}
    \put(67,24.5){$\mathbf{k}'_{i}$}
  \end{overpic}
  \caption{(left) Similarity mapping between line-segments. (right) Affine mapping between triangles.}
  \label{fig:test1}
\end{figure}

\begin{align}
d\left(\vert_{ij},\vert'_{ij}\right)^2=
\frac{1}{3}\Big[
\|\mathbf{k}_{i}-\mathbf{k}'_{i}\|^{2}+ \|\mathbf{k}_{j}-\mathbf{k}'_{j}\|^{2}+(\mathbf{k}_{i}-\mathbf{k}'_{i})^{T} (\mathbf{k}_{j}-\mathbf{k}'_{j})\Big].
\label{metricdp}
\end{align}

This distance metric equals the square root of the mean of squared distances of corresponding\footnote{The correspondence is given by the associated similarity transformation between  $\vert_{ij}$ and  $\vert_{ij}'$.} points over the entire line segment (see also \cite[Section 4.1]{survey}. 
One can extend this idea to compute the squared distance between two triangles $\blacktriangle_{ijk} = (\mathbf{k}_{i},\mathbf{k}_{j}, \mathbf{k}_{k})$ and $\blacktriangle'_{ijk}=(\mathbf{k}'_{i},\mathbf{k}'_{j}, \mathbf{k}'_{k})$ with  $i<j<k$ (cf.~Fig.~\ref{fig:test1} right) 
which yields (see ~\ref{minterpretation} for the derivation)
\begin{equation}\label{metricdp1}
d(\blacktriangle_{ijk}, \blacktriangle'_{ijk})^2=\frac{1}{6}\Big[\sum_{x=i,j,k} \|\mathbf{k}_x-\mathbf{k}'_x \|^{2} +(\mathbf{k}_{i}-\mathbf{k}'_{i})^{T} (\mathbf{k}_{k}-\mathbf{k}'_{k})
+(\mathbf{k}_{i}-\mathbf{k}'_{i})^{T} (\mathbf{k}_{j}-\mathbf{k}'_{j})+ (\mathbf{k}_{k}-\mathbf{k}'_{k})^{T} (\mathbf{k}_{j}-\mathbf{k}'_{j})\Big],
\end{equation}
where the index set $\{i,j,k\}$ equals either $\{1,2,3\}$ or $\{4,5,6\}$. Therefore Eq.\ (\ref{metricdp1}) equals the integral of the squared distance of corresponding\footnote{The correspondence is given by the associated affine transformation between  $\blacktriangle_{ijk}$ and  $\blacktriangle_{ijk}'$.} points over the triangle, divided by its area. 

By using Eqs.\ (\ref{metricdp}) and (\ref{metricdp1}) the squared extrinsic distance functions $D_\star^\circ(\mathbf{K},\mathbf{K}')^2$ with $\circ,\star\in\left\{\blacktriangle, \vartriangle, \hrectangleblack\right\}$ 
can be defined as the sum of the squared distances between corresponding deformable structural elements divided by the number of summands, which yields the following expressions:

\begin{align}
  D_{\rectangleblack}^{\rectangleblack}(\mathbf{K},\mathbf{K}')^2 &=
 \frac{1}{3} \sum_{(i, j) \in I_{1}} {d\left(\vert_{ij},\vert'_{ij}\right)^2},
   \label{Ex1} \\
    D^{\blacktriangle}_{\rectangleblack} (\mathbf{K},\mathbf{K}')^2 &= \frac{1}{4} \left[ \sum_{(i, j) \in I_{1}}  {d\left(\vert_{ij},\vert'_{ij}\right)^2} + {d(\blacktriangle_{456}, \blacktriangle_{456}')^2}\right],
    \label{Ex2}\\
     D_{\rectangleblack}^{\vartriangle}(\mathbf{K},\mathbf{K}')^2 &= \frac{1}{6} \sum_{(i, j) \in I_{2}} {d\left(\vert_{ij},\vert'_{ij}\right)^2},
      \label{Ex3} \\
   D^{\rectangleblack}_{\blacktriangle} (\mathbf{K},\mathbf{K}')^2 &=
   \frac{1}{4} \left[ \sum_{(i, j) \in I_{1}} {d\left(\vert_{ij},\vert'_{ij}\right)^2}  + {d(\blacktriangle_{123}, \blacktriangle_{123}')^2}\right],
    \label{Ex4} \\ 
   D_{\vartriangle}^{\rectangleblack}(\mathbf{K},\mathbf{K}')^2 &=\frac{1}{6} \sum_{(i, j) \in I_{3}} {d\left(\vert_{ij},\vert'_{ij}\right)^2} ,
    \label{Ex5} \\
     D_{\blacktriangle}^{\blacktriangle}(\mathbf{K},\mathbf{K}')^2 &=\frac{1}{5} \left[{ \sum_{(i, j)\in I_{1}}  d\left(\vert_{ij},\vert'_{ij}\right)^2}  +  d(\blacktriangle_{123}, \blacktriangle_{123}')^2 + d(\blacktriangle_{456}, \blacktriangle_{456}')^2 \right],
    \label{Ex6}
\end{align}
\begin{align}
   D_{\blacktriangle}^{\vartriangle}(\mathbf{K},\mathbf{K}')^2 &= \frac{1}{7}\left[{ \sum_{(i, j)\in I_{2}} d\left(\vert_{ij},\vert'_{ij}\right)^2} +{d(\blacktriangle_{123}, \blacktriangle_{123}')^2}\right],
 \label{Ex7} \\
  D_{\vartriangle}^{\blacktriangle}(\mathbf{K},\mathbf{K}')^2 &=\frac{1}{7}\left[{ \sum_{(i, j)\in I_{3}} d\left( \vert_{ij}, \vert_{ij}^{'}\right)^2}+{d(\blacktriangle_{456}, \blacktriangle_{456}')^2}\right],
 \label{Ex8} \\
  D_{\vartriangle}^{\vartriangle}(\mathbf{K},\mathbf{K}')^2 &=\frac{1}{9} \sum_{(i, j) \in I_{4}} {d\left( \vert_{ij}, \vert_{ij}^{'}\right)^2},
   \label{Ex9}
 \end{align}
with 
 \begin{align}
    I_1 &= \{(1, 4), (2,5), (3,6)\}, \\
    I_2 &= \{(1, 4), (2, 5), (3, 6), (4, 5), (4, 6), (5, 6)\},\\
    I_3 &= \{(1, 4), (2, 5), (3, 6), (1, 2), (2, 3), (1, 3)\}, \\
     I_4 &= \{(1, 2), (2, 3), (1, 3), (1, 4), (2, 5), (3, 6), (4, 5), (5, 6), (4, 6)\}.
 \end{align}

Note that we restrict to deformable structural components in order to keep the analogy to the corresponding intrinsic metrics mentioned in Remark \ref{rmk:sensitivity}.


\section{The constrained optimization problem for computing the singularity distance}\label{generalcase}

\subsection{Closest configuration on the singularity variety}\label{sec:singvar}

In the following, we set up the optimization problems for computing the closest singularity on the singularity variety $V=0$ with respect to the nine extrinsic metrics presented in Section \ref{sec:extrinsic}. 

\begin{itemize}
\item  
$D_{\rectangleblack}^{\rectangleblack}(\mathbf{K},\mathbf{K}')$: In this case, the transformations of the platform and the base are both restricted to the Euclidean motion group SE(2). 
As the so-called {\it point-based representation} of SE(2) has the best computational performance according to~\cite{kapilavai2020homotopy}, we use it for our calculations. The coordinates of $\mathbf{k}'_i$ can be given in dependence of $\mathbf{k}'_{i-1}$ and $\mathbf{k}'_{i-2}$ for $i=3,6$ by
\begin{multicols}{2}
\begin{equation}
\begin{pmatrix}
    c_3\\
    d_3
\end{pmatrix} =\begin{pmatrix}
\frac{(c_{2}-c_{1})x_{3}+(d_{1}-d_{2})y_{3}+c_{1}x_{2}}{x_{2}}\\
\frac{(d_{2}-d_{1})x_{3}+(c_{2}-c_{1})y_{3}+d_{1}x_{2}}{x_{2}}
\end{pmatrix},
\label{PBR1}
\end{equation}\quad 
\begin{equation}
\begin{pmatrix}
    c_6\\
    d_6
\end{pmatrix} =\begin{pmatrix}
\frac{(c_{5}-c_{4})x_{6}+(d_{4}-d_{5})y_{6}+c_{4}x_{5}}{x_{5}}\\
\frac{(d_{5}-d_{4})x_{6}+(c_{5}-c_{4})y_{6}+d_{4}x_{5}}{x_{5}}
\end{pmatrix}
\label{PBR2}
\end{equation}
\end{multicols}
under the side condition that the distance between  $\mathbf{k}'_{i-1}$ and $\mathbf{k}'_{i-2}$ does not change, which is expressed by the conditions $E_B=0$ and $E_P=0$ with
\begin{multicols}{2}
\begin{equation}
 E_B=\|\mathbf{k}'_{2}-\mathbf{k}'_{1}\|^{2}-\|\mathbf{k}_{2}-\mathbf{k}_{1}\|^2,
 \label{base1}
 \end{equation}\quad
\begin{equation}
E_P=\|\mathbf{k}'_{5}-\mathbf{k}'_{4}\|^{2}-\|\mathbf{p}_{5}-\mathbf{p}_{4}\|^2.
\label{platform1}
\end{equation}
\end{multicols}
Note that due to Lemma \ref{lem:app} given in \ref{app:lem} we can always assume without loss of generality that there exists a labeling of our 3-RPR manipulator such that $x_2\neq 0$ and $x_5\neq 0$ hold, which is needed for the properness of Eqs.\ (\ref{PBR1}) and (\ref{PBR2}).  

As a consequence, the Lagrange function $L$ reads as 
\begin{equation}
L = D_{\rectangleblack}^{\rectangleblack}(\mathbf{K},\mathbf{K}')^2+\lambda V+\mu E_B+\kappa E_P, \label{formulation1}
\end{equation}
 where $\lambda$, $\mu$ and $\kappa$  are the Lagrange multipliers. 
 \item  
 $D_{\rectangleblack}^\circ(\mathbf{K},\mathbf{K}')$ with $\circ \in \{{\blacktriangle}, {\vartriangle}\}$: In these two cases only the base is transformed by a Euclidean displacement, thus the Lagrange function reads as:
 \begin{equation} \label{side1}
 L=D_{\rectangleblack}^\circ(\mathbf{K},\mathbf{K}')^2+\lambda V+\mu E_B.
\end{equation}

 \item
$D^{\rectangleblack}_\star(\mathbf{K},\mathbf{K}')$ with $\star \in \{{\blacktriangle}, {\vartriangle}\}$: 
In these two cases, only the platform is transformed by a Euclidean displacement, thus the Lagrange function reads as:
 \begin{equation} \label{side2}
 L=D^{\rectangleblack}_\star(\mathbf{K},\mathbf{K}')^2+\lambda V+\kappa E_P.
\end{equation}

 \item  
 $D^{\circ}_\star(\mathbf{K,K'})$ with $\circ,\star \in \{{\blacktriangle}, {\vartriangle}\}$:
In these four cases the platform, as well as the base, are transformed affinely thus the Lagrangian reads as: 
\begin{equation}\label{side3}
L = D^{\circ}_\star(\mathbf{K},\mathbf{K}')^2 +\lambda V.
\end{equation}
\end{itemize}

The unknowns appearing in the given Lagrange functions $L$ are summarized in Table \ref{extab1}.

As already mentioned in Section \ref{sec:moti}, we get additional singular configurations if the platform/base is interpreted as a triangular bar structure. The optimization problem for computing the closest configurations with collinear platform/base anchor points is discussed in the next subsection. 

\begin{table}[t]
\caption{Summary of the Lagrange formulations of the constrained optimization problems of Sections \ref{sec:singvar} and \ref{collinearity}}
\centering
\begin{tabular}{|l||l|l|}
\hline
Extrinsic metric & unknowns in the Lagrangian & $\#$ unknowns \\ \hline
$D_{\rectangleblack}^{\rectangleblack}(\mathbf{K},\mathbf{K}')$                &   $c_{1},d_{1},c_{2},d_{2},c_{4},d_{4},c_{5},d_{5},
               \lambda, \kappa, \mu$  &  11 \\ \hline
$D_{\rectangleblack}^\circ(\mathbf{K},\mathbf{K}')$ with $\circ \in \{{\blacktriangle}, {\vartriangle}\}$               &  $c_{1},d_{1},c_{2},d_{2},c_{4},d_{4},c_{5},d_{5},c_{6},d_{6}, \lambda, \mu $        &  12 \\ \hline
$D^{\rectangleblack}_\star(\mathbf{K},\mathbf{K}')$ with $\star \in \{{\blacktriangle}, {\vartriangle}\}$               &  $c_{1},d_{1},c_{2},d_{2},c_{3},d_{3}, c_{4},d_{4},c_{5},d_{5}, \lambda, \kappa $          &  12 \\ \hline
$D^{\circ}_\star(\mathbf{K,K'})$ with $\circ,\star \in \{{\blacktriangle}, {\vartriangle}\}$ &
$c_{1},d_{1},c_{2},d_{2},c_{3},d_{3}, c_{4},d_{4},c_{5},d_{5}, c_{6},d_{6}, \lambda$        & 13  \\ \hline
\end{tabular}
\label{extab1}
\end{table}
\subsection{Closest configuration on the collinearity variety}\label{collinearity}

If the base or platform is interpreted as a triangular bar-structure $(\vartriangle)$ the additional singularities can be characterized 
algebraically by the condition $C_B=0$ and $C_P=0$, respectively, with
\begin{equation}
\mathrm{C}_B =\det\begin{pmatrix}
1 & 1 & 1 \\ 
c_{1} & c_{2} &  c_{3} \\
d_{1} & d_{2} &  d_{3} \\
\end{pmatrix},\quad
\mathrm{C}_P =\det\begin{pmatrix}
1 & 1 & 1 \\ 
c_{4} & c_{5} &  c_{6} \\
d_{4} & d_{5} &  d_{6} \\
\end{pmatrix}.
\label{coll}
\end{equation}
Note that the so-called {\it collinearity varieties} 
$C_B=0$ and $C_P=0$ are quadratic in $c_1,\ldots,c_6,d_1,\ldots,d_6$. 

In the following, we set up the optimization problems for computing the closest singularity on the collinearity variety $C_{B}=0$ and $C_{P}=0$, respectively. 

\begin{itemize}
    \item  
 $D_{\rectangleblack}^{\vartriangle}(\mathbf{K},\mathbf{K}')$: In this case, only the platform can deform, thus the Lagrange function with collinearity condition $C_{P}=0$ reads as: 
  \begin{equation} \label{c1}
 L=D_{\rectangleblack}^{\vartriangle}(\mathbf{K},\mathbf{K}')^2+\lambda C_{P}+\mu E_B.
\end{equation}

\item  $D_{\vartriangle}^{\rectangleblack}(\mathbf{K},\mathbf{K}')$: In this case, only the base can deform, thus the Lagrange function with collinearity condition $C_{B}=0$ reads as: 

  \begin{equation} \label{c2}
 L=D^{\rectangleblack}_{\vartriangle}(\mathbf{K},\mathbf{K}')^2+\lambda C_{B}+\kappa E_P.
\end{equation}
 
\item  $D^{\vartriangle}_\star(\mathbf{K,K'})$ with $\star \in \{{\blacktriangle}, {\vartriangle}\}$:
In these two cases the platform, as well as the base, are transformed affinely thus the Lagrangian reads as: 
\begin{equation}\label{c3}
L = D^{\vartriangle}_\star(\mathbf{K},\mathbf{K}')^2 +\lambda C_{P}.
\end{equation}

\item  $D^{\circ}_{\vartriangle}(\mathbf{K,K'})$ with $\circ \in \{{\blacktriangle}, {\vartriangle}\}$:
In these two cases the platform, as well as the base, are transformed affinely thus the Lagrangian reads as: 
\begin{equation}\label{c4}
L = D^{\circ}_\vartriangle(\mathbf{K},\mathbf{K}')^2 +\lambda C_{B}.
\end{equation}
\end{itemize}

The number of unknowns for the Lagrange optimization problems for the mentioned above four cases either with collinearity variety $C_{B}=0$ and  $C_{P}=0$ is the same as summarized in Table~\ref{extab1}.

\begin{thm}\label{regression1}
    The platform anchor points of the closest configuration $\mathbf{K'}$ on the collinearity variety $C_{P}=0$ with respect to the extrinsic metric $D_{\star}^{\vartriangle}(\mathbf{K},\mathbf{K}')$ with $\star \in \{{\blacktriangle}, {\vartriangle}\}$ are the pedal points of $\mathbf{k}_{4},\mathbf{k}_{5},\mathbf{k}_{6}$ on their line of regression (cf. Fig.~\ref{example}(b) and~\ref{example}(c)). 
    Moreover, the distance  $D_{\star}^{\vartriangle}(\mathbf{K},\mathbf{K}')$  only depends on the geometry of the manipulator; i.e.\ it is pose independent.
\end{thm}

\begin{proof}
The partial derivatives of Eq.\ (\ref{c3}) with respect to the  twelve  coordinates $c_{i}, d_{i}$ ($i=1,\dots 6$) and the 
Lagrange multiplier $\lambda$ results in the following ideal of thirteen equations:
\begin{equation}
\mathcal{I}_{1}:=\langle g_{1}, \dots, g_{13}\rangle \subseteq \mathbb{K}[c_{i},d_{i},x_{2},x_{3}, x_{4}, x_{5}, x_{6} ,y_{3},y_{4},y_{5},y_{6},\lambda] \quad \text{for}  \quad i= 1, \dots , 6.
\label{ideala}
\end{equation}
On the other hand, the pedal points of $\mathbf{k}_4,\mathbf{k}_5,\mathbf{k}_6$ on the line of regression can be obtained as a solution of the following optimization problem:
\begin{equation}
   L = {\frac{1}{3}\sum_{i=1}^{3}||{\mathbf{k}'_{i}}-\mathbf{k}_{i}||^2}+\lambda_{1} C_P
      \quad \text{for} \quad  i=4,5,6.
    \label{eq:colmetric}
    \end{equation}
The partial derivatives of Eq.\ (\ref{eq:colmetric}) with respect to the  six unknowns  $c_{i}, d_{i}$ ($i=4,5,6$) including the multiplier $\lambda_1$ results in the following ideal: 
 \begin{equation}
\mathcal{I}_{2}:=\langle f_{1}, \dots, f_{7}\rangle \subseteq \mathbb{K}[c_{i},d_{i},x_{4},x_{5}, x_{6},y_{4},y_{5},y_{6},\lambda_{1}] \quad \text{for}  \quad i=4,5,6.
\label{idealb}
\end{equation}
In order to show that the critical points of both Lagrange formulations are identical, one can eliminate $c_{i}, d_{i},\lambda$  for $i=1,2,3$ from the ideal $\mathcal{I}_{1}$ and $\lambda_1$ from $\mathcal{I}_{2}$  i.e.\
\begin{equation}
\mathcal{I}_{3}:=\mathcal{I}_1 \cap \mathbb{K}[c_{i},d_{i},x_{i},y_{i},x_{2},x_{3},y_{3}], \quad
\mathcal{I}_{4}:=\mathcal{I}_{2} \cap \mathbb{K}[c_{i},d_{i},x_{i},y_{i}]  
\quad \text{for} \quad i= 4,5,6.
\end{equation}
By the usage of the software \texttt{Maple}, it can be verified that $\mathcal{I}_{3}$ is contained in $\mathcal{I}_{4}$ and vice versa.

To prove that the distance $D_{\star}^{\vartriangle}(\mathbf{K},\mathbf{K}')$ 
is pose independent, we parameterize the pose determined by 
$\mathbf{R}$ and $\mathbf{t}$ in Eq.\ (\ref{initial}) by
\begin{equation}
\mathbf{R}:=\frac{1}{e_0^2+e_1^2}
\begin{pmatrix}
    e_0^2-e_1^2 & -2e_0e_1 \\
    2e_0e_1 & e_0^2-e_1^2
\end{pmatrix}, \quad
\mathbf{t}:=
\begin{pmatrix}
\alpha \\ \beta
\end{pmatrix}.
\end{equation}
Using this parametrization we take again the partial derivatives 
of Eq.\ (\ref{c3}) with respect to the  twelve  coordinates $c_{i}, d_{i}$ ($i=1,\dots 6$) and the 
Lagrange multiplier $\lambda$. 
 We solve the resulting system of $13$ equations for $\lambda, c_{i}, d_{i}$ ($i=1,\dots 6$) by using Gr\"obner basis package implemented in \texttt{Maple} and obtain two solution sets. 
 Substituting each of the two obtained solutions for $c_{i}, d_{i}$ ($i=1,\dots ,6$) back into the extrinsic distance functions given by Eq.~\ref{Ex7} and Eq.~\ref{Ex9}, respectively, shows that the resulting expression only depends on the geometry parameters $x_{5}, x_{6}, y_{6}$. 
 For the explicit expressions of 
 $D^{\circ}_{\vartriangle}(\mathbf{K,K'})$ with $\circ \in \{{\blacktriangle}, {\vartriangle}\}$ we refer to  \ref{explicit}.

The \texttt{Maple} files used for proving Theorem~\ref{regression1} can be downloaded from ~\cite{codes}.
\end{proof}

Clearly, this theorem also holds by exchanging the platform and the base which yields:

\begin{thm}\label{regression2}
    The base anchor points of the closest configuration $\mathbf{K'}$ on the collinearity variety $C_{B}=0$ with respect to the extrinsic metric $D_{\vartriangle}^{\circ}(\mathbf{K},\mathbf{K}')$ with $\circ \in \{{\blacktriangle}, {\vartriangle}\}$ are the pedal points of $\mathbf{k}_{1},\mathbf{k}_{2},\mathbf{k}_{3}$ on their line of regression. Moreover, the distance  $D_{\vartriangle}^{\circ}(\mathbf{K},\mathbf{K}')$  only depends on the geometry of the manipulator; i.e.\ it is pose independent.
\end{thm}

Note that the first sentence of Theorem~\ref{regression1} (resp.\ Theorem~\ref{regression2})  does not hold for the metric  $D_{\rectangleblack}^{\vartriangle}(\mathbf{K},\mathbf{K}')$ (resp.\ $D_{\vartriangle}^{\rectangleblack}(\mathbf{K},\mathbf{K}')$), which is demonstrated by the following counter-example.

\begin{Counter}\label{counter1}
We use as input for our numerical example the one discussed in  \cite[Section 3]{Nawratil_2019} for $\phi=\frac{\pi}{2}$. For this configuration, we solved the optimization problems stated in Eq.~(\ref{c1}) and Eq.~(\ref{c3}). 
The configurations that correspond to the global minima 
are displayed in Fig.~\ref{example} and their coordinates are given in Table.~\ref{coordinates}.
\end{Counter} 
\begin{figure}[h!]
\begin{center}
\begin{overpic}
    [width=150mm]{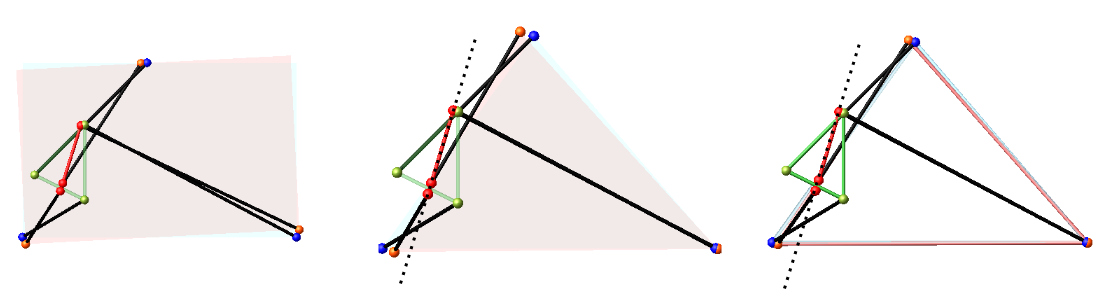}
    \begin{small}
     \put(5,0) {(a)  $D_{\rectangleblack}^{\vartriangle}(\mathbf{K,K'})$}
      \put(40,0) {(b)  $D_{\blacktriangle}^{\vartriangle}(\mathbf{K,K'})$}
       \put(80,0) {(c)  $D_{\vartriangle}^{\vartriangle}(\mathbf{K,K'})$}
      \end{small}
\end{overpic}
\end{center}
\caption{The line of regression is indicated by a dashed line and its equation is given by $0.9570920262x-0.2897841482y-0.9336136247=0$. From the visual point of view the platform anchor points in (a) are very similar to the ones of (b) and (c), which are identical due to Theorem \ref{regression1}. From Table~ \ref{coordinates} it can be seen that the coordinates of $\mathbf{k}'_{4}, \mathbf{k}'_{5}, \mathbf{k}'_{6}$ of (a) differ from those of (b) and (c) and that they are not fulfilling the equation of the line of regression.} 
\label{example}
\end{figure}

\begin{table}[H]
\caption{Coordinates of the closest singular configuration for counter-example \ref{counter1}}
\centering
\begin{tabular}{|l||l|l|l|}
\hline
$\mathbf{K}'$ & $D_{\rectangleblack}^{\vartriangle}(\mathbf{K},\mathbf{K}')=  0.5735791$   & $D_{\blacktriangle}^{\vartriangle}(\mathbf{K},\mathbf{K}')=0.5195729$  & $D_{\vartriangle}^{\vartriangle}(\mathbf{K},\mathbf{K}')=0.46807561$        \\ \hline
$\mathbf{k}'_{1}$ & $(0.1302373, -0.2775441)$ & $(0.3921934, 
-0.1187466)$ & $(0.1960967,
-0.0593733)$ \\ \hline
$\mathbf{k}'_{2}$ & $(11.114982, 0.3015644)$  & $(11.059373,
-0.0179767)$ & $(11.029686,
-0.0089883)$ \\ \hline
$\mathbf{k}'_{3}$ & $(4.7547798, 6.9759796)$  &  $(4.5484332,
7.1367234)$ & $(4.7742166,
7.0683617)$ \\ \hline
$\mathbf{k}'_{4}$  &  $(1.5271323,1.8504855)$  &  $(1.5195164,
1.7968665)$  &  $(1.5195164,
1.7968665)$ \\ \hline
$\mathbf{k}'_{5}$          &  $(2.3464552,
4.4803586)$        & $(2.3515667,
4.5449419)$         & $(2.3515667,
4.5449419)$   \\ \hline
$\mathbf{k}'_{6}$  & $(1.6264123,
2.1691557)$  & (1.6289168,
2.1581914) &  (1.6289168,
2.1581914)  \\ \hline
 \end{tabular}
\label{coordinates}
\end{table}

Also, the second sentence of Theorem~\ref{regression1} (resp.\ Theorem~\ref{regression2})  does not hold for   $D_{\rectangleblack}^{\vartriangle}(\mathbf{K},\mathbf{K}')$ (resp.\ $D_{\vartriangle}^{\rectangleblack}(\mathbf{K},\mathbf{K}')$), which can be seen from the example (cf.\ Fig.\ \ref{fig:result3}(a,d)) discussed in Section \ref{numericalexample}.

\section{Singular points of the constraint varieties}\label{sec:singpoints}
As already mentioned the Lagrangian formulations given in Eqs.\ (\ref{formulation1}--\ref{side3}) and Eqs.\ (\ref{c1}--\ref{c4})  do not take the singular points of the constraint varieties into account.  Therefore we have to take care of them separately, 
which is done in the following two subsections:
\subsection{Singular points of the singularity variety}

We are interested in giving a geometric characterization of the singular points of the singularity variety $V=0$. As a preparatory work towards this goal, we prove the following lemma:

\begin{lem} \label{invar}
The set of singular points of the singularity variety $V=0$ remains invariant under affine motions.
\end{lem}

\begin{proof}
We apply a regular affine transformation on the 
configuration space $\mathbb{R}^{12}$ by 
\begin{equation}
    \mathbf{k}'_i \mapsto \mathbf{A} \mathbf{k}'_i+ \mathbf{a} \quad \text{for}\quad
    i=1,\ldots,6
\end{equation}
where $\mathbf{A}$ is a regular $2\times2$ matrix and $\mathbf{a}\in\mathbb{R}^2$. 
This transformation induces a linear automorphism of the singularity variety $V=0$, as $\det \mathbf{V}$ is mapped to $(\det \mathbf{A})^2 \det \mathbf{V}$. As $\det \mathbf{A}\neq 0$ holds, this already shows the linear automorphism, which maps regular points to regular ones and singular ones to singular ones. 

\end{proof}

By using Lemma~\ref{invar}we can prove the following theorem:
\begin{thm}\label{col1}
Singular points of the singularity variety $V=0$ correspond to one of the following configurations:
\begin{enumerate}
\item Three legs of the manipulator are collinear (see Fig.~\ref{characterization}(a)).

\item  Two legs are collinear and one leg degenerates to a point (see Fig.~\ref{characterization}(b)). 

\item Two legs degenerate to points (see Fig.~\ref{characterization}(c)).

\item One leg degenerates to a point and the carrier lines of the remaining two legs pass through that point (see Fig.~\ref{characterization}(d)). 
\end{enumerate}
\end{thm}

\begin{figure}
 \begin{center}
 \begin{overpic}[width=14cm]{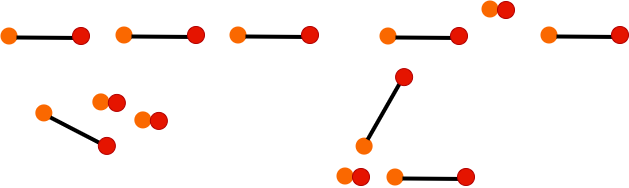}
   \begin{small}
\put(30,20){(a)}
\put(75,20){(b)}
\put(30,10){(c)}
\put(75,10){(d)}
\end{small}
\end{overpic}
\caption{Schematic sketch of the geometric characterization of singular points of the singularity variety $V=0$. Base and platform anchor points are indicated in orange and red, respectively.}\label{characterization}
 \end{center}
\end{figure}
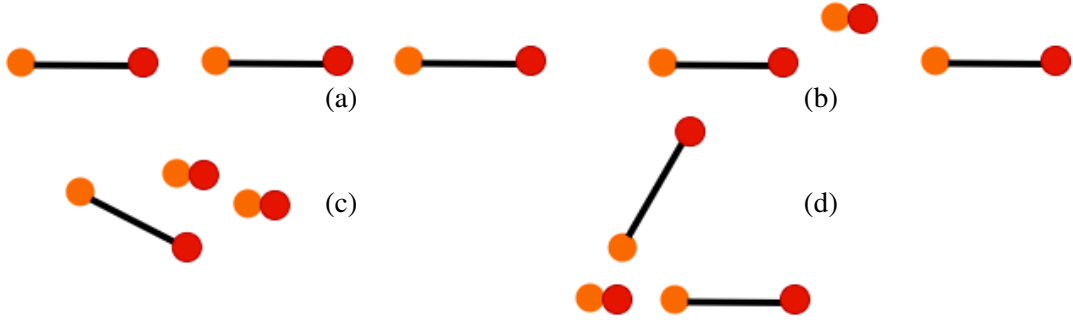
\begin{proof}

From Eq.\ (\ref{variety}) it can be seen that $V$ only depends on the free coordinates $c_{i}, d_{i}$ ($i=1,\dots,6$). Partial derivatives of Eq.\ (\ref{variety}) with respect to these twelve unknowns plus the singularity polynomial $V$ results in an overdetermined system of thirteen equations spanning the following ideal: 
\begin{equation}
\mathcal{I}:=\langle g_{1}, \dots, g_{13}\rangle \subseteq \mathbb{K}[c_{i},d_{i}] \quad i= 1, \dots ,6.
\label{ideal4}
\end{equation}

Due to Lemma~\ref{invar} we can set $c_{1}=d_{1}=d_{2}=0$ 
in $g_i$ of Eq.\ (\ref{ideal4}) for $i=1,\ldots, 13$. This simplification allows us to solve the resulting overdetermined system in nine unknowns by using Gr\"obner basis method implemented in \texttt{Maple}. 
It can easily be checked that each of the obtained 30 solution sets falls into one of the given four geometric characterizations, which shows their necessity. 
The \texttt{Maple} file for the computation of these  30 solution sets can be downloaded from ~\cite{codes}.

The proof of sufficiency is straightforward by checking that the 13 equations of  Eq.\ (\ref{ideal4}) are fulfilled under consideration of the listed four geometric conditions.
\end{proof}

\begin{rmk}
Note that case 4 of Theorem \ref{col1} also shows that singular points of the singularity variety are not characterized by rk($\mathbf{V}$)=$1$ with $\mathbf{V}$ of Eq.\ (\ref{variety}). 
Therefore, this also falsifies the conjecture of \cite[End of Sec.\ IV]{manfred} that singularities of the singularity variety yield higher-order singularities. \hfill $\diamond$
\end{rmk}

If we restrict the base (resp.\ platform) to be transformed by 
Euclidean motions, then our set of singular configurations is only a subset $V_B=0$ (resp.\ $V_P=0$) of $V=0$. 

Let us first assume that the base is transformed by Euclidean motions. 
Then the variety $V_B=0$ can be obtained as the intersection of the four hypersurfaces $V=0$, $E_B=0$, $F_1=0$, $F_2=0$ with 
\begin{equation}
F_{1}=(c_{2}-c_{1})x_{3}+(d_{1}-d_{2})y_{3}+(c_{1}-c_{3})x_{2}, \quad 
F_{2}=(d_{2}-d_{1})x_{3}+(c_{2}-c_{1})y_{3}+(d_{1}-d_{3})x_{2},
\end{equation}
where the latter two conditions are implied by Eq.\ (\ref{PBR1}). 

\begin{thm}\label{base}
The set of singular points of the singularity variety $V_B=0$ remains invariant under Euclidean motions. Moreover, these points are also characterized by the four cases given in Theorem \ref{col1}.
\end{thm}

\begin{proof}

The proof of the first part of the theorem can be done similarly to Lemma \ref{invar}, but by restricting to Euclidean motions; i.e.\ 
\begin{equation}
    \mathbf{k}'_i \mapsto \mathbf{R} \mathbf{k}'_i+ \mathbf{t} \quad \text{with}\quad
\mathbf{R}:=    \begin{pmatrix}
\cos{\phi} & -\sin{\phi} \\
\sin{\phi} & \cos{\phi}
\end{pmatrix}.
\end{equation}
We know already from Lemma \ref{invar} that  $\det \mathbf{V}=0$ remains invariant under this action. It can easily be seen that the same holds true for $E_B=0$. The remaining two hypersurfaces are transformed as follows:
\begin{equation}
F_{1} \mapsto \cos{\phi}F_{1}-\sin{\phi}F_{2}, \qquad 
F_{2} \mapsto \sin{\phi}F_{1}+\cos{\phi}F_{2}.
\end{equation}
 This already shows that Euclidean motions imply a linear automorphism of $V_B=0$.  

We proceed with the proof of the second part of the theorem. 
A singular point of $V_B=0$ is either a singular point of one of the four hypersurfaces  $V=0$, $E_B=0$, $F_1=0$, $F_2=0$ or it is a point, where the four tangent hyperplanes to these four hypersurfaces are linearly dependent. Algebraically this can be expressed by the set of equations resulting from the partial differentiation of 
\begin{equation}
\lambda_0V+\lambda_{1} E_B+\lambda_{2}F_{1}+\lambda_{3}F_{2}
\label{main3}
\end{equation}
with respect to the 19 unknowns
\begin{equation}
    c_1,\ldots,c_6,d_1,\ldots, d_6, x_2,x_3,y_3,\lambda_0,\ldots,\lambda_3.
    \label{unknowns}
\end{equation}
Due to the first part of the proof we can set $c_{1}=d_{1}=d_{2}=0$ in order to simplify the set of 19 equations, which allows us to solve them again by using Gr\"obner basis method implemented in \texttt{Maple}. The corresponding \texttt{Maple} file, which results in  $48$ solution sets, can be downloaded from ~\cite{codes}.

Out of the obtained $48$ solution sets, there are $24$ sets with $x_{2}=0$ contradicting our assumption $x_2\neq 0$ implied by Lemma \ref{lem:app}. There are $10$ sets  of trivial solutions as $\lambda_{0}=\lambda_{1}=\lambda_{2}=\lambda_{3}=0$ holds. For the remaining 14 solution sets $x_{2}\neq 0$ holds true and $\lambda_{1}=\lambda_{2}=\lambda_{3}=0$. The latter already shows that these 14 solution sets correspond to singular points of $V=0$.

 Note that the solution sets
 which correspond to case 1 of Theorem \ref{col1}
 also include a condition on the design parameters ($x_2,x_3,y_3$); namely $y_3=0$ rendering the base collinear.
\end{proof}

Clearly, this theorem also holds by exchanging the platform and the base which yields:

\begin{thm}\label{platform}
The set of singular points of the singularity variety $V_P=0$ remains invariant under Euclidean motions. Moreover, these points are also characterized by the four cases given in Theorem \ref{col1}.
\end{thm}

If we restrict both base and platform to be transformed by 
Euclidean motions, then our set $V_{BP}=0$ of singular configurations are obtained by intersecting the varieties $V_B=0$ and $V_P=0$. 
Note that the variety  $V_P=0$ is also  obtained as the intersection of the four hypersurfaces $V=0$, $E_P=0$, $F_3=0$, $F_4=0$ with 
\begin{equation}
F_{3}=(c_{5}-c_{4})x_{6}+(d_{4}-d_{5})y_{6}+(c_{4}-c_{6})x_{5},\quad
F_{4}=(d_{5}-d_{4})x_{6}+(c_{5}-c_{4})y_{6}+(d_{4}-d_{6})x_{5}
\end{equation}
where the latter two conditions are implied by Eq.\ (\ref{PBR2}). 

\begin{thm}\label{eucledian}
The set of singular points of the singularity variety $V_{BP}=0$ remains invariant under Euclidean motions. Moreover, these points are also characterized by the four cases given in Theorem \ref{col1}.
\end{thm}
\begin{proof}
The first statement follows directly from the fact that Euclidean motions imply a linear automorphism of $V_{BP}=0$, which results from the proofs of Theorems \ref{base} and \ref{platform}.

Therefore, we can focus on proving the second part of the theorem.
A singular point of the singularity variety $V_{BP}=0$ 
is either (I) a singular point of  $V_B=0$ or $V_P=0$ or (II) it is a point, where the seven tangent hyperplanes to the seven hypersurfaces  
\begin{equation}\label{seven}
V=0,\quad E_{B}=0,\quad E_P=0, \quad F_{1}=0,\quad F_{2}=0,\quad F_3=0,\quad F_4=0, 
\end{equation}
are linearly dependent. Algebraically this can be expressed by the set of equations resulting from the partial differentiation of 
\begin{equation}
\lambda_0 V+\lambda_{1} E_{B}+\lambda_{2} F_{1}+\lambda_{3} F_{2}+\lambda_{4} E_{P}+\lambda_{5} F_{3}+\lambda_{6} F_{4}
\label{nabla1}
\end{equation}
with respect to the 25 unknowns
\begin{equation}
    c_1,\ldots,c_6,d_1,\ldots, d_6, x_2,x_3, x_5, x_6,y_3, y_6, \lambda_0,\ldots,\lambda_6.
\end{equation}
They form the following ideal:
\begin{equation}
\langle g_{1}, \dots, g_{25}\rangle \subseteq \mathbb{K}[c_i,d_i, x_2,x_3, x_5, x_6,y_3, y_6, \lambda_0,\ldots,\lambda_6] \quad i= 1, \dots ,6.
\label{idealnew}
\end{equation}
Due to the first sentence of the theorem, we can set $c_{1}=d_{1}=d_2=0$.
Moreover, due to Lemma \ref{lem:app} given in \ref{app:lem} we can assume without loss of generality that $x_{2}x_{5}\neq 0$ holds. 
Thus by assuming a suitable scale unit we can set  $x_{2}x_{5}=1$. 
By adding this equation to the ideal of Eq.\  \ref{idealnew}, we end up with:
\begin{equation}
\mathcal{I}:=\langle g_{1}, \dots, g_{26}\rangle \subseteq \mathbb{K}[c_{2},\ldots,c_6,d_{3},\ldots,d_6,x_2,x_3, x_5, x_6,y_3, y_6,\lambda_0,\ldots,\lambda_6 ] 
\label{ideal10}
\end{equation}
In the following, we prove that no singular points of type (II)  exist. By setting $\lambda_j=1$ for $j\in\left\{1,\dots,6\right\}$ we  obtain the ideal
\begin{equation}
\mathcal{I}_{j}:=\langle g_{1}, \dots, g_{26}\rangle \subseteq \mathbb{K}[c_{2},\ldots,c_6,d_{3},\ldots,d_6,x_2,x_3, x_5, x_6,y_3, y_6,\lambda_0,\ldots \lambda_{j-1}, \lambda_{j+1},\ldots, \lambda_6 ]
\label{ideal11}
\end{equation}
The corresponding basis with respect to graded reverse lexicographic order 
is denoted by $\mathcal{B}_{j}$. According to \cite{sturmfels2005grobner}, \textit{Hilbert’s Nullstellensatz} implies that 
the variety of $\mathcal{I}_{j}$ is empty if and only if $\mathcal{B}_{j}=\left\{1\right\}$. 

It can be verified by using \texttt{Maple} that 
$\mathcal{B}_{j}=\left\{1\right\}$ holds true. The used \texttt{Maple} file can be downloaded from ~\cite{codes}.

Therefore, the 25 equations $g_i$ of Eq.\ (\ref{idealnew}), which are homogeneous with respect to $\lambda_0,\ldots,\lambda_6$, can only have a solution  for $(\lambda_0:\lambda_1:\ldots:\lambda_6)=(1:0:\ldots :0)$. Furthermore,  they belong to type (I).

 Clearly, the design parameters $x_2,x_3, x_5, x_6,y_3, y_6$ have to fulfill certain conditions such that the four cases given in Theorem \ref{col1} are feasible. Case 1 can only occur when the base and platform are both linear; case 2 when the base triangle and platform triangle have an equal corresponding height; 
 case 3 when they have an equal corresponding side and case 4 when they have an equal corresponding angle.
\end{proof}

\subsubsection{Parametrizing the set of singular points of the singularity variety}\label{sec:para}
To include the set of singular points from the singularity variety in our computation of the singularity distance, we parameterize it. We restrict ourselves to the singular points corresponding to case 1 of Theorem \ref{col1}. In practice, no leg can have zero length; thus, cases 2-4 are not of interest. As case 1 of Theorem \ref{col1} imposes more restrictive conditions than the collinearity conditions in Section~\ref{collinearsing}, the following relations must hold:
\begin{equation}
D_\star^{\vartriangle}(\mathbf{K},\mathbf{K}')\leq
D_\star^{\vartriangle}(\mathbf{K},\mathbf{K}''),\quad
D^\circ_{\vartriangle}(\mathbf{K},\mathbf{K}')\leq
D^\circ_{\vartriangle}(\mathbf{K},\mathbf{K}'') \quad
\text{with}\quad
\star,\circ\in\left\{\blacktriangle, \vartriangle, \hrectangleblack\right\}
\end{equation}
where $\mathbf{K}'$ is the global minimizer of Theorem 
\ref{regression1} and \ref{regression2}, respectively, and
$\mathbf{K}''$ denotes the closest singular point belonging to case 1 of  Theorem \ref{col1}. 

Therefore one only has to compute $\mathbf{K}''$
for the metrics $D^{\circ}_\star(\mathbf{K,K''})$ with $\circ,\star \in \{{\blacktriangle}, {\hrectangleblack}\}$, which is done next:

\begin{itemize}
\item
$D_\blacktriangle^{\blacktriangle}(\mathbf{K},\mathbf{K}'')$: 
The point $\mathbf{k}''_{1}$ is parameterized by its coordinates $(a,b)^T$  and the remaining points by:
\begin{equation}\label{eq:para}
    \mathbf{k}_{i+1}''=\mathbf{k}_{1}''+\delta_{i} \begin{pmatrix}
   e^2_{0}-e^2_{1} \\
   2e_{0}e_{1} 
    \end{pmatrix}     \quad    \textnormal{for} \quad i=1,\dots,5. 
\end{equation}
Now we minimize $D_\blacktriangle^{\blacktriangle}(\mathbf{K},\mathbf{K}'')$ under the normalization condition $N=0$ with $N:=e^2_{0}+e^2_{1}-1$, which results in the following Lagrangian formulation:
\begin{equation}
L=D_\blacktriangle^{\blacktriangle}(\mathbf{K},\mathbf{K}'')^2+\lambda N.
\label{singularformulation}
\end{equation}
\item
$D_\hrectangleblack^{\blacktriangle}(\mathbf{K},\mathbf{K}'')$: Clearly, $\mathbf{K}''$ can only exist if the base points of the given manipulator are collinear. If this is the case the parametrization can be done analogously to Eq.\ (\ref{eq:para}) with the sole difference that $\delta_1$ and $\delta_2$ are already known as they have to equal $x_2$ and $x_3$, respectively. 
\item
$D^\hrectangleblack_{\blacktriangle}(\mathbf{K},\mathbf{K}'')$: The same procedure can be applied as in the last case, just by swapping the roles of the platform and the base. 
\item
$D_\hrectangleblack^{\hrectangleblack}(\mathbf{K},\mathbf{K}'')$: Now $\mathbf{K}''$ can only exist if the given base points are collinear as well as the platform points. If this is the case the parametrization can be done as in Eq.\ (\ref{eq:para}) under consideration of $\delta_1=x_2,\delta_2=x_3$ with $\delta_4=\delta_3 \pm x_5$ and $\delta_5=\delta_3 \pm x_6$. 
Due to $\pm$ one has to run two optimization problems with six unknowns $a,b,\delta_3,e_0,e_1,\lambda$ and take the minimum over both. 
\end{itemize}

\begin{table}[]
\caption{Summary of the Lagrange formulations of the optimization problems of Section \ref{sec:para}.}
\centering
\begin{tabular}{|l||l|l|}
\hline
Extrinsic metric & unknowns in the Lagrangian & $\#$ unknowns \\ \hline
$D_{\blacktriangle}^\blacktriangle(\mathbf{K},\mathbf{K}')$                  &  $a, b, \delta_{1},  \delta_{2}, \delta_{3},  \delta_{4},  \delta_{5}, e_{0}, e_{1}, \lambda$                         &     $10$     \\ \hline
$D_{\hrectangleblack}^\blacktriangle(\mathbf{K},\mathbf{K}')$, $D_{\blacktriangle}^{\hrectangleblack}(\mathbf{K},\mathbf{K}')$                &   $a, b,  \delta_{3},  \delta_{4},  \delta_{5}, e_{0}, e_{1}, \lambda$     &    $8$      \\ \hline
$D_{\hrectangleblack}^{\hrectangleblack}(\mathbf{K,K'})$                   &      $a, b,  \delta_{3},  e_{0}, e_{1}, \lambda$                        &    $6$      \\ \hline
\end{tabular}
\label{singunknowns}
\end{table}
\subsection{Singular points of the collinearity variety}\label{collinearsing}
It can easily be verified by the reader that the following theorem, which is illustrated in Fig.\ \ref{characterization1}, holds true:
\begin{thm}\label{col3}
The set of singular points of the collinearity variety $C_B=0$ (resp.\ $C_P=0$) of the base (resp.\ platform) remains invariant under affine motions. Moreover, these points are characterized by 
$\mathbf{k}'_1=\mathbf{k}'_2=\mathbf{k}'_3$ (resp.\  
$\mathbf{k}'_4=\mathbf{k}'_5=\mathbf{k}'_6$).
\end{thm}
\begin{figure}[b]
 \begin{center}
 \begin{overpic}[width=6cm]{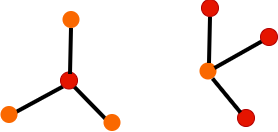}
   \begin{small}
\put(25,0){(a)}
\put(75,0){(b)}
\end{small}
\end{overpic}
\caption{Schematic sketch of the geometric characterization of singular points of the collinearity variety (a) $C_{P}=0$ (b) $C_{B}=0$. Base and platform anchor points are indicated in orange and red, respectively.}\label{characterization1}
 \end{center}
\end{figure}

As the condition that three points collapse to a single point is more restrictive than the condition that these three points are collinear the following relation has to hold:

\begin{equation}\label{eq:ineq}
D_\star^{\vartriangle}(\mathbf{K},\mathbf{K}')\leq
D_\star^{\vartriangle}(\mathbf{K},\mathbf{K}'')\quad\text{for}\quad
\star\in\left\{\blacktriangle, \vartriangle\right\}
\end{equation}
where $\mathbf{K}'$ is the global minimizer of Theorem \ref{regression1} and 
$\mathbf{K}''$ denotes the closest singular point of 
the collinearity variety $C_P=0$. 
We only have to take a closer look at the case of equality in Eq.\ (\ref{eq:ineq}). According to the geometric interpretation given in Theorem \ref{regression1} the three points $\mathbf{k}'_4,\mathbf{k}'_5,\mathbf{k}'_6$ can only
collapse into one point if they are already collinear which contradicts our assumption that the given configuration is not singular. Therefore  Eq.\ (\ref{eq:ineq}) can be sharpened as follows:
\begin{equation}\label{eq:ineq2}
D_\star^{\vartriangle}(\mathbf{K},\mathbf{K}')<
D_\star^{\vartriangle}(\mathbf{K},\mathbf{K}'')\quad\text{for}\quad
\star\in\left\{\blacktriangle, \vartriangle\right\}.
\end{equation}
Analogous considerations to the above case yield the inequality
\begin{equation}\label{eq:ineq3}
D_{\vartriangle}^{\circ}(\mathbf{K},\mathbf{K}')<
D_{\vartriangle}^{\circ}(\mathbf{K},\mathbf{K}'')\quad\text{for}\quad\circ\in\left\{\blacktriangle, \vartriangle\right\},
\end{equation}
where $\mathbf{K}'$ is the global minimizer of 
Theorem \ref{regression2} and $\mathbf{K}''$ denotes the closest singular point of 
the collinearity variety $C_B=0$. 

Due to the validity of the Eqs.\ (\ref{eq:ineq2}) and 
(\ref{eq:ineq3}), we can abstain from computing the  configurations $\mathbf{K}''$ for the respective metrics.

If we restrict the base (resp.\ platform) to be transformed by Euclidean motions,  then our set of collinear configurations is only a subset $\tilde{C}_P$ (resp.\ $\tilde{C}_B$) of $C_P$ (resp.\ $C_B$). The singular points of these subsets have the same geometric characterization as given in Theorem \ref{col3}, which is proven next:
\begin{thm}\label{collinearbase}
The set of singular points of the collinearity variety  $\tilde{C}_{P}=0$ remains invariant under Euclidean motions. Moreover, these points are characterized by
$\mathbf{k}'_4=\mathbf{k}'_5=\mathbf{k}'_6$.
\end{thm}

\begin{proof}
The first sentence of the theorem can be done similarly to the 
proof of Theorem \ref{base}.

Therefore, we can proceed with the second part.

A singular point of $\tilde{C}_{P}=0$ is either a singular point of one of the four hypersurfaces  $C_P=0$, $E_B=0$, $F_1=0$, $F_2=0$ or it is a point, where the four tangent hyperplanes to these four hypersurfaces are linearly dependent. Algebraically this can be expressed by the set of equations resulting from the partial differentiation of 
    \begin{equation}
\lambda_0C_{P}+\lambda_{1} E_B+\lambda_{2}F_{1}+\lambda_{3}F_{2}
\label{main6}
\end{equation}
with respect to the 19 unknowns listed in Eq.\ (\ref{unknowns}). 
Due to the first part of the proof we can set $c_{1}=d_{1}=d_{2}=0$ in order to simplify the set of 19 equations, 
which is solved by using Gr\"obner basis method implemented in \texttt{Maple}.

The corresponding \texttt{Maple} file, which results in  $24$ solution sets, can be downloaded from ~\cite{codes}. Out of the obtained $24$ solution sets, there are $16$ solutions sets with $x_{2}=0$, which contradicts Lemma \ref{lem:app}. Moreover there are six trivial solution sets with $\lambda_{0}=\lambda_{1}=\lambda_{2}=\lambda_{3}=0$ (and $x_{2}\neq 0$). The remaining two solution sets correspond to singular points of $C_{P}=0$.
\end{proof}

Clearly, this theorem also holds by exchanging the platform and the base which yields:

\begin{thm}\label{collinearplatform}
The set of singular points of the collinearity variety  $\tilde{C}_{B}=0$ remains invariant under Euclidean motions. Moreover, these points are characterized by
 $\mathbf{k}'_1=\mathbf{k}'_2=\mathbf{k}'_3$.
\end{thm}

If the base or platform is made of undeformable material we have to compute the closest singular points $\mathbf{K}''$, on the collinearity varieties $\tilde{C}_{P}=0$ and $\tilde{C}_{B}=0$ of  Theorems \ref{collinearbase} and \ref{collinearplatform},  which can be done by solving the optimization problems formulated in the following two equations:

\begin{equation} \label{cc1}
 L=D_{\rectangleblack}^{\vartriangle}(\mathbf{K},\mathbf{K}'')^2+\mu E_B,
\end{equation}
where the three platform anchor points of $\mathbf{K}''$
degenerate to a point $\mathbf{k}'':=(c,d)^T$ and with 
$\mathbf{k}''_{3}$ according to Eq.~\ref{PBR1};
 \begin{equation} \label{cc2}
 L=D^{\rectangleblack}_{\vartriangle}(\mathbf{K},\mathbf{K}'')^2+\kappa E_p,
\end{equation}
where the three base anchor points of $\mathbf{K}''$
degenerate to a point $\mathbf{k}'':=(c,d)^T$ and with $\mathbf{k}''_{6}$ according to Eq.~\ref{PBR2}.
The number of unknowns for these two optimization problems is summarized in Table~\ref{spcv}.

\begin{table}[]
\caption{
Summary of the Lagrange formulations of the constrained optimization problems of Section \ref{collinearsing}.}
\centering
\begin{tabular}{|l||l|l|}
\hline
Extrinsic metric & unknowns in the Lagrangian & $\#$ unknowns \\ \hline
$D_{\rectangleblack}^{\vartriangle}(\mathbf{K},\mathbf{K}')$               &     $c_{1}, d_{1}, c_{2},d_{2}, c, d, \mu$                        &      $7$    \\ \hline
$D^{\rectangleblack}_{\vartriangle}(\mathbf{K},\mathbf{K}')$                 &   $c, d, c_{4},d_{4}, c_{5}, d_{5}, \kappa$                           &    $7$        \\ \hline
\end{tabular}
\label{spcv}
\end{table}

\section{Computational procedure}\label{results}
In this section, we introduce a pipeline for computing the closest singular configuration along a one-parametric motion, which is discretized into $n$ poses, as illustrated in Fig.\ \ref{inter}. We employ a numerical algebraic geometry algorithm implemented in the freeware \texttt{Bertini} to calculate generic finite solutions and enable user-defined homotopy. For each of the $n$ poses along the one-parametric motion, we compute critical points using \texttt{Paramotopy}~\cite{bates2018paramotopy}.
 The practical reasons for utilizing the software \texttt{Bertini} are discussed in \cite[Section 1]{kapilavai2020homotopy}. A detailed discussion of homotopy continuation algorithms and the functionality of the software is beyond the scope of this paper; for that, we refer readers to~\cite{BHSW06, sommese2005numerical}.

 \begin{figure}[t]
    \begin{center}
       \includegraphics[width=16cm]{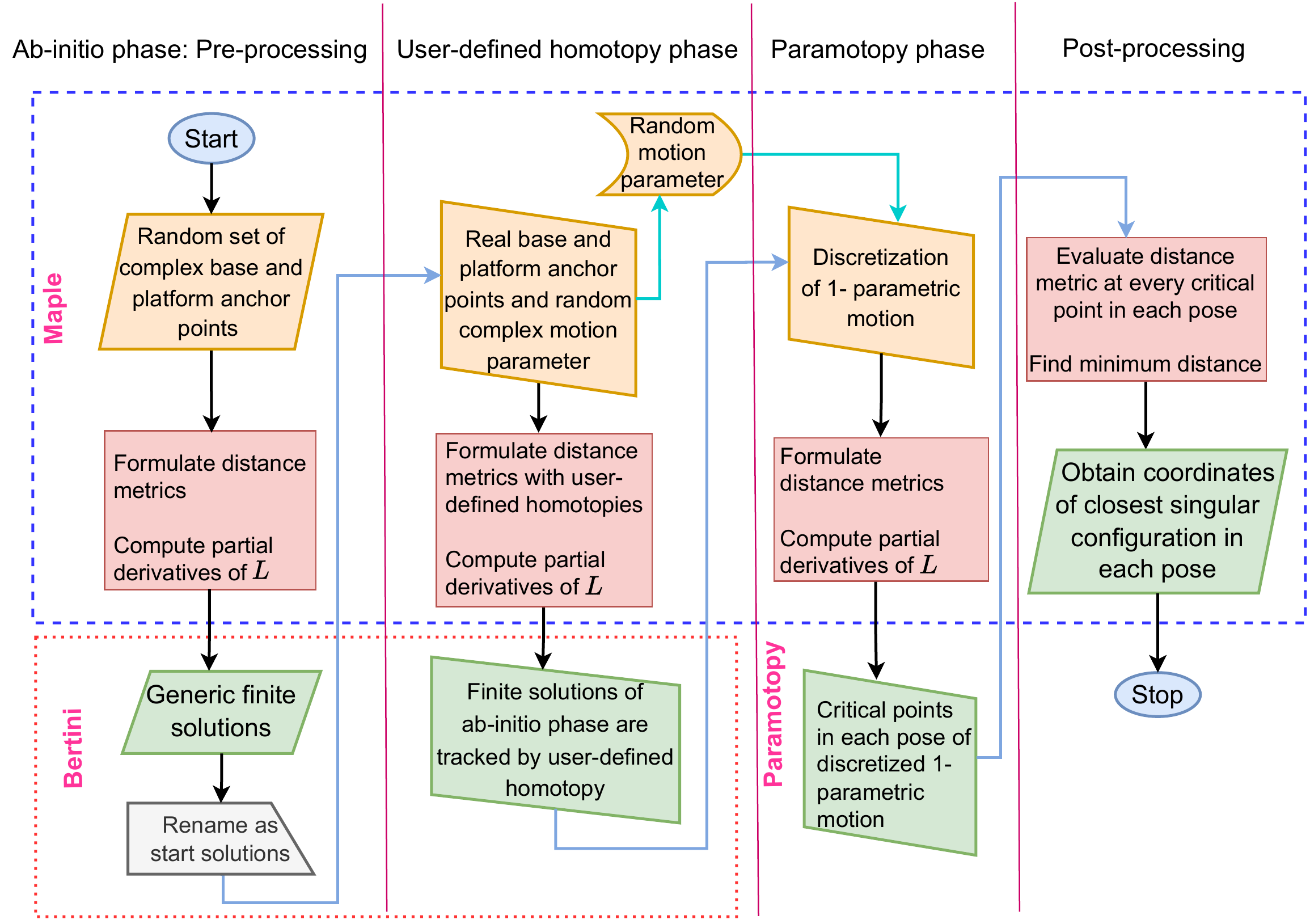}
      \caption{Operational flowchart illustrating the computational pipeline.}
      \label{flow}
    \end{center}
\end{figure}
The
computations were performed in parallel using a total of 64 threads using AMD Ryzen 7 2700X, 3.7 GHz processor. The computational procedure can be divided into four steps, which are discussed within the following subsections and are illustrated in Fig.~\ref{flow} of the operational flowchart.

\subsection{Step 1: Ab-Initio phase}\label{step1}  
We start with a generic framework by picking randomly all nine remaining
coordinates blue (besides $x_1=y_1=y_2=0$) with respect to the fixed frame from the set of complex numbers $\mathbb{C}$. Therefore, we denote the resulting vectors of the anchor points with respect to the fixed frame by $\mathbf{k}^{\mathbb{C}}_i$ for $i=1,\ldots,6$ and consequently the complex configuration by $\mathbf{K}^{\mathbb{C}}$, which is called   {\it source configuration.}
For this input, we want to find the critical points over $\mathbb{C}$,
 for each of the Lagrange optimization problems formulated in Section~\ref{sec:singvar}, Section~\ref{collinearity},  and Section~\ref{sec:para} and Section~\ref{collinearsing}. The corresponding unknowns in $L$  for the optimization problems are summarized in Table~\ref{extab1}, Table~\ref{singunknowns}, and Table~\ref{spcv}, respectively. As already mentioned the number of unknowns for the optimization problems presented in Section~\ref{collinearity} is the same as summarized in Table~\ref{extab1}.
 
Note that the partial derivatives of $L$ with the listed unknowns always result in the square system of non-homogeneous polynomial equations. The critical points obtained from the {\it ab-initio phase} are referred to as solutions to the {\it source configuration.} 
 
 As \texttt{Bertini} computes solutions to the system of polynomial equations with numerical approximations up to 16 digits (default), it is not possible to determine the exact root count beforehand without tracking all the paths. For the optimization problems presented in Sections \ref{sec:singvar}, \ref{collinearity}, \ref{sec:para}, and \ref{collinearsing}, we must identify the global minimizers since they correspond to the closest singular configuration. Therefore, it is necessary to track all the finite solutions over $\mathbb{C}$ without encountering any numerical errors (such as path failures) in the ab-initio phase.

The input file needed for \texttt{Bertini} to compute the solutions of the source configuration is pre-processed in \texttt{Maple}. It should be noted that the \texttt{Bertini} package offers various configuration settings in the input file to enable solution tracking without encountering path failures. Currently, there is no artificial intelligence module implemented in \texttt{Bertini1.6v} to suggest the necessary configuration settings for a user-provided input file to ensure solution tracking without numerical errors. Therefore, these settings must be manually fine-tuned to suit the specific problem. For a detailed understanding of configuration settings, please refer to~\cite[p. 301-321]{bates2013numerically}. 

Firstly, we need to check whether the solution sets of the presented optimization problems contain any higher-dimensional components. This can be accomplished in \texttt{Bertini} by setting \texttt{Tracktype:1} in the configuration settings of the input files while keeping the remaining settings as default. This analysis reveals that no higher-dimensional components exist for the source configuration.

As a result, we can proceed with the computation of the isolated solutions using the regeneration algorithm~\cite{hauenstein2011regeneration}, which is implemented in \texttt{Bertini1.6v}. It is worth noting that the regeneration algorithm is not implemented in HC.jl. This algorithm finds solutions to the systems by sequentially solving equation-by-equation but has the following limitation.

 \begin{limit}\label{regenaration}
According to \cite [p.~93]{bates2013numerically} using the regeneration algorithm, it is not guaranteed to obtain all the solutions as it discards singular solutions during the computation.
\end{limit}

As mentioned earlier, our goal is to obtain all generic finite solutions, including singular solutions. By employing symbolic methods, such as the G\"oberner basis method within \texttt{Maple}, we can reduce each optimization problem to solving a univariate polynomial. The degree of this polynomial corresponds to the root count, which enables us to verify whether all solutions were successfully tracked by the regeneration algorithm, performed with the configuration settings outlined in Table \ref{tab:reg}. The verification of this tracking is confirmed by the results presented in Tables \ref{extab3} to \ref{spcvresult}. Note that these generic solutions only have to be computed once for each of the different optimization problems.

\begin{table}[h!]
\caption{Configuration settings used for the regeneration algorithm in the ab initio phase.}
\centering
\begin{tabular}{|l|l||l|l|}
\hline
{Configuration settings} & value & {Configuration settings} & value\\ \hline
TRACKTolBEFOREEG      & 1e-8  &
TRACKTolDURINGEG      & 1e-8  \\ \hline
SliceTolBeforeEG      & 1e-8  &
SliceTolDuringEG      & 1e-8  \\ \hline
SECURITYLEVEL         & 1     &
UseRegeneration       & 1     \\ \hline
\end{tabular}\label{tab:reg}
\end{table}

\begin{table}[h!]
\caption{Summary of generic finite solutions from the ab-initio phase for the Lagrangians formulated in Section \ref{sec:singvar}.}
\centering
\begin{tabular}{|l||l|l|}
\hline
Extrinsic metric &
  \begin{tabular}[c]{@{}l@{}} $\#$ finite solutions \\ over $\mathbb{C}$ (\texttt{Bertini})\end{tabular} &
  \begin{tabular}[c]{@{}l@{}}degree of univariate \\ polynomial (\texttt{Maple})\end{tabular} \\ \hline
$D_{\rectangleblack}^{\rectangleblack}(\mathbf{K,K'})$ & 88 & 88 \\ \hline
$D_{\rectangleblack}^\circ(\mathbf{K},\mathbf{K}')$ with $\circ \in \{{\blacktriangle}, {\vartriangle}\}$ & 80 & 80 \\ \hline
$D^{\rectangleblack}_\star(\mathbf{K},\mathbf{K}')$ with $\star \in \{{\blacktriangle}, {\vartriangle}\}$  & 80  & 80  \\ \hline
$D^{\circ}_\star(\mathbf{K,K'})$ with $\circ,\star \in \{{\blacktriangle}, {\vartriangle}\}$ & 50 & 50  \\ \hline
\end{tabular}
\label{extab3}
\end{table}
\begin{table}[h!]
\caption{Summary of generic finite solutions from the ab-initio phase for the Lagrangians formulated in Section~\ref{collinearity}}
\centering
\begin{tabular}{|l||l|l|l|}
\hline
Extrinsic metric &
  \begin{tabular}[c]{@{}l@{}} $\#$ finite solutions\\over $\mathbb{C}$(\texttt{Bertini})\end{tabular} &
  \begin{tabular}[c]{@{}l@{}}degree of univariate \\ polynomial (\texttt{Maple})\end{tabular} \\ \hline
$D_{\rectangleblack}^{\vartriangle}(\mathbf{K},\mathbf{K}')$, $D_\vartriangle^{\rectangleblack}(\mathbf{K},\mathbf{K}')$  & $8 $ & 8 \\ \hline
$D_\star^{\vartriangle}(\mathbf{K},\mathbf{K}')$, $D_{\vartriangle}^{\circ}(\mathbf{K},\mathbf{K}')$  
with $\star, \circ\in\left\{\blacktriangle, \vartriangle\right\}$    & 2     &  2     \\  \hline
\end{tabular}
\label{collinearityresult}
\end{table}
\begin{table}[h!]
\caption{Summary of generic finite solutions from the ab-initio phase for the Lagrangians formulated in Section \ref{sec:para}}
\centering
\begin{tabular}{|l||l|l|}
\hline
Extrinsic metric & \begin{tabular}[c]{@{}l@{}}  $\#$ finite solutions\\over $\mathbb{C}$ \texttt({Bertini})\end{tabular} & \begin{tabular}[c]{@{}l@{}}degree of univariate \\ polynomial (\texttt{Maple})\end{tabular} \\ \hline
$D_{\hrectangleblack}^{\hrectangleblack}(\mathbf{K,K'})$  & 8  & {8} \\ \hline
$D_{\rectangleblack}^\blacktriangle(\mathbf{K},\mathbf{K}')$, $D_{\blacktriangle}^{\rectangleblack}(\mathbf{K},\mathbf{K}')$  & $8$ & {8}  \\ \hline
$D_{\blacktriangle}^\blacktriangle(\mathbf{K},\mathbf{K}')$  & $8$ & {8}  \\ \hline
\end{tabular}
\label{spsvresult}
\end{table}
\begin{table}[h!]
\caption{Summary of generic finite solutions from the ab-initio phase for the Lagrangians formulated in Section \ref{collinearsing}}
\centering
\begin{tabular}{|l|l|l|}
\hline
Extrinsic metric &   \begin{tabular}[c]{@{}l@{}}  $\#$ finite solutions\\over $\mathbb{C}$\texttt({Bertini})\end{tabular}  & \begin{tabular}[c]{@{}l@{}}degree of univariate \\ polynomial (\texttt{Maple})\end{tabular} \\ \hline
$D_{\rectangleblack}^{\vartriangle}(\mathbf{K},\mathbf{K}')$,  $D^{\rectangleblack}_{\vartriangle}(\mathbf{K},\mathbf{K}')$                   & $2$
& $2$ \\ \hline
\end{tabular}
\label{spcvresult}
\end{table}
\subsection{Step 2: User-defined homotopy phase}\label{step2}
The input for this step is the geometry of the manipulator given by the real values {$x_2,x_3,y_3,x_5,x_6,y_6$} and a 1-parametric motion with parameter $\phi\in[v;w]$; i.e.\ 
\begin{equation}\label{motion_phi}
\mathbf{k}_{j}(\phi) =
\mathbf{R}(\phi) \mathbf{p}_j+ \mathbf{t}(\phi)
\quad \text{for} \quad j=4,5,6 
\end{equation}
according to Eq.~(\ref{initial}).
We consider one generic random complex pose of this motion, referred to as {\it seed configuration}, by setting 
\begin{equation}
  \phi^\mathbb{C}= v-(v-w)(1-\alpha^\mathbb{C})
\label{complex}
\end{equation}
where $\alpha^\mathbb{C}$ is a randomly chosen complex number. Now we define a linear homotopy between the vector coordinates of the ab-initio configuration $\mathbf{k}_{i}^\mathbb{C}$ and 
the seed configuration by
\begin{equation}
\begin{split}
\mathbf{k}_{i}^\mathbb{C} &\mapsto 
{{\mathbf{k}}}_{i} + m\left(\mathbf{k}_{i}^\mathbb{C} - {{\mathbf{k}}}_{i}\right)
\quad \text{for} \quad j=1,2,3 \\
\mathbf{k}_{i}^\mathbb{C} &\mapsto 
{{\mathbf{k}}}_{i}(\phi^\mathbb{C}) + m\left(\mathbf{k}_{i}^\mathbb{C} - {{\mathbf{k}}}_{i}(\phi^\mathbb{C})\right)
\quad \text{for} \quad j=4,5,6
\label{homotopy}
\end{split}
\end{equation}
where $m$ is the homotopy parameter. 
By applying this mapping to the polynomial systems of the ab-initio phase, we are able to track the finite solutions of the \textit{source configuration} 
to the finite solutions of the \textit{seed configuration} while $m$ is running from $1$ to $0$. 

\subsection{Step 3: Paramotopy phase}\label{step3} The utilization of the \texttt{Paramotopy} software, as described in \cite{bates2018paramotopy}, involves two steps. In the first step, \texttt{Paramotopy} calls \texttt{Bertini} to solve the system of polynomial equations at a generic parameter point. Subsequently, in the second step, \texttt{Paramotopy} tracks these solutions to all desired parameter values by again calling \texttt{Bertini}.

With our seed configuration and its solutions computed in Section \ref{step2}, we can proceed to the second step of \texttt{Paramotopy}. This step requires the additional parameters $\alpha^\mathbb{C}$ and the number $n$ of poses (which can be defined by the user) in which the one-parametric motion described by Eq.\ (\ref{motion_phi}) should be discretized. Specifically, the interval $\alpha\in[0;1]$ is evenly subdivided into $n$ values. While $\alpha$ traverses this interval, \texttt{Paramotopy} tracks the corresponding solutions for the $n$ poses.

During the execution of Step 2 of \texttt{Paramotopy}, there may still be path failures for each discretized pose. To address these failures, it is necessary to re-run the failed paths by adjusting the configuration settings.

\subsection{Step 4: Post-processing}\label{postprocess} In this phase, the obtained real solutions for $n$ discretized poses are given as input to \texttt{Maple}. For the corresponding configuration, the distance function is evaluated in order to filter out the global minimizer.

\subsection{Interface}
To operate this pipeline, we developed an open-source software interface between \texttt{Maple} and \texttt{Bertini} as well as
\texttt{Paramotopy}, thus that all calls can be made within \texttt{Maple} rather than by switching between the systems. 
Note that the user needs to run only steps 2--4. 

\begin{rmk}
Steps 2 and 3 of the above given computational pipeline can also be replaced by using both steps of \texttt{Paramotopy}~\cite{bates2018paramotopy}. The drawback is that one always needs to run the first step of \texttt{Paramotopy} which is substantially longer, as it can be seen from  Table~\ref{runtime}, which lists the computation times for the numerical example discussed in the next section. \hfill $\diamond$
\end{rmk}

\begin{rmk}
If one is interested in computing the closest singular configuration for a 3-RPR manipulator, then the generic finite solutions together with the complex coordinates chosen in Section ~\ref{step1} are provided as supplementary material~\cite{codes}. 
 In this case, it is sufficient to run only Sections \ref{step2}, \ref{step3} and \ref{postprocess} of the computation. 
\hfill $\diamond$
\end{rmk}

\section{Results and discussion}\label{numericalexample}
The presented computational procedure for determining singularity distances is illustrated by the following numerical example, which was also used for the singularity distance computations with respect to intrinsic metrics  \cite{intrinsic}. The base and platform anchor points and the one parametric motion are given by: 

\begin{equation}
\mathbf{k}_{1}=\mathbf{p}_{1}=(0,0), \quad \mathbf{k}_{2}=(11,0), \quad \mathbf{p}_{2}=(3,0), \quad \mathbf{k}_{3}=(5,7),   \quad \mathbf{p}_{3}=(1,2)
\label{design}
\end{equation}
and
\begin{equation}
\mathbf{k}_{i+3} =
\underbrace{\begin{pmatrix}
\cos{\phi} & -\sin{\phi} \\
\sin{\phi} & \cos{\phi}
\end{pmatrix}}_\mathbf{R}
\mathbf{p}_{i}
+
\underbrace{\frac{1}{2}
\begin{pmatrix}
{11}-6\sin{\phi} \\
{3}-{3}\cos{\phi}
\end{pmatrix}}_\mathbf{t} \text{for} \quad {i}={1,2,3}.
\label{parameter}
\end{equation}
Note that the one-parametric motion of the manipulator exhibits two singular configurations at $\phi=0$ and $\phi\approx 3.0356972$ radians. The interval for the  motion parameter $\phi$ is chosen between 
$[0;2\pi]$; i.e.\ $v=0$ and $w=2\pi$.
Following Eq.(\ref{complex}), we generate a random complex number $\alpha^\mathbb{C}$ using floating-point arithmetic and express the pose of the platform with respect to the base frame using Eq.(\ref{parameter}). By defining user-defined homotopies given by Eq.~(\ref{homotopy}), we obtain the system of polynomial equations depending on the homotopy parameter $m$. We invoke \texttt{User homotopy:2} 
in the input files, to avoid tracking infinitely long paths (for details see~\cite[Sec. 7.5.2]{bates2013numerically}). We use generic finite solutions of the \textit{source configuration} 
as \texttt{start solutions} and track them to the \textit{seed configuration} without any path failures. Note that the number of paths to be tracked equals the generic finite root count given in Tables \ref{extab3}--\ref{spcvresult}. For the  \texttt{Paramotopy} phase we set $n=90$; i.e.\ we discretized the 1-parametric motion into $90$ poses. The resulting solution set for each pose is post-processed according to Step 4. 
A comparison of the average  CPU run time\footnote{In this context we are referring to \textit{response time}.} (over 5 runs)  between both steps of \texttt{Paramotopy} and Step 2 and Step 3 of the proposed algorithm is given in Table \ref{runtime}. 

\begin{table}[t]
\caption{A comparison of average run time (in seconds) for Step 2 and Step 3 and both steps of \texttt{Paramotopy}.} 
\centering
\begin{tabular}{|l||l|l|l|l|}
\hline
Extrinsic metric & A) Step 2 & B) Step 3 & C) both steps of \texttt{Paramotopy} & Ratio $C/(A+B)$\\ \hline
$D_{\rectangleblack}^{\rectangleblack}(\mathbf{K},\mathbf{K}')$             &  $\approx 2.160$     &    $\approx 236.861$    &  $\approx 651.200$     &  $\approx 2.152$   \\ \hline
 $D_{\rectangleblack}^{\blacktriangle}(\mathbf{K},\mathbf{K}')$    &  $\approx 4.074$   &    $\approx 263.130$       &  $\approx 5\,100.240$    &    $\approx 19.087$   \\ \hline
  $D_{\rectangleblack}^{\vartriangle}(\mathbf{K},\mathbf{K}')$             &    $\approx2.860$          &  $\approx 292.250$     &  $\approx 5\,563.370$     &  $\approx 18.851$    \\ \hline
  $D^{\rectangleblack}_{\blacktriangle}(\mathbf{K},\mathbf{K}')$                 &       $\approx 4.258$   &     $\approx 315. 413$   &   $\approx 5\,152.523$    & $\approx 17.244$   \\ \hline
$D^{\blacktriangle}_{\blacktriangle}(\mathbf{K,K'})$        &  $\approx 1.640$     &    $\approx 104.636$    &    $\approx 6\,595.126$      &  $\approx 62.056$  \\ \hline
$D^{\vartriangle}_{\blacktriangle}(\mathbf{K,K'})$                  &     $\approx1.056$   &   $\approx 160.260$     & $\approx 5\,915.350$     &  $\approx 36.669$    \\ \hline
  $D^{\rectangleblack}_{\vartriangle}(\mathbf{K},\mathbf{K}')$       &    $\approx3.686$    &   $\approx 276.653$     &   $\approx 5\,548.023$     &   $\approx 19.790$    \\ \hline
 $D_{\vartriangle}^{\blacktriangle}(\mathbf{K,K'})$                 &  $\approx 1.280$     &     $\approx 147.720$
   &    $\approx 6\,063.070$   &   $\approx $40.691   \\ \hline
 $D^{\vartriangle}_{\vartriangle}(\mathbf{K,K'})$                &   $\approx 1.778$      &   $\approx 113.292$     &    $\approx 6\,731.132 $    &   $\approx 58.495$  \\ \hline
\end{tabular}
\label{runtime}
\end{table}

In Fig.\ \ref{measure_sing} we compare the distances to the singularity variety $V=0$ with respect to the different extrinsic metrics 
$D_\star^\circ(\mathbf{K},\mathbf{K}')$ with $\circ,\star\in\left\{\blacktriangle, \vartriangle, \hrectangleblack\right\}$  to the preliminary one  
$D^{\bullet}_{\bullet}(\mathbf{K},\mathbf{K}')$ of Eq.\ (\ref{eq:metric0}). 
Note that due to the discretization of the motion into $90$ poses by \texttt{Paramotopy}, the linearly interpolated value at the singularity $\phi\approx 3.0356972$  radians is not exactly zero. 
In Fig.~\ref{singulardemo} the closest singular configurations on $V=0$ are illustrated for the pose $\phi= 0.8471710528$, which is indicated by the black dashed line in Fig.~\ref{measure_sing}.
For the metric $D_{\bullet}^{\bullet}(\mathbf{K,K'})$ we refer to Fig.\ \ref{measure}.

\begin{figure}[h!]
    \begin{center}
    \begin{overpic}[width=14cm]{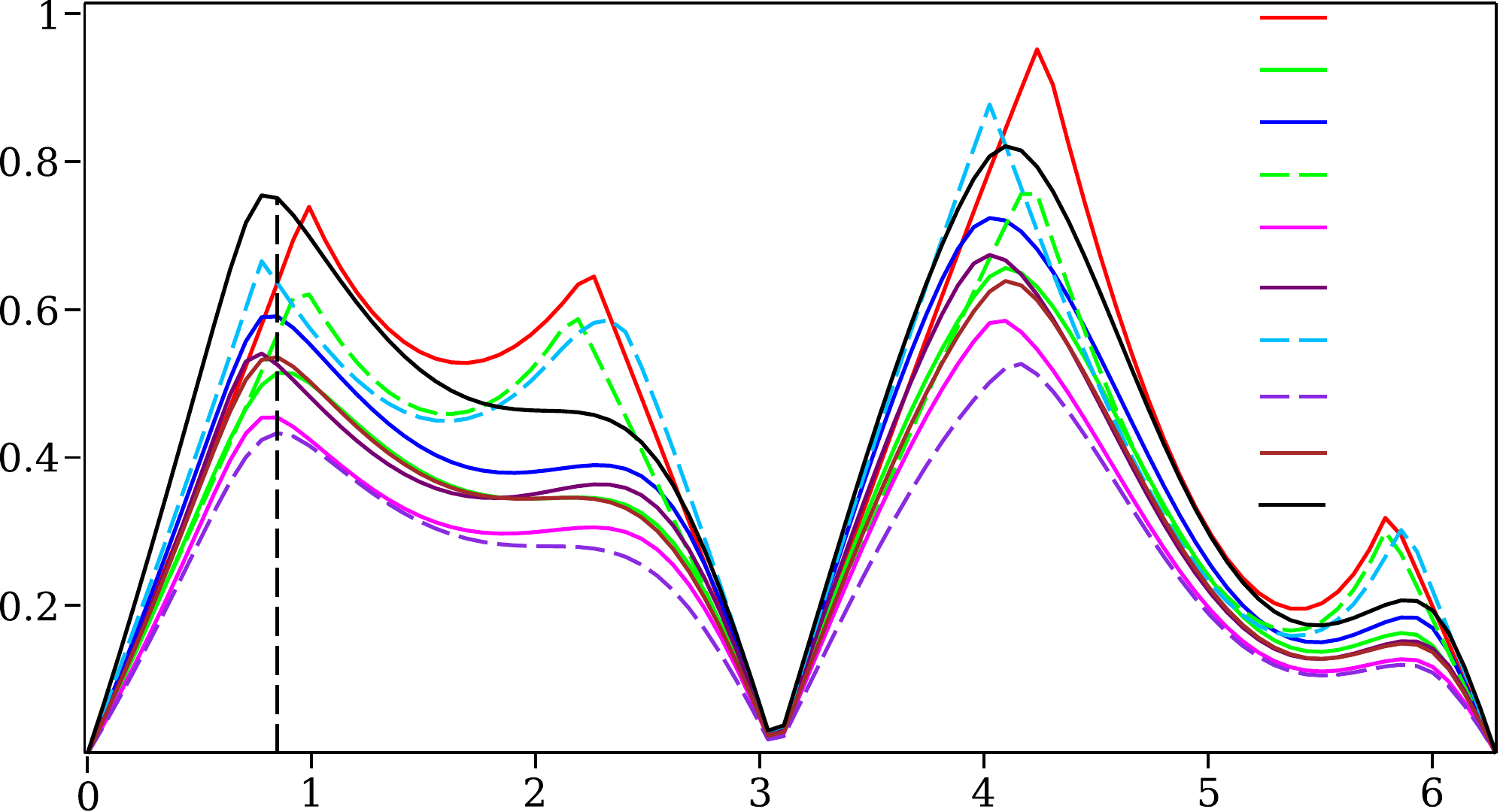}
     \put(85,1.3){$\phi$ (rad)} 
     \begin{footnotesize}
      \put(88.5,52) {$D_{\rectangleblack}^{\rectangleblack}(\mathbf{K,K'})$ } 
         \put(88.5,48.5) {$D_{\rectangleblack}^{\blacktriangle}(\mathbf{K,K'})$ } 
         \put(88.5,45) {$D_{\rectangleblack}^{\vartriangle}(\mathbf{K,K'})$ }
         \put(88.5,41.5) {$D^{\rectangleblack}_{\blacktriangle}(\mathbf{K,K'})$ }
          \put(88.5,38) {$D^{\blacktriangle}_{\blacktriangle}(\mathbf{K,K'})$ }
          \put(88.5,34.5) {$D_{\blacktriangle}^{\vartriangle}(\mathbf{K,K'})$ }
          \put(88.5,31) {$D^{\rectangleblack}_{\vartriangle}(\mathbf{K,K'})$ }
          \put(88.5,27) {$D^{\blacktriangle}_{\vartriangle}(\mathbf{K,K'})$ }
           \put(88.5,23) {$D_{\vartriangle}^{\vartriangle}(\mathbf{K,K'})$ }
           \put(88.5,20) {$D_{\bullet}^{\bullet}(\mathbf{K,K'})$ }
     \end{footnotesize}
       \put(-2.2,30){\makebox(0,0){\rotatebox{90}{Distance}}}
        \end{overpic}
     \caption{Distances to the singularity variety $V=0$  with respect to the different extrinsic metrics.}
     \label{measure_sing}
    \end{center}
\end{figure}
\begin{figure}[h!]
    \begin{center}
       \includegraphics[width=14.5cm]{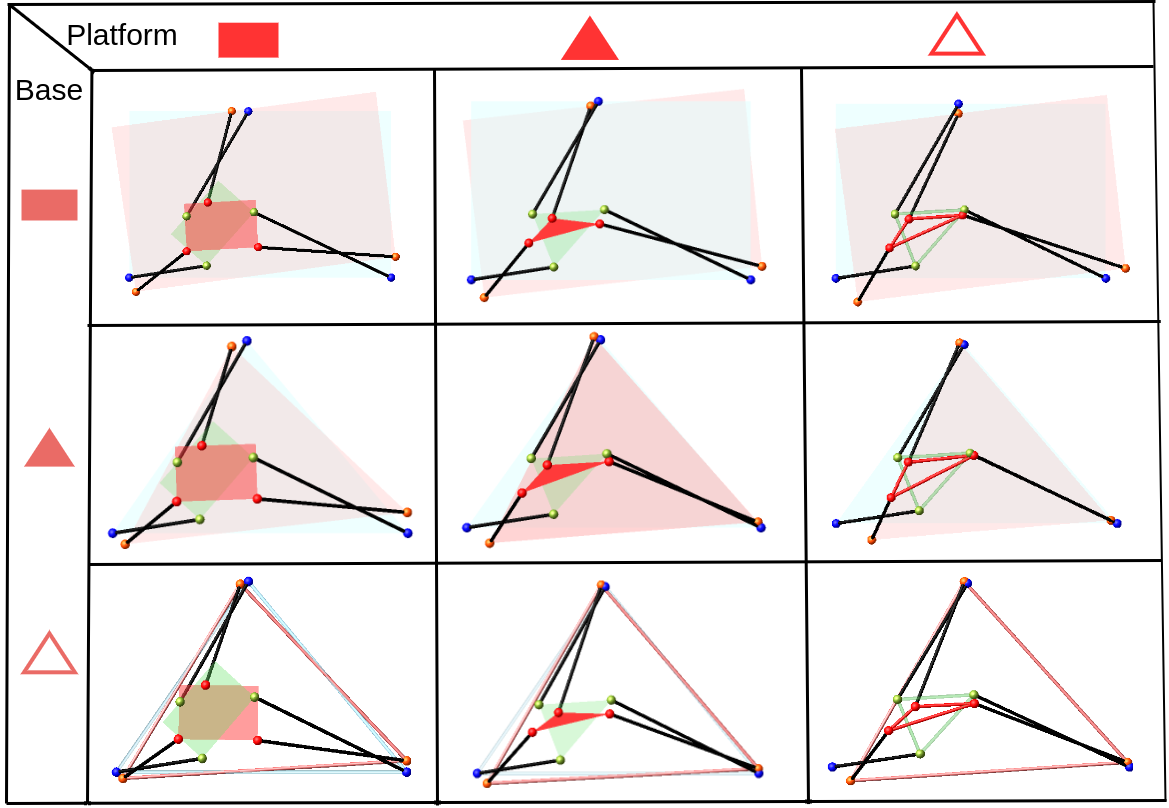}
      \caption{Closest singular configurations on $V=0$ for $\phi= 0.8471710528$ radians.}
      \label{singulardemo}
    \end{center}
\end{figure}

For the interpretations possessing a pin-jointed triangular bar structure in the platform or base, the distances to the singularity variety and collinearity variety are compared in Fig.~\ref{fig:result3}.
 In Fig.~\ref{fig:my_label} the corresponding closest configurations on the collinearity varieties are again illustrated for the pose $\phi= 0.8471710528$.

\begin{figure}[h!]
\begin{center}
\begin{overpic}
    [width=50mm]{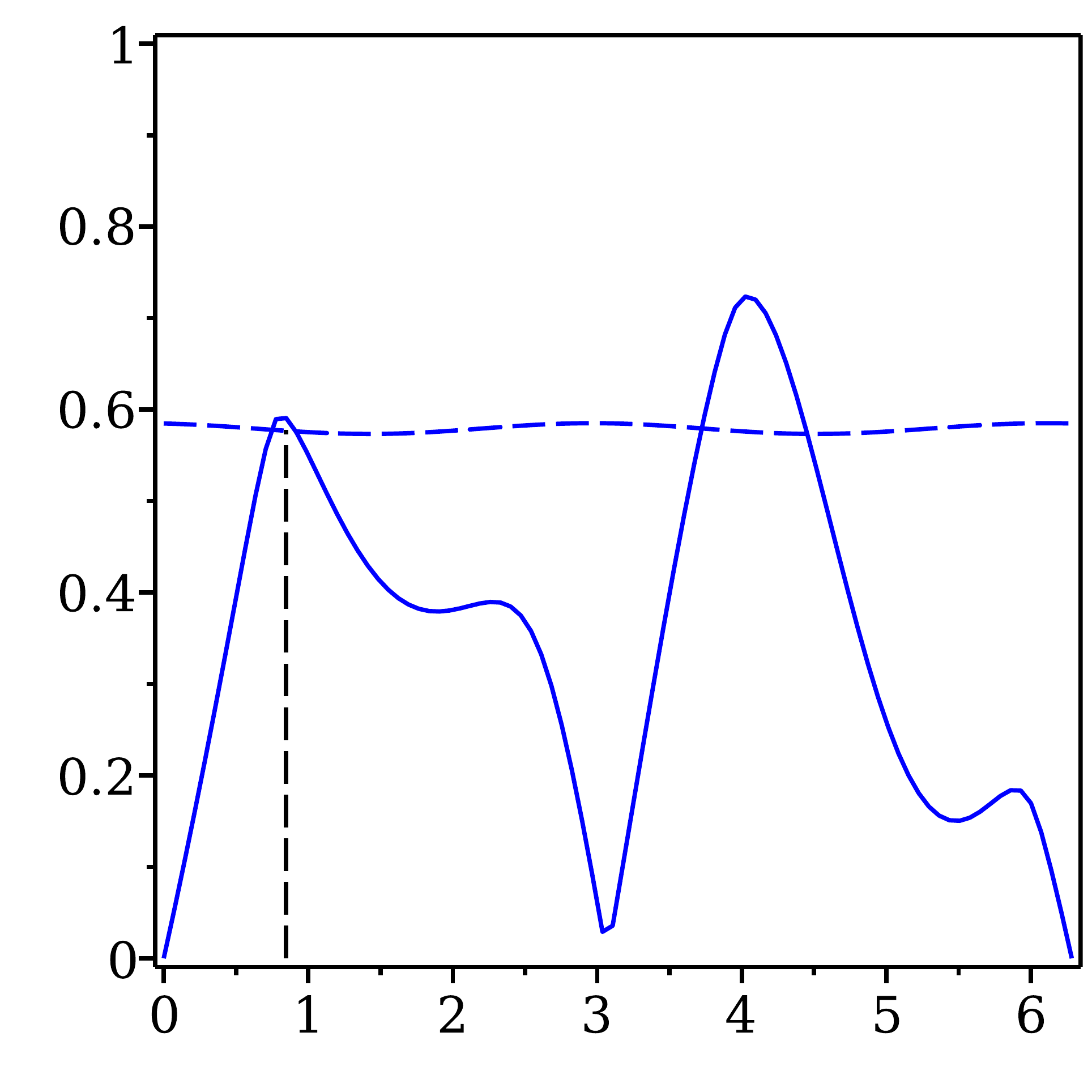}
    \begin{small}
     \put(35,-3) {(a)  $D_{\rectangleblack}^{\vartriangle}(\mathbf{K,K'})$}
      \end{small}
        \begin{scriptsize}   
   \put(0,50){\makebox(0,0){\rotatebox{90}{Distance}}}
   \put(85,1.5){$\phi$ (rad)}
   \end{scriptsize}
\end{overpic}
\quad
\begin{overpic}
   [width=50mm]{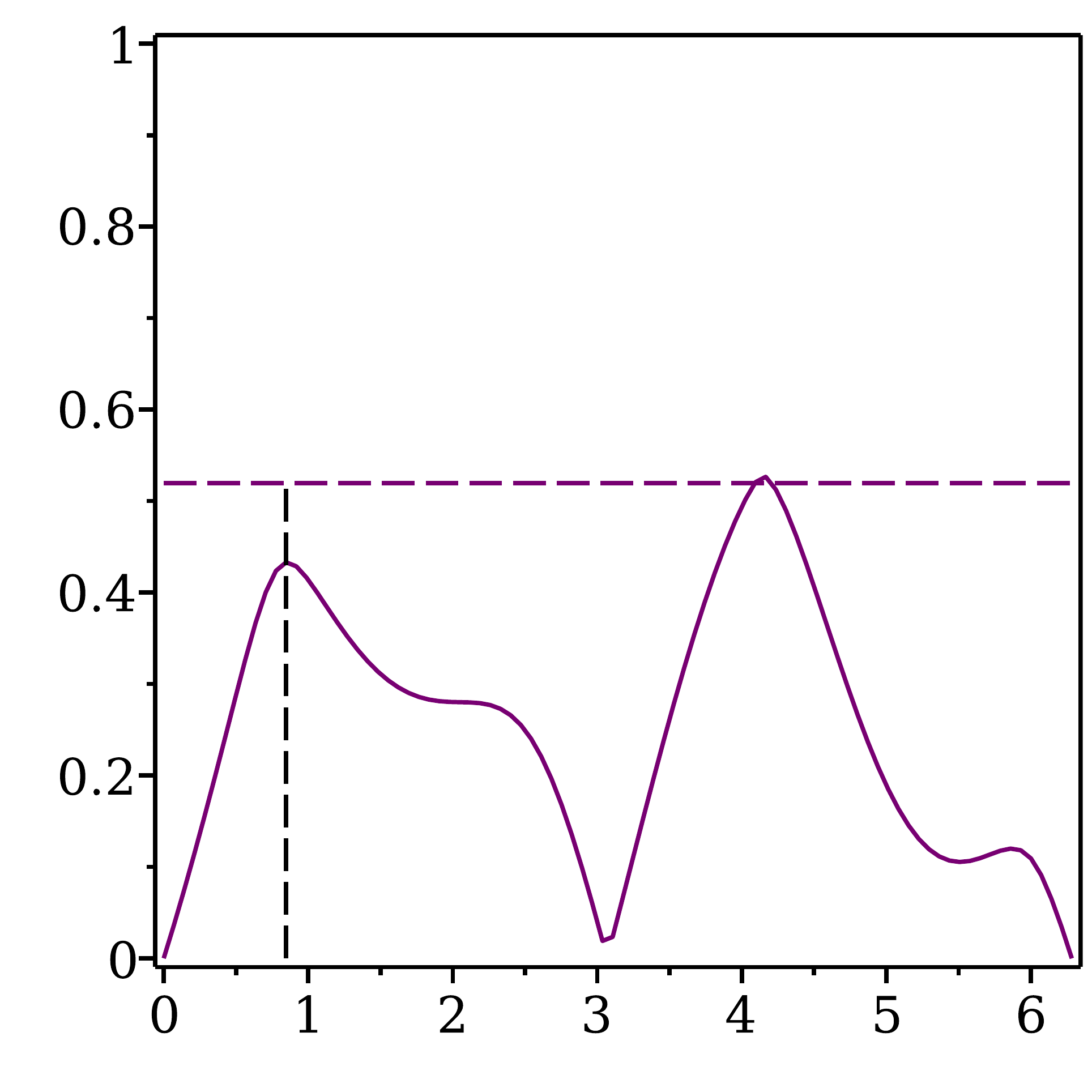}
     \begin{small}
     \put(35,-3) {(b) $D_{\blacktriangle}^{\vartriangle}(\mathbf{K,K'})$ }
     \end{small}
       \begin{scriptsize}   
   \put(0,50){\makebox(0,0){\rotatebox{90}{Distance}}}
   \put(85,1.5){$\phi$ (rad)}
   \end{scriptsize}
\end{overpic}
\quad 
\begin{overpic}
    [width=50mm]{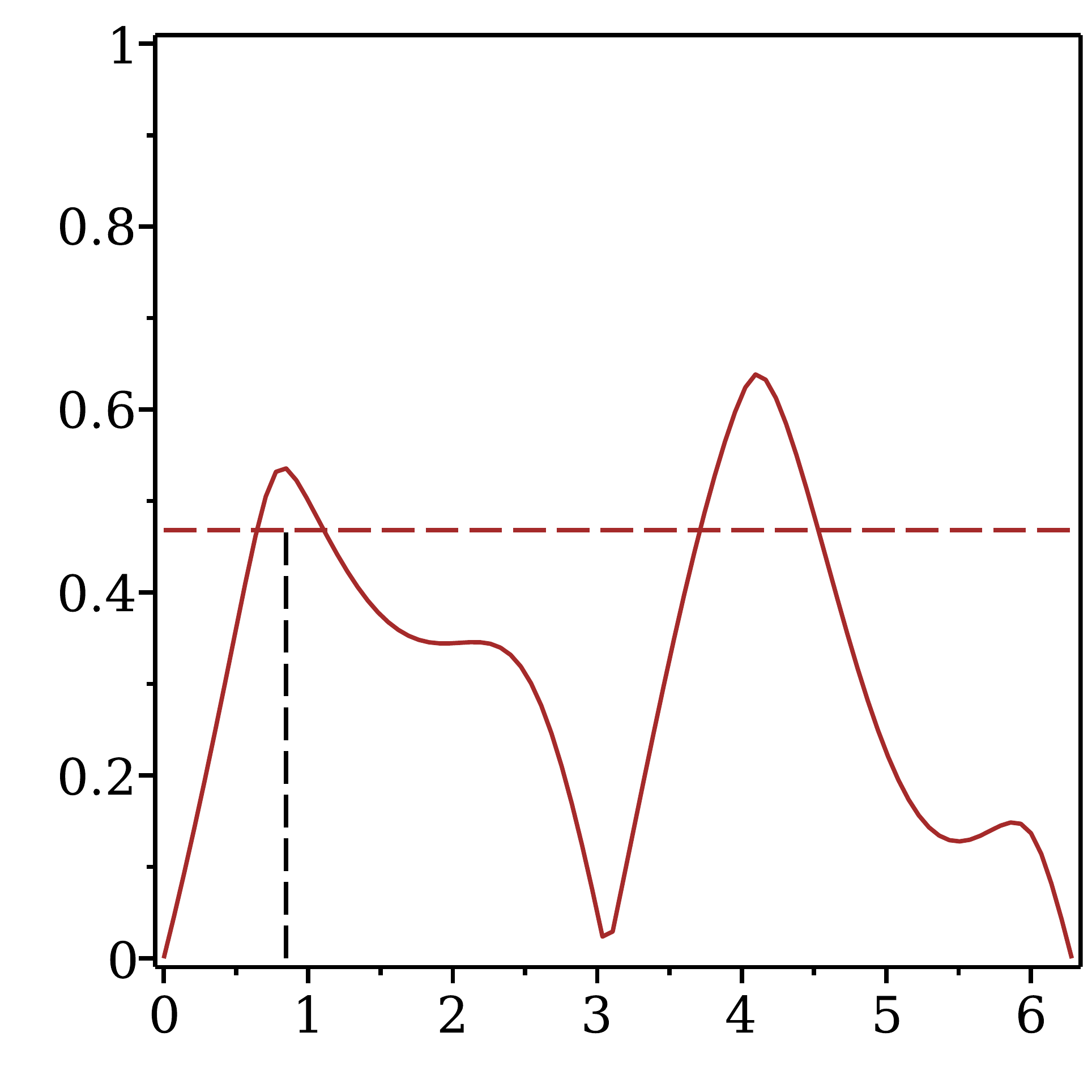}
    \begin{small}
        \put(30,-3.5){(c) $D_{\vartriangle}^{{\text{coll}(\vartriangle)}}(\mathbf{K,K'})$}  
    \end{small}
       \begin{scriptsize}   
   \put(0,50){\makebox(0,0){\rotatebox{90}{Distance}}}
    \put(86,1.5){$\phi$ (rad)}
   \end{scriptsize}
\end{overpic} 
\\ \phantom{x} \\
\begin{overpic}
    [width=50mm]{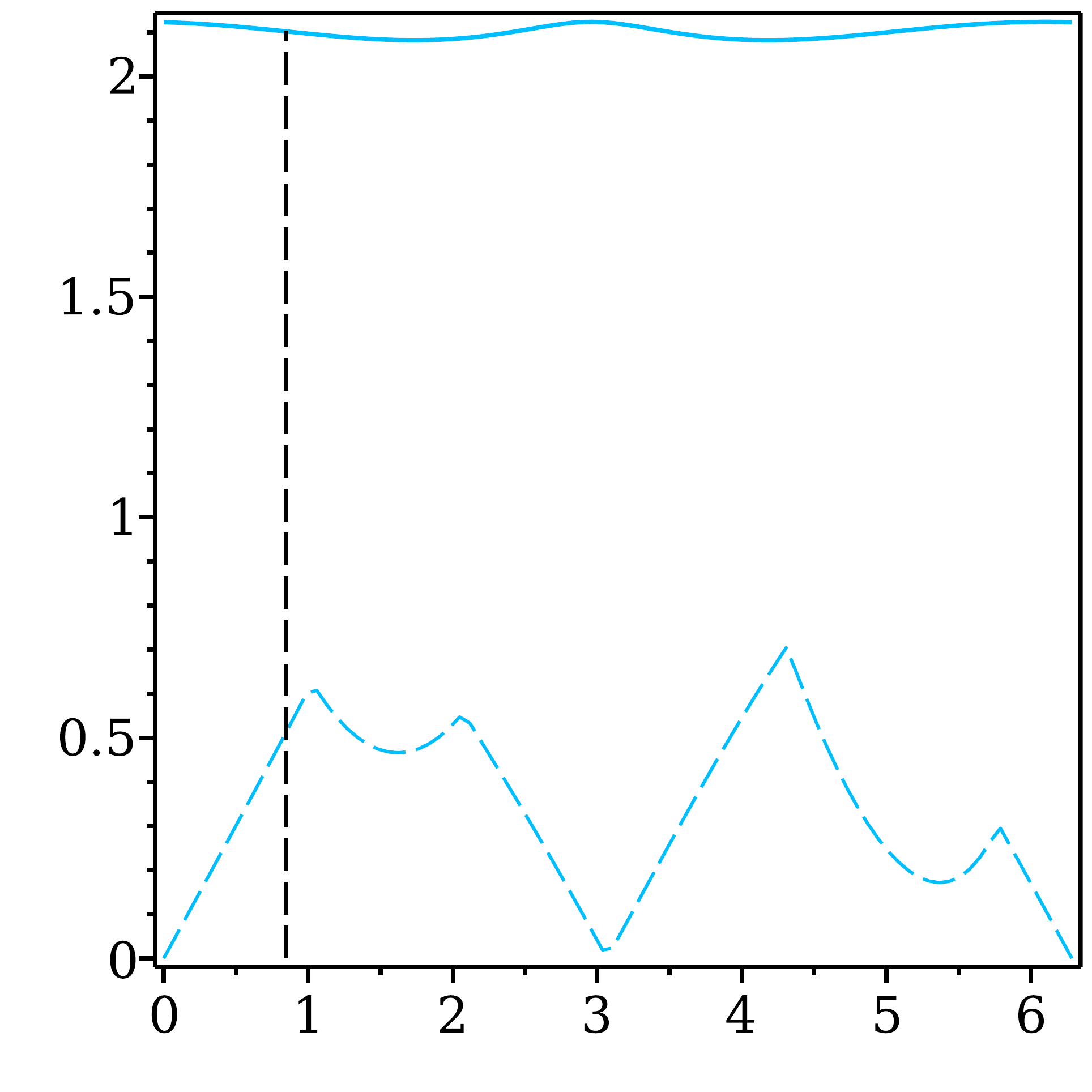}
    \begin{small}
    \put(35,-3) {(d) $D_{\vartriangle}^{\rectangleblack}(\mathbf{K,K'})$}
    \end{small}
      \begin{scriptsize}   
   \put(0,50){\makebox(0,0){\rotatebox{90}{Distance}}}
    \put(85,1.5){$\phi$ (rad)}
   \end{scriptsize}
\end{overpic}
\quad
\begin{overpic}
    [width=50mm]{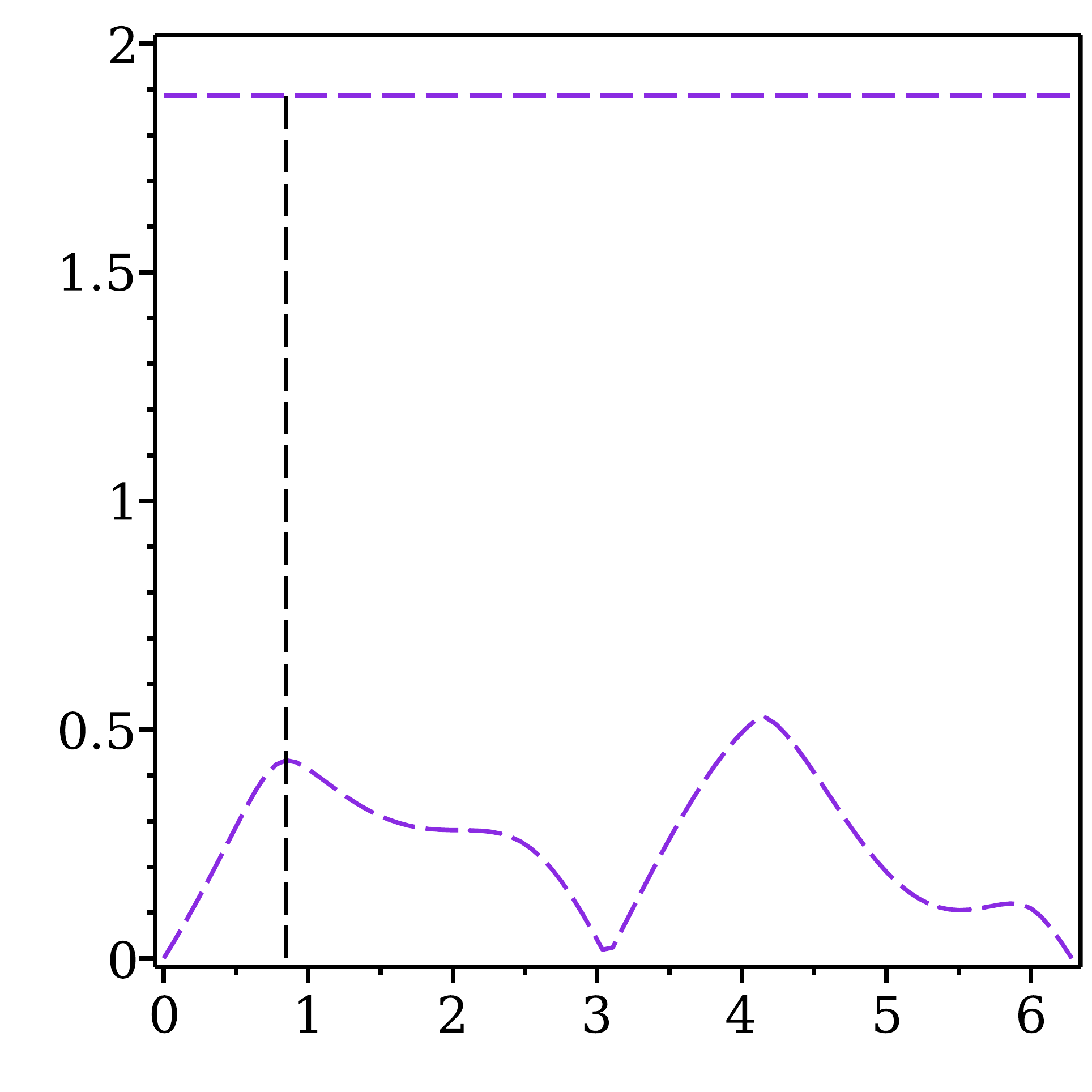}
    \begin{small}
      \put(35,-3) {(e) $D^{\blacktriangle}_{\vartriangle}(\mathbf{K,K'})$ }  
    \end{small}
  \begin{scriptsize}   
   \put(0,50){\makebox(0,0){\rotatebox{90}{Distance}}}
    \put(85,1.5){$\phi$ (rad)}
   \end{scriptsize}  
\end{overpic}
\quad
\begin{overpic}
   [width=50mm]{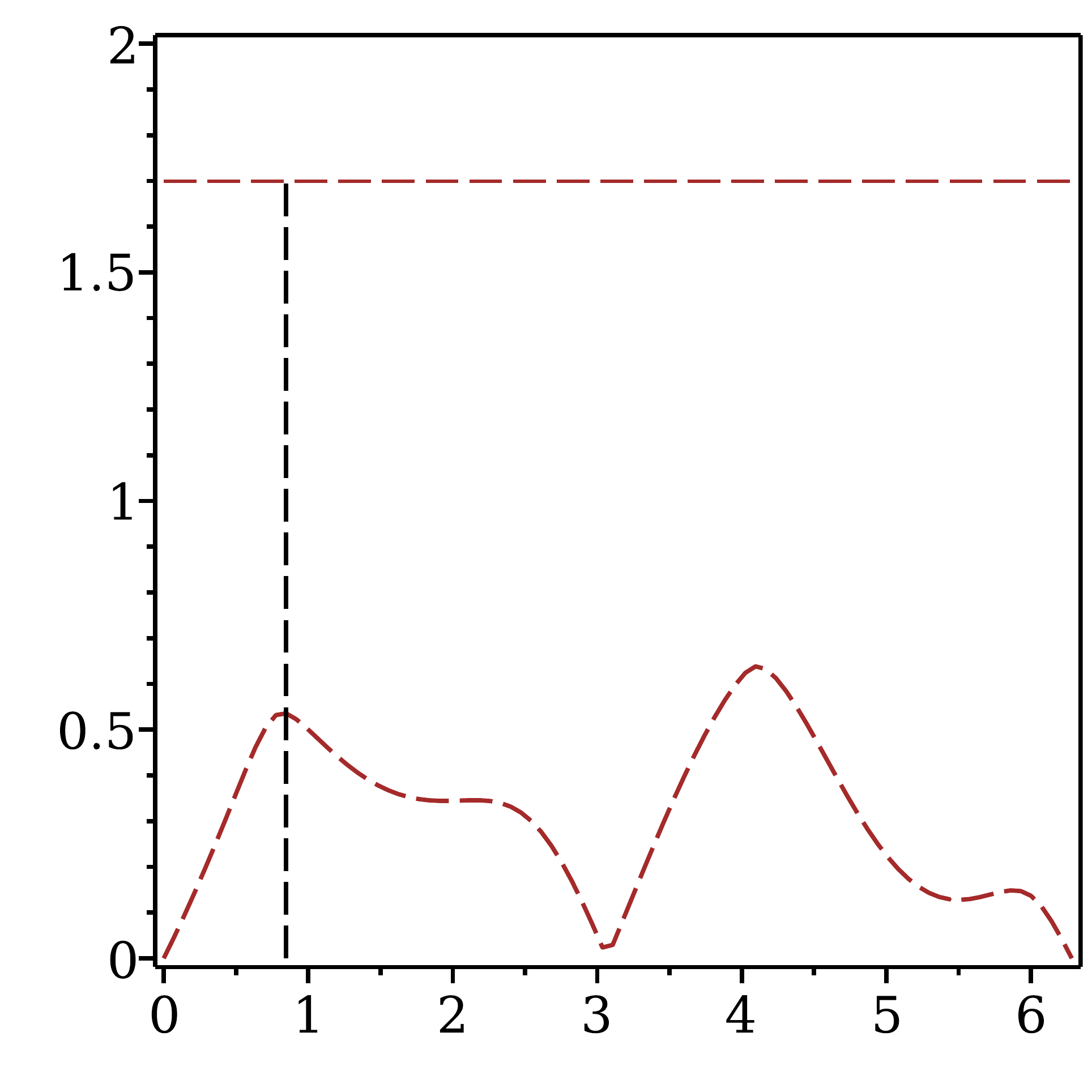}
    \begin{small}
      \put(30,-3) {(f) $D^{\vartriangle}_{\text{coll}(\vartriangle)}(\mathbf{K,K'})$ }  
    \end{small}
    \begin{scriptsize}   
   \put(0,50){\makebox(0,0){\rotatebox{90}{Distance}}}
    \put(85,1.5){$\phi$ (rad)}
   \end{scriptsize}
\end{overpic} 
\end{center}
 \caption{Comparison of distances to the singularity variety $V=0$ and the collinearity variety  $C_P=0$ (top) and $C_B=0$ (bottom), respectively.}
 \label{fig:result3}
  \end{figure}   

\begin{figure}[h!]
\begin{center}
\begin{overpic}
    [width=45mm]{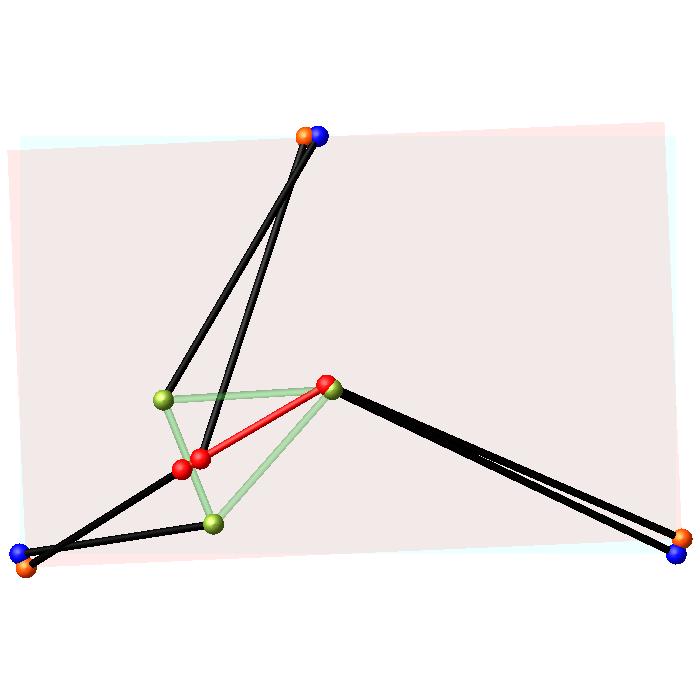}
    \begin{small}
     \put(35,10) {(a)  $D_{\rectangleblack}^{\vartriangle}(\mathbf{K,K'})$}
      \end{small}
\end{overpic}
\quad
\begin{overpic}
    [height=45mm]{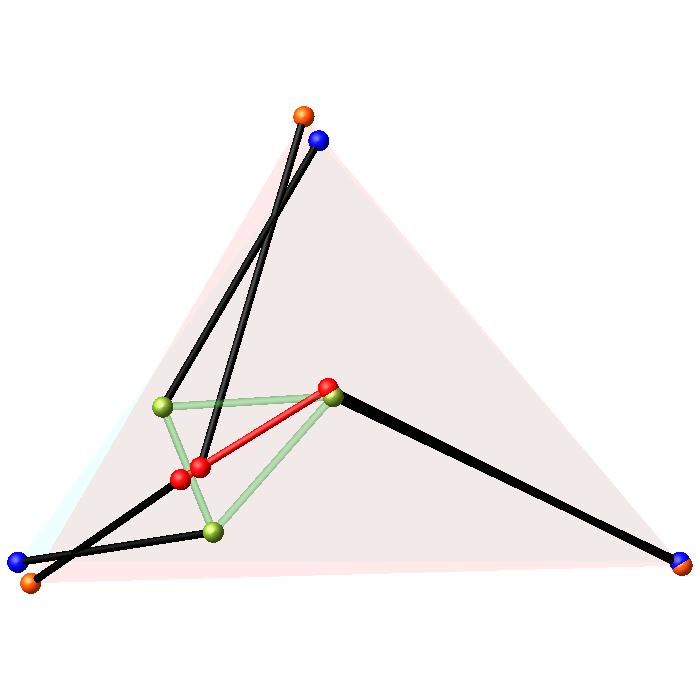}
     \begin{small}
     \put(35,10) {(b) $D_{\blacktriangle}^{\vartriangle}(\mathbf{K,K'})$ }
     \end{small}
\end{overpic}
\quad
\begin{overpic}
    [height=45mm]{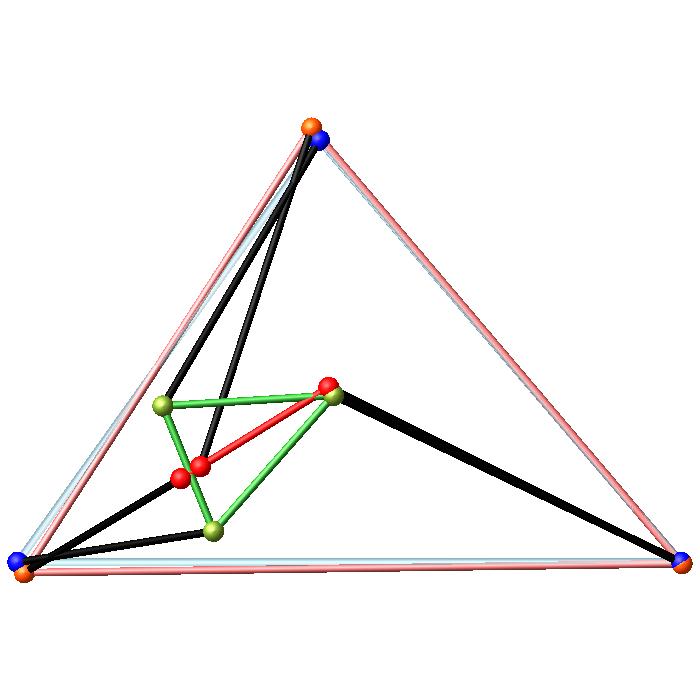}
    \begin{small}
        \put(35,10){(c) $D_{\vartriangle}^{\vartriangle}(\mathbf{K,K'})$}  
    \end{small}
\end{overpic} 
\newline
\begin{overpic}
    [height=45mm]{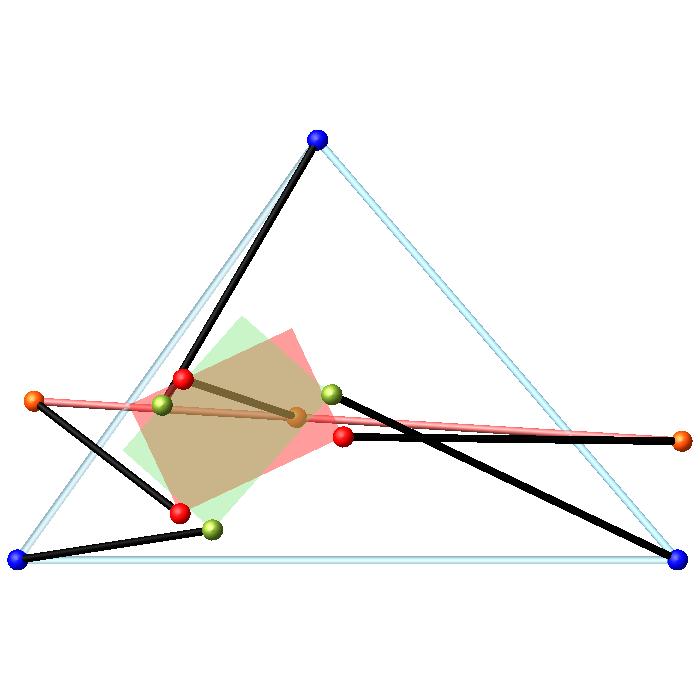}
    \begin{small}
    \put(35,10) {(d) $D_{\vartriangle}^{\rectangleblack}(\mathbf{K,K'})$}
    \end{small}
\end{overpic}
\quad
\begin{overpic}
    [height=45mm]{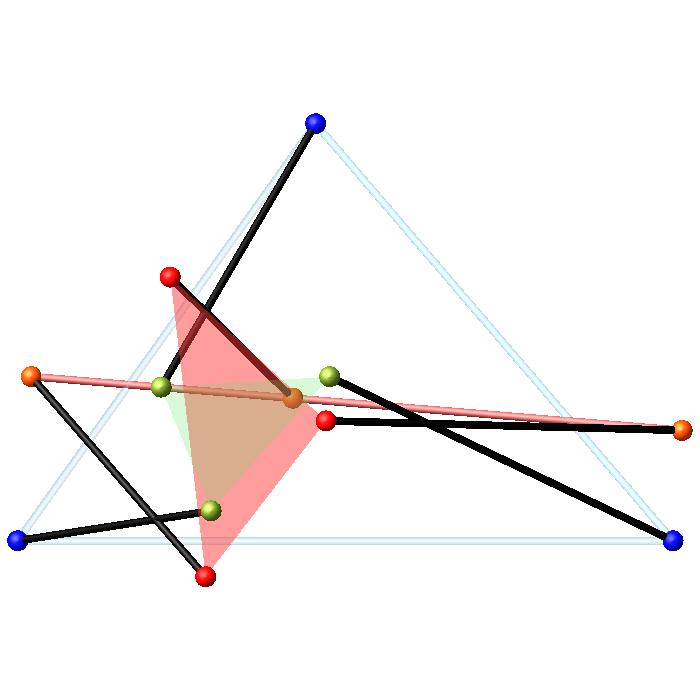}
    \begin{small}
      \put(35,10) {(e) $D^{\blacktriangle}_{\vartriangle}(\mathbf{K,K'})$ }  
    \end{small}
\end{overpic}
\quad
\begin{overpic}
    [height=45mm]{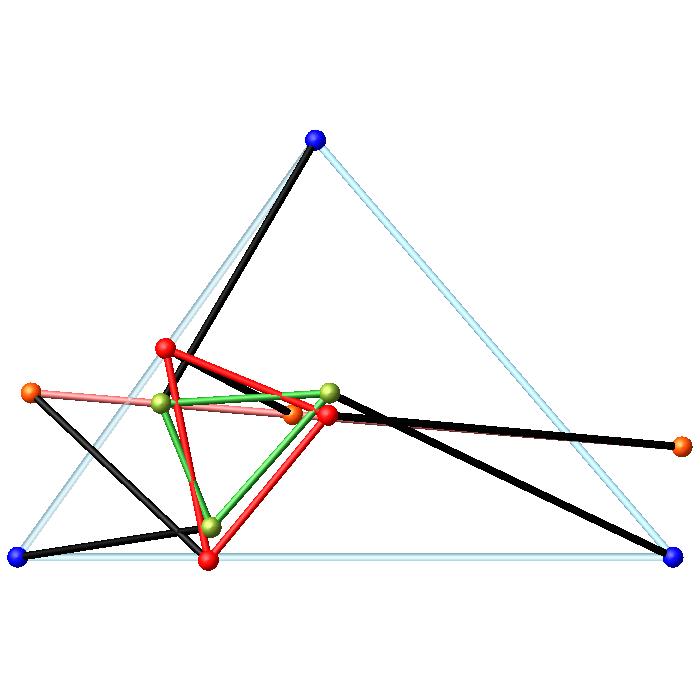}
    \begin{small}
      \put(35,10) {(f) $D^{\vartriangle}_{\vartriangle}(\mathbf{K,K'})$ }  
    \end{small}
\end{overpic} 
\end{center}
\caption{Closest singular configurations on the collinearity variety 
$C_P=0$ (top) and  $C_B=0$ (bottom), respectively, 
for $\phi= 0.8471710528$ radians.}
  \label{fig:my_label}
\end{figure}

\begin{figure}[t]
\begin{center}
\begin{overpic}
    [width=45mm]{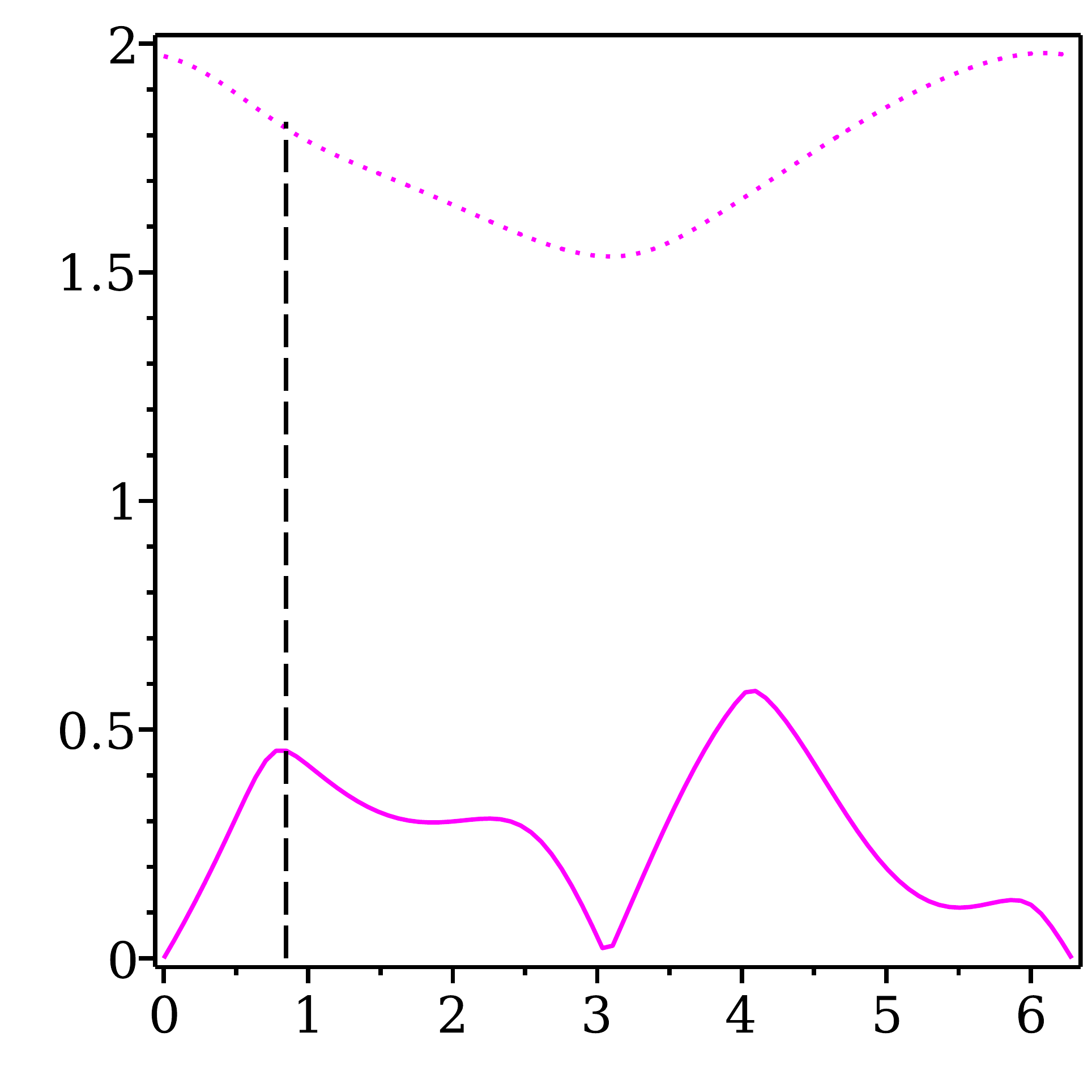}
    \begin{small}
    \put(30,-3) {(a) $D_{\blacktriangle}^{\blacktriangle}(\mathbf{K,K''})$}
     \end{small}
    \begin{scriptsize}   
   \put(0,50){\makebox(0,0){\rotatebox{90}{Distance}}}
   \put(85,1.5){$\phi$ (rad)}
   \end{scriptsize}  
    \end{overpic}
\quad
\begin{overpic}
    [width=45mm]{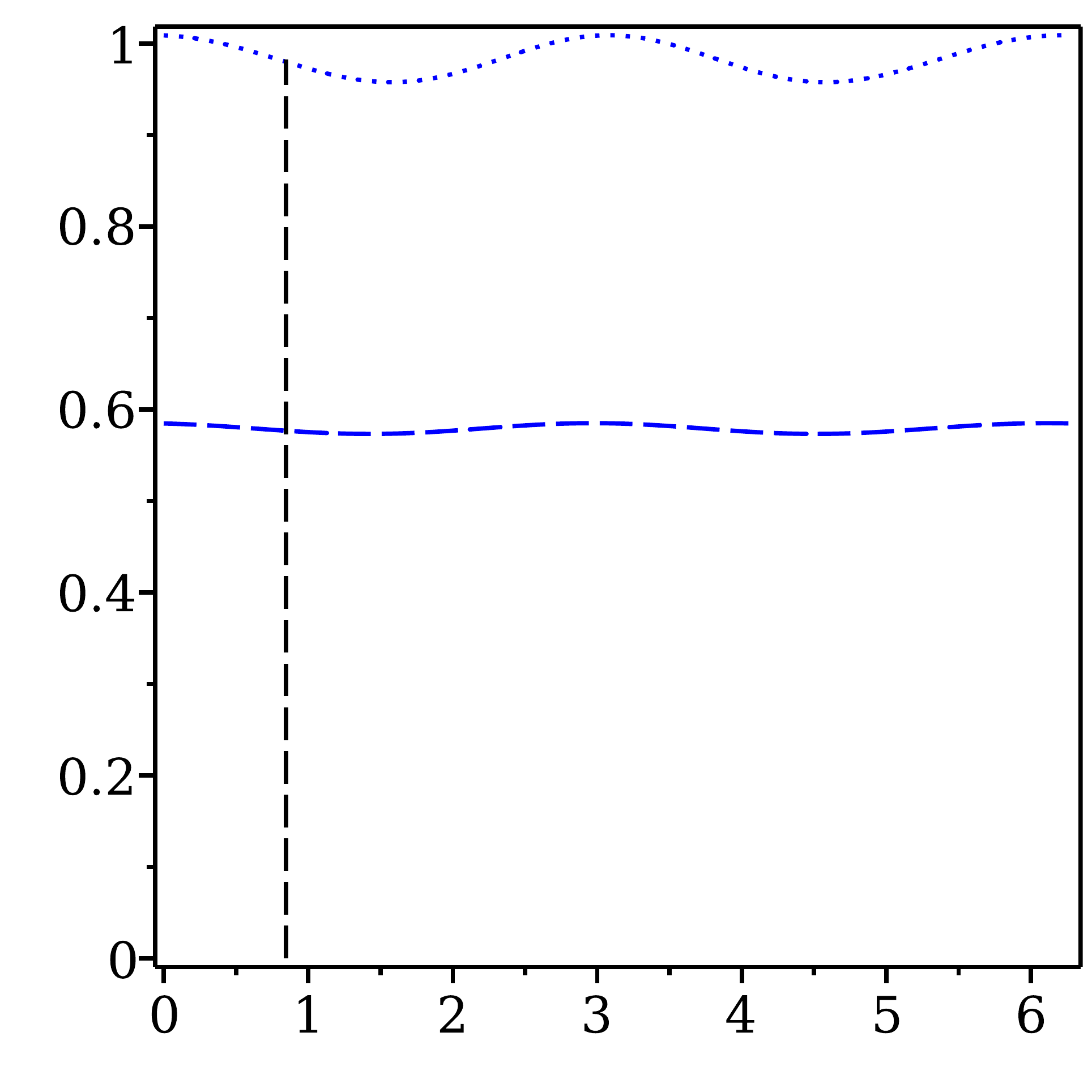}
    \begin{small}
      \put(30,-3) {(b) $D_{\rectangleblack}^{\vartriangle}(\mathbf{K,K''})$ }  
    \end{small}
  \begin{scriptsize}   
   \put(0,50){\makebox(0,0){\rotatebox{90}{Distance}}}
   \put(85,1.5){$\phi$ (rad)}
   \end{scriptsize}  
\end{overpic}
\quad
\begin{overpic}
   [width=45mm]{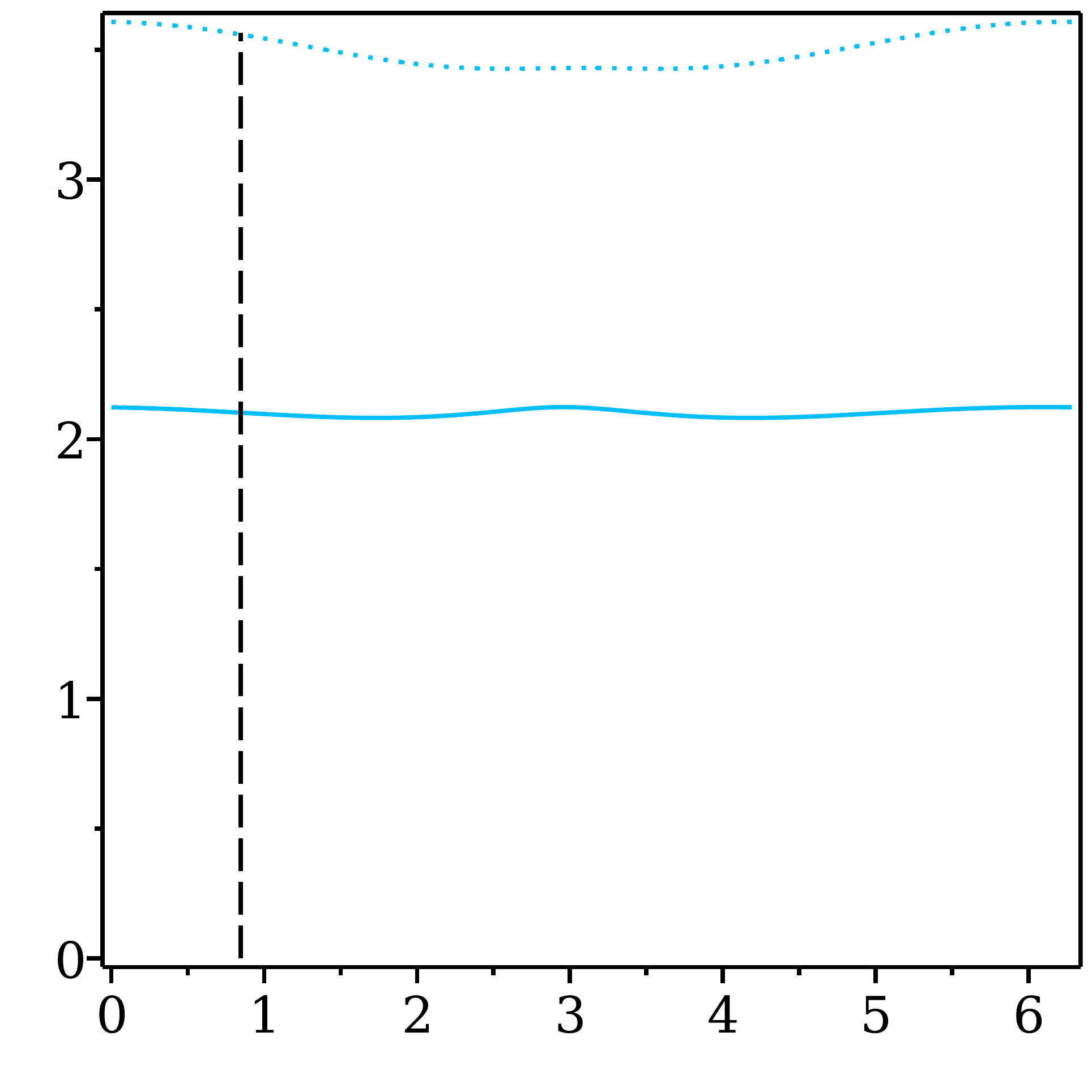}
    \begin{small}
      \put(30,-3) {(c) $D^{\rectangleblack}_{\vartriangle}(\mathbf{K,K''})$ }  
    \end{small}
    \begin{scriptsize}   
   \put(0,50){\makebox(0,0){\rotatebox{90}{Distance}}}
   \put(85,1.5){$\phi$ (rad)}
   \end{scriptsize}
\end{overpic} 
\\ \phantom{x} \\
\begin{overpic}
  [width=45mm]{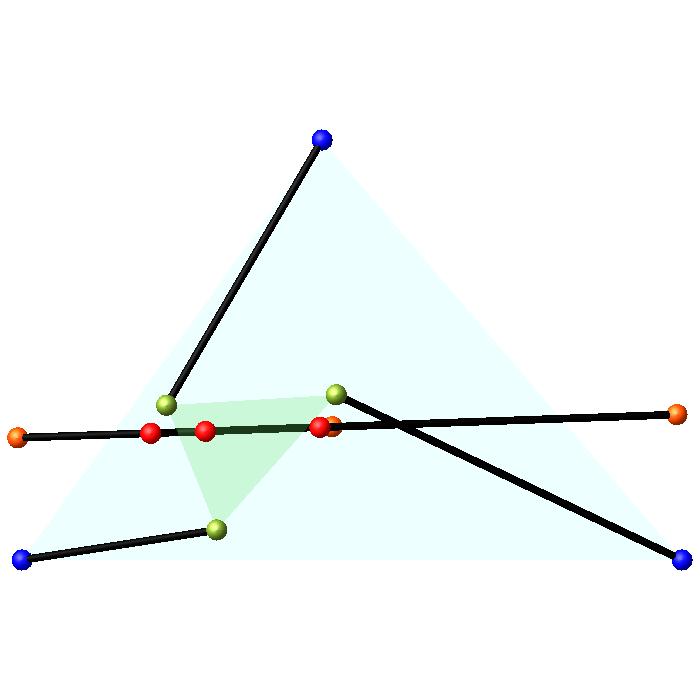}
    \begin{small}
     \put(35,-3) {(d)  $D_{\blacktriangle}^{\blacktriangle}(\mathbf{K,K''})$}
      \end{small}
\end{overpic}
\quad
\begin{overpic}
    [width=45mm]{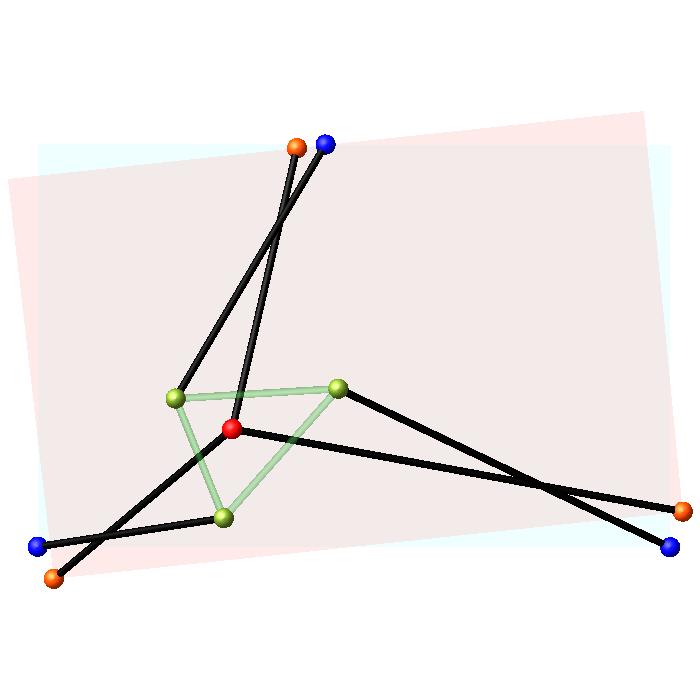}
     \begin{small}
     \put(35,-3) {(e) $D_{\rectangleblack}^{\vartriangle}(\mathbf{K,K''})$ }
     \end{small}
\end{overpic}
\quad
\begin{overpic}
    [width=50mm]{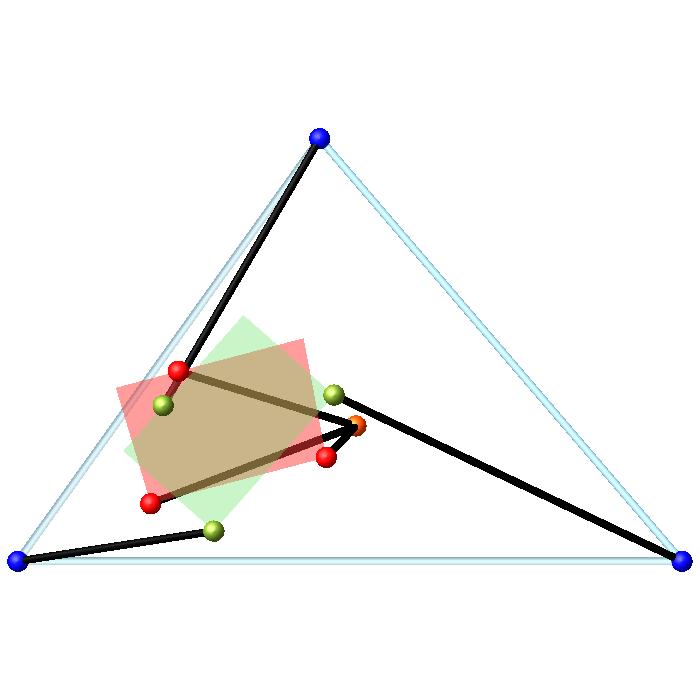}
    \begin{small}
        \put(35,-3){(f) $D_{\vartriangle}^{\rectangleblack}(\mathbf{K,K''})$}  
    \end{small}
\end{overpic} 
\end{center}
 \caption{Comparison of the distances to the closest regular and singular points of the constraint variety, (a) $V=0$,   (b) $C_P=0$,  and (c) $C_B=0$, respectively. The configuration, which corresponds to the closest singular point on the constraint variety (a,b,c), is illustrated in (d,e,f) for $\phi= 0.8471710528$ radians.} 

 \label{finalfig}
  \end{figure}    

In Fig.\ \ref{finalfig} (a,b,c), we present the comparison of the distances to the closest regular and singular points of the constraint variety. It can be observed that the latter is always larger. 
The configuration, which corresponds to the closest singular point on the constraint variety, is illustrated in Fig.\ \ref{finalfig} (d,e,f) for the pose  $\phi= 0.8471710528$.

For the singularity distance, we have to take the minimum of the graphs given in Figs. \ 
\ref{measure_sing}, \ref{fig:result3} and \ref{finalfig}, which yields the graphs given in Fig.\ \ref{measure1}. 
Finally, we want to point out that the cusps
in the graphs of Figs.\  \ref{measure_sing} and \ref{measure1} belong to the cut loci of the extrinsic distance functions; i.e.\ in the corresponding configurations, the closest singular configuration is not determined uniquely (there exist at least two closest singular configurations).

Note that the graphs of all nine presented extrinsic metrics have a similar course (cf.\ Fig.\ \ref{measure_sing} and Fig.\ \ref{measure1}, respectively), which will be compared in the following with the singularity-closeness indices 
reviewed in Section \ref{sec:review}.

 \begin{figure}[h!]
    \begin{center}
    \begin{overpic}[width=14cm]{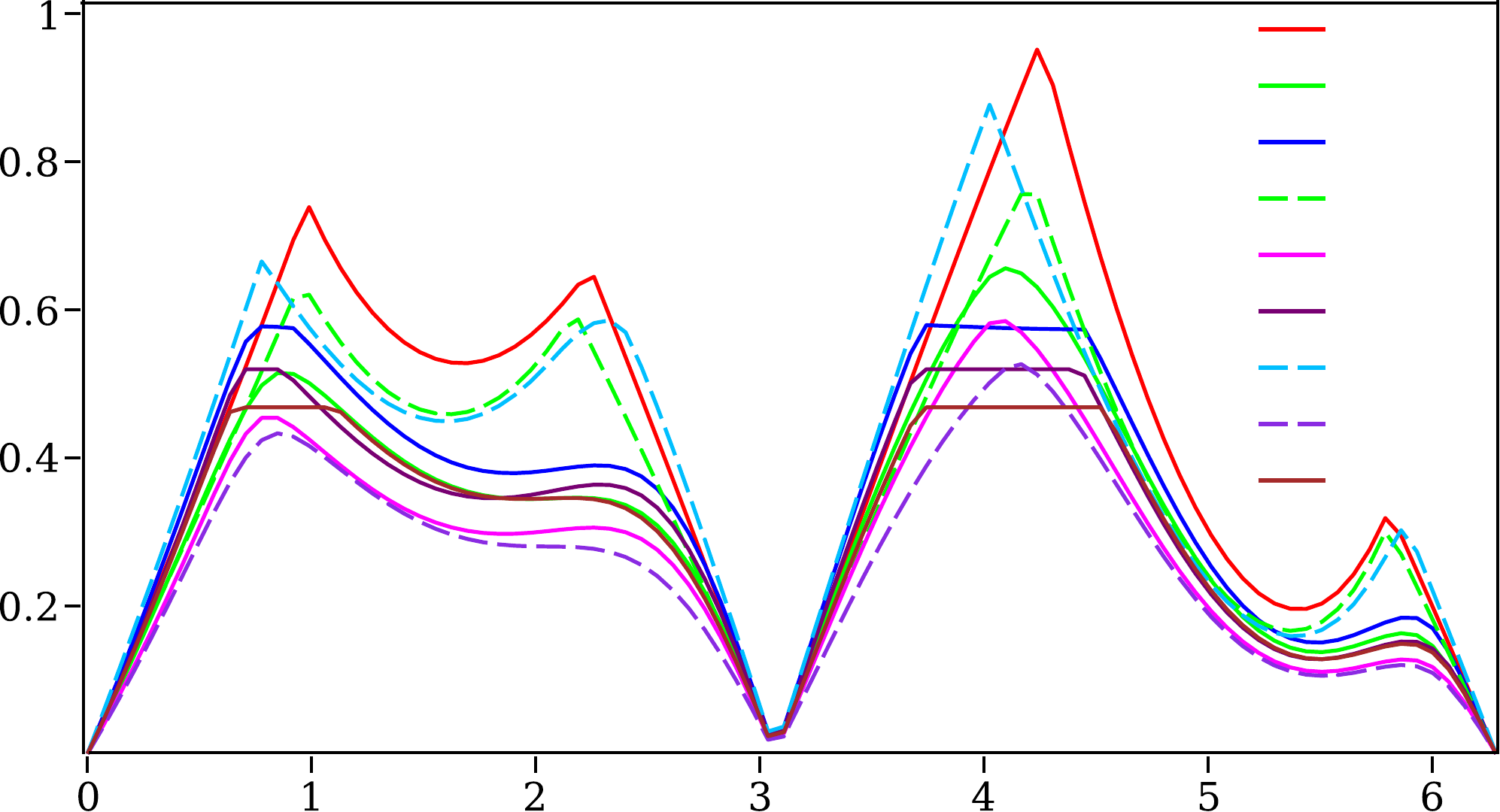}
     
     \begin{footnotesize}
      \put(88.5,51.5) {$D_{\rectangleblack}^{\rectangleblack}(\mathbf{K,K'})$ } 
         \put(88.5,48) {$D_{\rectangleblack}^{\blacktriangle}(\mathbf{K,K'})$ } 
         \put(88.5,44) {$D_{\rectangleblack}^{\vartriangle}(\mathbf{K,K'})$ }
         \put(88.5,40.5) {$D^{\rectangleblack}_{\blacktriangle}(\mathbf{K,K'})$ }
          \put(88.5,36.5) {$D^{\blacktriangle}_{\blacktriangle}(\mathbf{K,K'})$ }
          \put(88.5,33) {$D_{\blacktriangle}^{\vartriangle}(\mathbf{K,K'})$ }
          \put(88.5,29) {$D^{\rectangleblack}_{\vartriangle}(\mathbf{K,K'})$ }
          \put(88.5,25) {$D^{\blacktriangle}_{\vartriangle}(\mathbf{K,K'})$ }
           \put(88.5,21.5) {$D_{\vartriangle}^{\vartriangle}(\mathbf{K,K'})$ }
          \end{footnotesize}
        \put(85,1.3){$\phi$ (rad)}  
       \put(-2.2,30){\makebox(0,0){\rotatebox{90}{Distance}}}
        \end{overpic}
     \caption{Singularity distances with respect to the presented extrinsic metrics.}
     \label{measure1}
    \end{center}
\end{figure}

\subsection{Comparison with EE-independent KPIs}

As a consequence of Remark \ref{rmk:transmission}, both the TI and DS indices encounter the following issue: If we alter the viewpoint and exchange the roles of the platform and the base, we obtain different pressure angles $\beta_i$, resulting in different index values for the same configuration. Therefore, we introduce the {\it modified transmission index} (MTI):
\begin{equation*}
	min(\cos\tfrac{\alpha_1+\beta_1}{2}, \ldots,\cos\tfrac{\alpha_3+\beta_3}{2})\in [0;1],
\end{equation*}
and {\it modified distance to singularity} (MDS):
\begin{equation}\label{eq:MDS}
 1-\frac{2max(\tfrac{\alpha_1+\beta_1}{2},\ldots,\tfrac{\alpha_3+\beta_3}{2})}{\pi}\in [0;1].   
\end{equation}
In addition to the control number, we also provide a modified version (MCN) obtained by taking the square root of $1/\mu_+$, as our goal is not to achieve an index that remains invariant under similarities (which was the rationale behind defining the original control number as a fraction).

Besides the manipulability (M), we also plot the function $\det \mathbf{V}(\mathbf{K})$ as it has the advantage of taking into account the lengths of the legs compared to M (cf.\ Remark \ref{rmk:mani}).
In Fig.\ \ref{comparsionindicies} the graphs of the  EE-independent KPIs are displayed, where M and {$\mathbf{V}(\mathbf{K})$ are scaled between $0$ and $1$, while CN is scaled by a factor of 5} for better visualization.
Note that the IR graph exhibits a jump discontinuity at the pose $\phi =4.385748$ radians where the first and second leg are parallel (cf.\ Footnote \ref{jump}). 
The animation showing the incircle can be downloaded from \cite{codes}.

\begin{figure}[h!]
 \centering
\begin{overpic}[width=14cm]{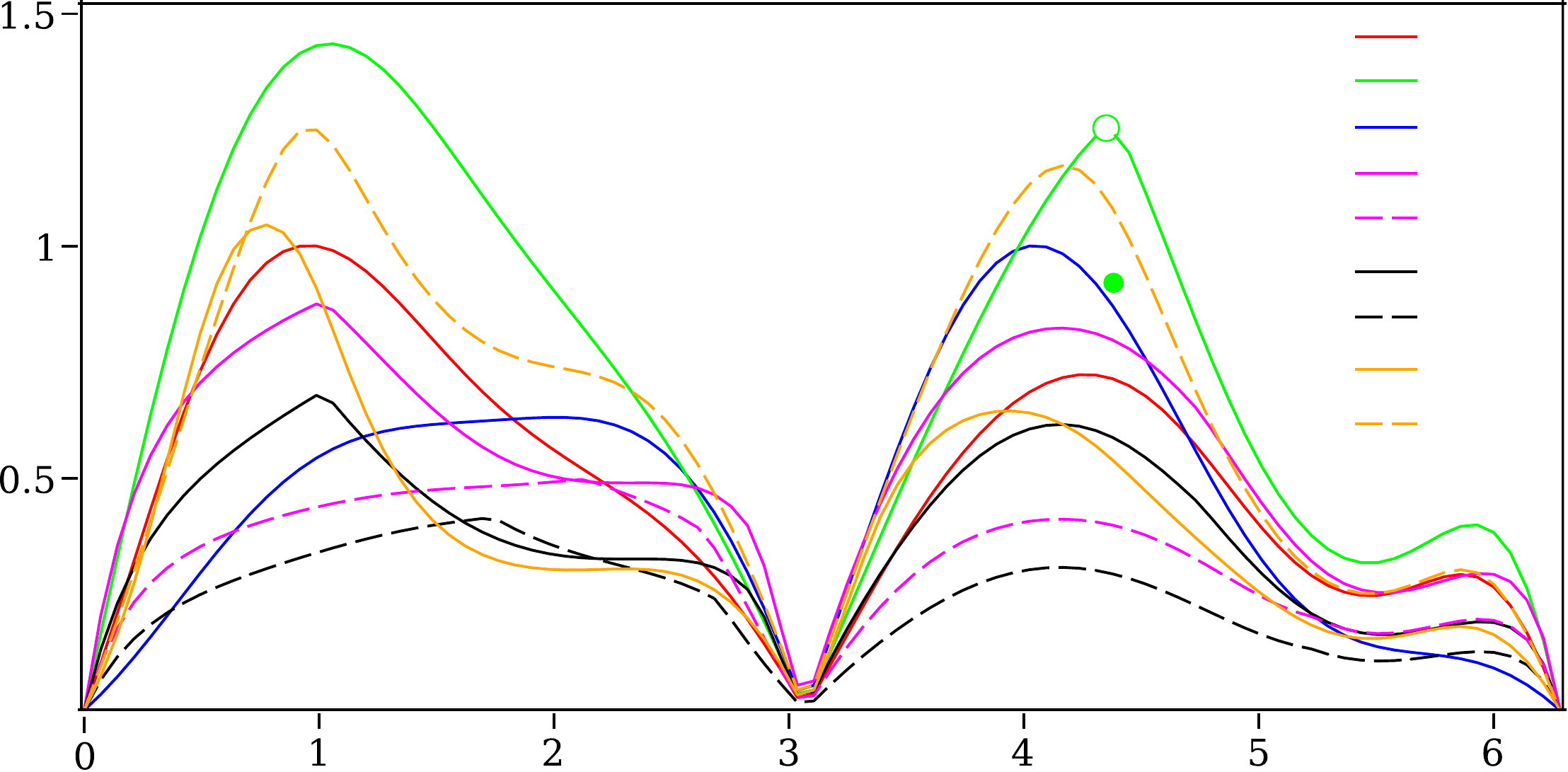}
\begin{small}
  \put(92,46.5) {M}
       \put(92,43.5) {IR}
       \put(92,40) {$\mathbf{V}(\mathbf{K})$}
       \put(92,37) {TI}
        \put(92,34) {MTI}
        \put(92,31) {DS}
         \put(92,28) {MDS}
          \put(92,25) {CN}
          \put(92,21.5) {MCN}
\end{small}
\put(85,1.3){$\phi$ (rad)}
   \put(-2.2,30){\makebox(0,0){\rotatebox{90}{Performance indicies}}}
\end{overpic}
\caption{Comparison with EE-independent KPIs.}
\label{comparsionindicies}
\end{figure}

\subsection{Comparison with the approach of singularity-free cylinders}

In~\cite{LI20061157} and \cite{Abbasnejad}, the orientation range of the platform is assumed to be bounded by some interval, i.e., $\zeta\in\left[\zeta_{\text{min}};\zeta_{\text{max}}\right]$. However, since we are not considering any motion limits, we have $\zeta_{\text{max}}=-\zeta_{\text{min}}=\pi$. Consequently, in every pose for $d$ of Eq.\ (\ref{compare1}), the value is zero, indicating that the singularity-free cylinder degenerates to a line. In other words, for the given position, there always exists an orientation causing a singularity. Following the idea presented in \cite{mao2013new}, in this case, we measure the distance as $|\zeta_0 - \zeta|$ to the nearest singularity in the orientation workspace (as shown in Fig. \ref{leiexample} (a)).

\begin{figure}[h!]
 \centering
\begin{overpic}[width=60mm]{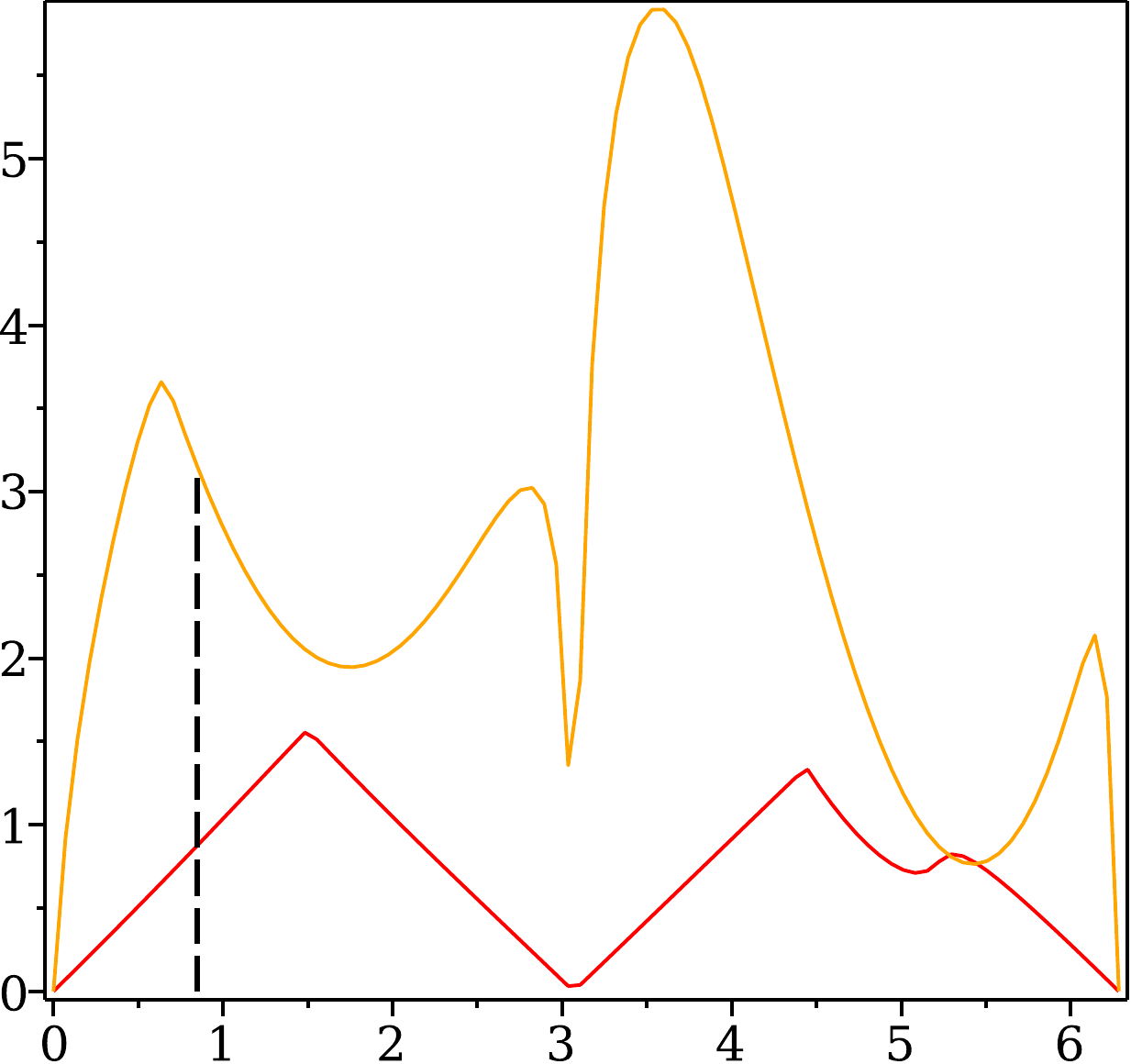}
 \begin{small}
     \put(40,-5) {(a)} 
    \begin{footnotesize}
        \put(-8,40){\makebox(0,0){\rotatebox{90}{Distance}}}
     \put(88,-3.5){$\phi$ (rad)} 
    \end{footnotesize}
    \end{small}
\end{overpic}
\quad \quad
\begin{overpic}[height=55mm]{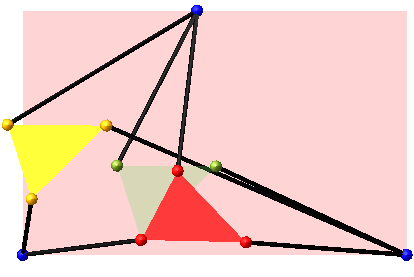}
\begin{small}
     \put(40,-4) {(b)}    
    \end{small}
\end{overpic}
\medskip
\caption{Distance to the closest singularity in the orientation/position workspace represented by the red/yellow graph. 
 (b) The closest singularity for a fixed position and orientation at the configuration $\phi = 0.8471710528$.}
\label{leiexample}
\end{figure}

To complete the picture, we also computed the distance to the closest singularity in the position workspace (e.g., \cite{nag2021singularity}), as illustrated in Fig.\ \ref{leiexample} (a). The closest singularity for a fixed position and orientation at $\phi=0.8471710528$ is shown in Fig.\ \ref{leiexample} (b). For corresponding animations, please refer to the supplementary material \cite{codes}. As previously mentioned in Section \ref{sec:motivation}, it's worth noting that we cannot draw conclusions about the distance to the next singularity within the configuration space of the 3-RPR robot from the two separate pieces of information regarding orientation and position singularity distances.

\begin{rmk}\label{rmk:signed}

Due to the discretization of motion into 90 poses, the singularity occurring at $\phi\approx 3.0356972$ radians is not precisely reached, in contrast to the singularity at $\phi=0$. This discrepancy is the reason why the graphs in {Figs.~\ref{measure1}, \ref{comparsionindicies} and \ref{leiexample} (a)} do not reach zero at this specific point. This deviation is particularly noticeable in the plot showing translational singularity distance, depicted in Fig.\ \ref{leiexample} (a). To address this issue, \footnote{Assumed that the path of the 1-dimensional motion crosses the singularity variety and does not only touch it.}, one can employ a signed distance function. We will illustrate the utilization of such a function in the forthcoming example discussed in the next section. \hfill $\diamond$
\end{rmk}

\subsection{Applying method to 3-RRR mechanisms} \label{sec:SG}

The method for computing the extrinsic distance to the closest singularity can be adopted straight forward to any type of robot (also redundant ones), which can be abstracted into a jointed composition of bars, triangular plates, and tetrahedral bodies.

For the latter, a comparable formula to Equation (\ref{metricdp1}) can be derived by drawing parallels to the explanation provided in Section \ref{minterpretation}. 

In the following, we want to demonstrate
the feasibility of our approach to more complicated mechanisms, like the 3-RRR robot. 
We consider the design parameters and the circular translation of the platform, as outlined in \cite[Section 5]{nasa2011trajectory}, as our input. The motion along the circle is discretized into $90$ equally spaced poses.

{\bf Distance to parallel singularities:} It is well known {\cite[Section 2]{Gomez}} 
that the {\it parallel singularities} of this mechanism can be determined similarly to those of 3-RPR robots. For computing the distance to these singularities we choose the interpretation that the 
platform and base are triangular plates ($\blacktriangle$), which are made of deformable material like the six bars needed for assembling the three legs.

By labeling the given nine points by $\mathbf{k}_1,\ldots ,\mathbf{k}_9$ and the corresponding ones of the closest singularity by $\mathbf{k}_1',\ldots ,\mathbf{k}_9'$ according to 
Fig.\ \ref{3rrrexample} (b), we can make use of the 
 singularity polynomial $V$ of Eq.(\ref{variety}). By employing Eqs.~(\ref{metricdp}) and (\ref{metricdp1})  the distance between two 3-RRR configurations can be written as follows:

\begin{align}\label{distance3rrr}
D_{\blacktriangle}^{\blacktriangle}(\mathbf{K},\mathbf{K}')^2 &=\frac{1}{8} \left[{ \sum_{(i, j)\in I_{5}}  d\left(\vert_{ij},\vert'_{ij}\right)^2}  +  d(\blacktriangle_{789}, \blacktriangle_{789}')^2 + d(\blacktriangle_{456}, \blacktriangle_{456}')^2 \right]    
\end{align}
with ${I_5}=\{(1, 4), (1, 7), (2, 5),(2, 8), (3, 6), (3, 9)\}$.  
According to Remark \ref{rmk:signed}, we want to use a signed distance function for this example. The sign of $D_{\blacktriangle}^{\blacktriangle}(\mathbf{K},\mathbf{K}')$ is obtained by the signum function of $V(\mathbf{K})$. By using Eq.~(\ref{distance3rrr}) the Lagrangian function for finding the closest singular configuration reads as:

\begin{equation}\label{3rrr}
L = D^{\blacktriangle}_\blacktriangle(\mathbf{K},\mathbf{K}')^2 +\lambda V.
\end{equation}


This is a non-homogeneous polynomial equation involving the unknowns $c_1,\ldots, c_9,d_1,\ldots,d_9$ and $\lambda$. The partial derivatives of $L$ with respect to the mentioned unknowns result in $19$ equations. This square system is solved for the source configuration $\mathbf{K}^{\mathbb{C}}$ (cf.\ Section \ref{step1}) using the regeneration algorithm implemented in Bertini. 
We obtain $50$ generic finite solutions over $\mathbb{C}$ (computed without any path failures), which are used to run the remaining computational pipeline described in Section \ref{results}. 
The obtained results are illustrated in Fig.\ \ref{3rrrexample}. 
For comparison, we plot the determinant of the Jacobian (used in \cite{nasa2011trajectory}) in  Fig.\ \ref{3rrrexample} (a) scaled by the factor 0.04
(without taking the absolute value as done in Fig. 7).

{\bf Distance to leg singularities:}
Our method can also be used to compute the distance to {\it leg singularities}. We just have to replace the side-condition $V=0$ in the Langrange function of Eq.~(\ref{3rrr}) by the collinearity condition $C_i=0$ of the three points of the $i$-th leg ($i=1,2,3$). Moreover, the sign of the distance is obtained by the signum function of $C_i(\mathbf{K})$. 
In this way, we get the three additional graphs plotted in Fig.\ \ref{3rrrexample} (a).

\begin{figure}[h!]
  \centering
\begin{overpic}[width=70mm]{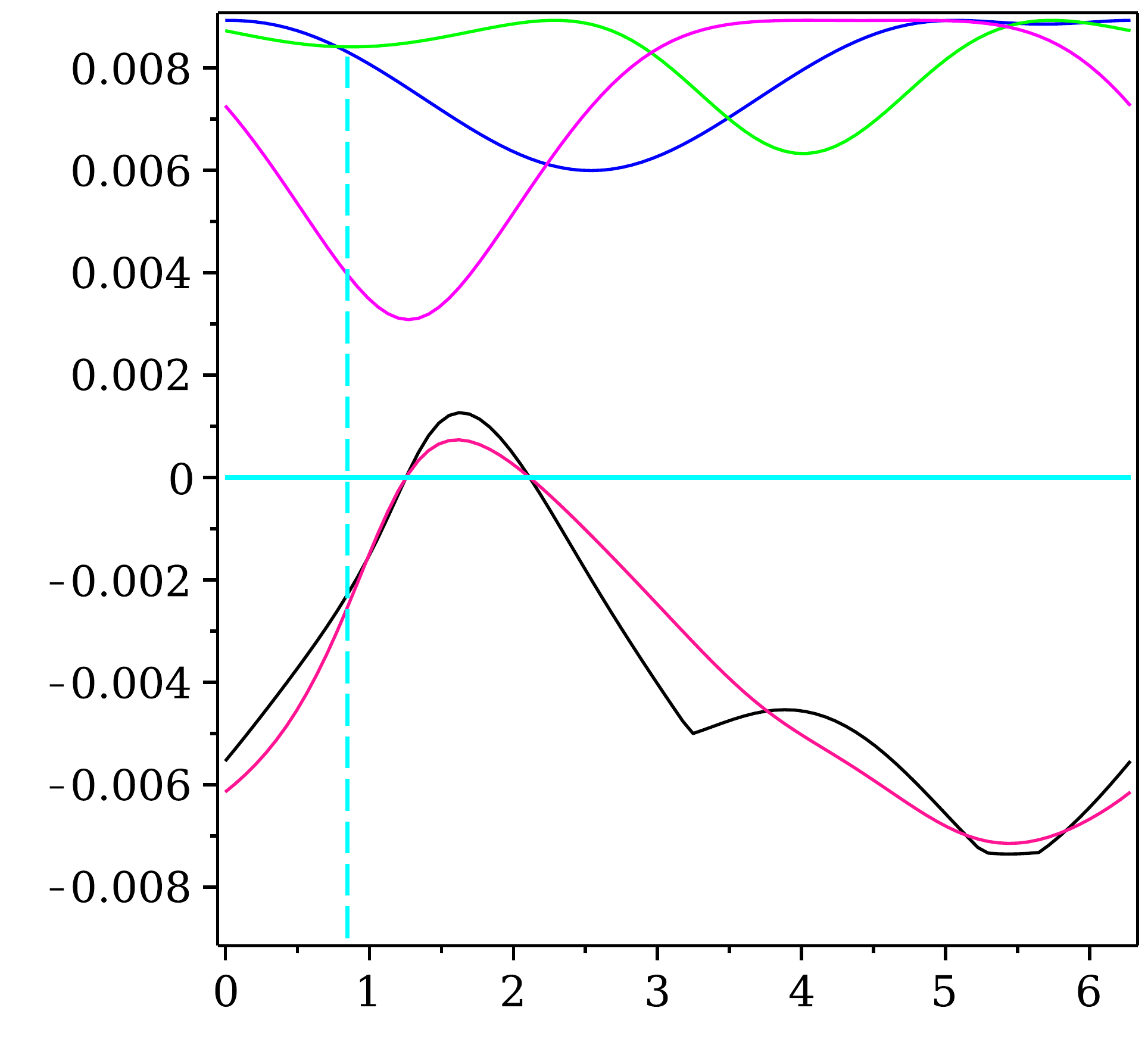}
\begin{footnotesize}
        \put(-6,40){\makebox(0,0){\rotatebox{90}{Distance}}}
     \put(88,1){$\phi$ (rad)} 
    \end{footnotesize}
     \begin{small}
     \put(40,-5) {(a)}
     \end{small}
 \end{overpic}
\quad \quad
\begin{overpic}[width=70mm]{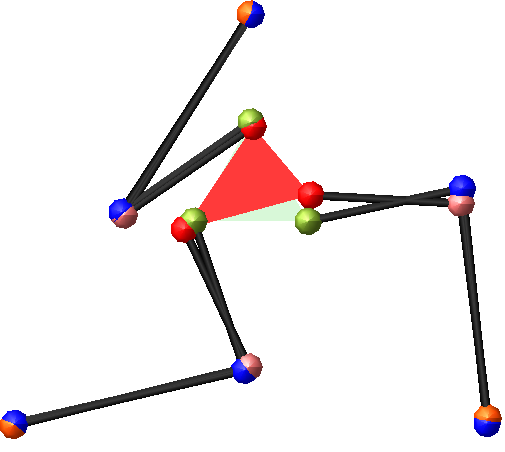}
 \begin{small}
     \put(40,-7) {(b)} 
      \put (-.1,10){$\mathbf{k}_{7}$}
    \put (-0,-2){$\mathbf{k}_{7}'$}
      \put (50,10){$\mathbf{k}_{1}$}
      \put (53,15){$\mathbf{k}_{1}'$}
       \put (42.5,41){$\mathbf{k}_{4}$}
       \put (29.5,38){$\mathbf{k}_{4}'$}
        \put (87,5){$\mathbf{k}_{8}$}
        \put (98.5,8){$\mathbf{k}_{8}'$}
         \put (90,57){$\mathbf{k}_{2}$}
          \put (95,48){$\mathbf{k}_{2}'$}
          \put (60,39){$\mathbf{k}_{5}$}
          \put (60,55.5){$\mathbf{k}_{5}'$}
          \put (46,71.5){$\mathbf{k}_{6}$}
           \put (53.5,63.5){$\mathbf{k}_{6}'$}
          \put (53,82){$\mathbf{k}_{9}$}
           \put (40,87){$\mathbf{k}_{9}'$}
          \put (19,52.5){$\mathbf{k}_{3}$}
           \put (21,39){$\mathbf{k}_{3}'$}
            \put (30,5){$\vert_{71}$}
            \put (35,26){$\vert_{14}$}
            \put (96,25){$\vert_{82}$}
             \put (70,40){$\vert_{25}$}
              \put (25,65){$\vert_{93}$}
              \put (38,63){$\vert_{36}$}
    \end{small}
     
\end{overpic}
\bigskip
\caption{(a) Singularity distances with respect to the presented distance metric, where the signed distance to the closest parallel singularity is given in black, and the distance to the closest singularity of the {$1^{st}$/$2^{nd}$/$3^{rd}$} leg is represented in blue/green/magenta. The course of the manipulability index is shown in red. (b) The closest parallel singularity (red) to the configuration is indicated by the cyan dashed line in (a).
} 
\label{3rrrexample}
\end{figure}

\begin{figure}[h!]
    \centering
    \begin{overpic}[width=52mm]{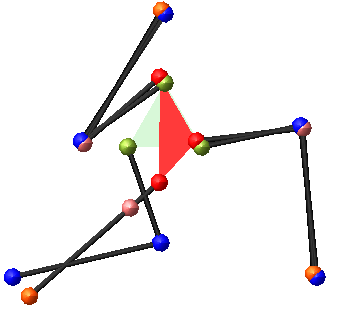}
    \begin{small}
     \put(35,5) {(a)}    
    \end{small}
\end{overpic}
\begin{overpic}[width=52mm]{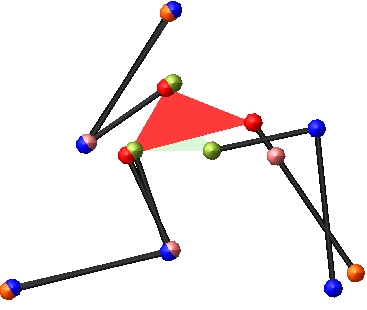}
\begin{small}
     \put(35,5) {(b)}    
    \end{small}
 \end{overpic}
\quad 
\begin{overpic}[width=52mm]{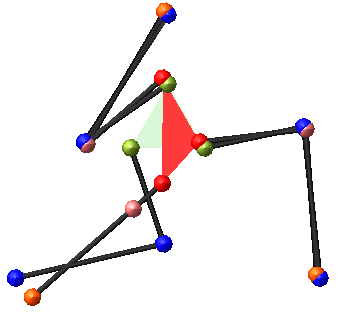}
\begin{small}
     \put(35,5) {(c)}    
    \end{small}
\end{overpic}
    \caption{The closest singularity for the {$1^{st}$/$2^{nd}$/$3^{rd}$} leg (a/b/c) to the configuration indicated by the cyan dashed line in Fig. \ref{3rrrexample} (a).}
\end{figure}

\begin{rmk}
The animations of the manipulator's one-parametric motion given in Eq.\ (\ref{parameter}) together with the closest singular configurations implying the graphs of  Figs.\ \ref{measure_sing}, \ref{fig:result3}, \ref{finalfig} \ref{leiexample} and \ref{3rrrexample} can be downloaded from~\cite{codes}.  \hfill $\diamond$
\end{rmk}

The singularity distance along the 1-parametric motion is then obtained by the distance from the nearest point on the four graphs\footnote{It can easily be checked that the closest singular points on the constrained varieties of the parallel singularity and leg singularity, respectively, always result in distances larger than those implied by the closest regular points.} 
  to the axis of the motion parameter. This shows up a further advantage of using a well-defined metric for measuring singularity distances as it allows to treat leg and parallel singularities in a uniform way.

\begin{rmk}
The application of our method to spatial mechanisms like the Stewart Gough platform needs some additional studies -- dedicated to future research -- concerning the various motion representations, which heavily affect the efficiency
of the homotopy continuation-based singularity distance computation  (cf.\ \cite{kapilavai2020homotopy}).
\hfill $\diamond$
\end{rmk}

\section{Conclusion and future work}\label{end}
We presented extrinsic metrics for the computation of the closest singular configuration of a 3-RPR parallel manipulator, 
which takes the combinatorial structure of the mechanism into account as well as different design options. For each of the resulting nine interpretations (Section \ref{sec:moti}) the corresponding extrinsic metric relies on the distance computation between structural components (Section \ref{sec:extrinsic}). 

The constrained optimization problem of computing the closest singular configuration with respect to these nine metrics was formulated using the Lagrangian approach (Section \ref{generalcase}). As the singular points of the constraint varieties are missed by this method, we determined them algebraically. Based on their obtained geometric characterization, we parameterized them for a separate optimization (Section \ref{sec:singpoints}).

Moreover, we presented an efficient computational algorithm for the singularity distance computation along a 1-parametric motion of the manipulator using the computer algebra system \texttt{Maple} as well as the numerical algebraic geometry software \texttt{Bertini} and \texttt{Paramotopy} (Section \ref{results}). We coded 
an open-source software interface allowing to make all calls from  \texttt{Maple} rather than by switching between the systems. Finally, we validated our approach on hand of a numerical example (Section \ref{numericalexample}).

The discussion of this concrete example
 showed that the graphs of all nine presented extrinsic metrics have a similar course (cf.\ Fig.\ \ref{measure_sing} and Fig.\ \ref{measure1}, respectively), which was compared with existing indices that claim to evaluate the closeness to singularities (cf.\ Figs.\ \ref{comparsionindicies} and \ref{leiexample} (a)). 
Note that all nine indices have the potential  (cf.\ Table \ref{runtime}) to be calculated in real time as the suggested method of homotopy continuation allows the parallelization of computations.

In \cite{intrinsic}, we conducted a study on the computation of singularity distances for the same nine interpretations of 3-RPR configurations using intrinsic metrics, and we compared them to the corresponding extrinsic metrics presented in this paper. The observed relationship between the extrinsic and intrinsic distances is not yet fully understood. This open problem, as well as the sensitivity analysis based on the paired intrinsic and extrinsic metrics mentioned in Remark \ref{rmk:sensitivity}, is dedicated to future research.

\bigskip

\noindent
{\bf Acknowledgement}\newline
This research is supported by Grant No.\ P 30855-N32 from the Austrian Science Fund (FWF). The first author would like to thank Silviana Amethyst for valuable suggestions and technical discussions on \texttt{Bertini}/\texttt{Paramotopy}. 

\bibliography{bibfile1}

\begin{thebibliography}{10}
\expandafter\ifx\csname url\endcsname\relax
  \def\url#1{\texttt{#1}}\fi
\expandafter\ifx\csname urlprefix\endcsname\relax\def\urlprefix{URL }\fi
\expandafter\ifx\csname href\endcsname\relax
  \def\href#1#2{#2} \def\path#1{#1}\fi

\bibitem{merlet2006parallel}
J.-P. Merlet, Parallel Robots, Vol. 128, Springer Science \& Business Media, 2006.

\bibitem{hubert}
J.~Hubert, J.-P. Merlet, Static of parallel manipulators and closeness to singularity, Journal of Mechanisms and Robotics 1~(1) (2008) 1--6.

\bibitem{intrinsic}
A.~Kapilavai, G.~Nawratil, Singularity distance computations for 3-{RPR} manipulators using intrinsic metrics (2023).
\newblock \href {http://arxiv.org/abs/2307.14721} {\path{arXiv:2307.14721}}.

\bibitem{angeles}
J.~Angeles, Theory, Methods, and Algorithms, 2nd Edition, Springer-Verlag New York, 2002.

\bibitem{bu2016closeness}
W.~Bu, Closeness to singularities of robotic manipulators measured by characteristic angles, Robotica 34~(9) (2016) 2105--2115.

\bibitem{bu2017closeness}
W.~Bu, Closeness to singularities of manipulators based on geometric average normalized volume spanned by weighted screws, Robotica 35~(7) (2017) 1616--1626.

\bibitem{hartley2001invariant}
{D.M. Hartley, D.R. Kerr}, Invariant measures of the closeness to linear dependence of six lines or screws, Proceedings of the Institution of Mechanical Engineers, Part C: Journal of Mechanical Engineering Science 215~(10) (2001) 1145--1151.

\bibitem{huang2014force}
{T. Huang, M. Wang, S. Yang, T. Sun, D.G. Chetwynd, F. Xie}, Force/motion transmissibility analysis of six degree of freedom parallel mechanisms, Journal of Mechanisms and Robotics 6~(3) (2014) 031010.

\bibitem{laryushkin2015estimation}
{P.A. Laryushkin, V.A. Glazunov}, On the estimation of closeness to singularity for parallel mechanisms using generalized velocities and reactions, in: Proceedings of the $14^{th}$ IFToMM World Congress, 2015, pp. 286--291.

\bibitem{liu2012new}
{X.J. Liu, C. Wu, J. Wang}, {A New Approach for Singularity Analysis and Closeness Measurement to Singularities of Parallel Manipulators}, Journal of Mechanisms and Robotics 4~(4) (2012) 041001.

\bibitem{mao2013new}
{J. Mao, Y. Guo, J. Ren, W. Guo}, A new {E}uclidian distance based approach to measure closeness to singularity for parallel manipulators, in: Intelligent Robotics and Applications: $6^{th}$ International Conference, ICIRA 2013, Busan, South Korea, September 25-28, 2013, Proceedings, Part II 6, Springer, 2013, pp. 41--53.

\bibitem{nawratil2009new}
G.~Nawratil, New performance indices for 6-dof {UPS} and 3-dof {RPR} parallel manipulators, Mechanism and Machine Theory 44~(1) (2009) 208--221.

\bibitem{takeda1996kinematic}
{Y. Takeda, H. Funabashi}, Kinematic and static characteristics of in-parallel actuated manipulators at singular points and in their neighborhood, Transactions of the Japan Society of Mechanical Engineers, Part C 39~(1) (1996) 85--93.

\bibitem{Wolf}
A.~Wolf, M.~Shoham, {Investigation of Parallel Manipulators Using Linear Complex Approximation }, Journal of Mechanical Design 125~(3) (2003) 564--572.

\bibitem{Chao}
{C. Wu, X.J.Liu, F. Xie, J. Wang}, New measure for "closeness" to singularities of parallel robots, in: 2011 IEEE International Conference on Robotics and Automation, 2011, pp. 5135--5140.

\bibitem{ebrahimi2007actuation}
{I. Ebrahimi, J.A. Carretero, R. Boudreau}, Actuation scheme for a 6-dof kinematically redundant planar parallel manipulator, in: Proceedings of the $12^{th}$ IFTOMM World Congress, 2007.

\bibitem{Gomez}
{O.A. Gómez, P. Wenger, A. Pámanes}, Performance indices for kinematically redundant parallel planar manipulators, Problems of Mechanics 22~(1) (2006) 22--40.

\bibitem{pottmann1998approximation}
{H. Pottmann, M. Peternell, B. Ravani}, Approximation in line space—applications in robot kinematics and surface reconstruction, Advances in Robot Kinematics: Analysis and Control (1998) 403--412.

\bibitem{yoshikawa1985manipulability}
{T. Yoshikawa}, Manipulability of robotic mechanisms, The International Journal of Robotics Research 4~(2) (1985) 3--9.

\bibitem{lee1996optimum}
{J. Lee, J. Duffy, M. Keler}, The optimum quality index for the stability of in-parallel planar platform devices, in: International Design Engineering Technical Conferences and Computers and Information in Engineering Conference, Vol. 97584, 1996, p. V02BT02A068.

\bibitem{takeda1995motion}
{Y. Takeda, H. Funabashi}, Motion transmissibility of in-parallel actuated manipulators, Transactions of the Japan Society of Mechanical Engineers, Ser. C, Dynamics, control, robotics, design and manufacturing 38~(4) (1995) 749--755.

\bibitem{Voglewede}
P.~A. Voglewede, I.~Ebert-Uphoff, Measuring "closeness" to singularities for parallel manipulators, in: IEEE International Conference on Robotics and Automation, 2004. Proceedings. ICRA '04. 2004, Vol.~5, 2004, pp. 4539--4544.

\bibitem{Voglewede1}
P.~A. Voglewede, I.~Ebert-Uphoff, Overarching framework for measuring closeness to singularities of parallel manipulators, IEEE Transactions on Robotics 21~(6) (2005) 1037--1045.

\bibitem{LI20061157}
H.~Li, C.~M. Gosselin, M.~J. Richard, Determination of maximal singularity-free zones in the workspace of planar three-degree-of-freedom parallel mechanisms, Mechanism and {M}achine Theory 41~(10) (2006) 1157--1167.

\bibitem{Abbasnejad}
{G. Abbasnejad, H. M. Daniali, S. M. Kazemi}, A new approach to determine the maximal singularity-free zone of 3-{RPR} planar parallel manipulator, Robotica 30~(6) (2012) 1005–1012.

\bibitem{RASOULZADEH2020104002}
A.~Rasoulzadeh, G.~Nawratil, Variational path optimization of linear pentapods with a simple singularity variety, Mechanism and Machine Theory 153 (2020) 104002.

\bibitem{li2007determination}
{H. Li, C.M. Gosselin, M.J. Richard }, Determination of the maximal singularity-free zones in the six-dimensional workspace of the general {G}ough--{S}tewart platform, Mechanism and {M}achine Theory 42~(4) (2007) 497--511.

\bibitem{nag2021singularity}
{A. Nag, S. Bandyopadhyay}, Singularity-free spheres in the position and orientation workspaces of {S}tewart platform manipulators, Mechanism and Machine Theory 155 (2021) 104041.

\bibitem{kaloorazi2016determining}
{M. H. F. Kaloorazi, M. T. Masouleh, S. Caro}, Determining the maximal singularity-free circle or sphere of parallel mechanisms using interval analysis, Robotica 34~(1) (2016) 135--149.

\bibitem{jiang2006maximal}
{Q. Jiang, C.M. Gosselin}, The maximal singularity-free workspace of planar 3-{RPR} parallel mechanisms, in: 2006 International Conference on Mechatronics and Automation, 2006, pp. 142--146.

\bibitem{Nawratil_2019}
G.~Nawratil, Singularity distance for parallel manipulators of {S}tewart {G}ough type, in: T.~Uhl (Ed.), Advances in Mechanism and Machine Science, Springer, 2019, pp. 259--268.

\bibitem{jaquier2021geometry}
{N. Jaquier, L. Rozo, D.G. Caldwell, S. Calinon}, Geometry-aware manipulability learning, tracking, and transfer, The International Journal of Robotics Research 40~(2-3) (2021) 624--650.

\bibitem{MARIC2021103865}
{F. Marić, L. Petrović, M. Guberina, J. Kelly, I. Petrović}, A {R}iemannian metric for geometry-aware singularity avoidance by articulated robots, Robotics and Autonomous Systems 145 (2021) 103865.

\bibitem{BHSW06}
D.~J. Bates, J.~D. Hauenstein, A.~J. Sommese, C.~W. Wampler, Bertini: Software for numerical algebraic geometry, Available at bertini.nd.edu with permanent doi: dx.doi.org/10.7274/R0H41PB5 (2013).

\bibitem{bates2013numerically}
D.~J. Bates, A.~J. Sommese, J.~D. Hauenstein, C.~W. Wampler, Numerically solving polynomial systems with {B}ertini, SIAM, 2013.

\bibitem{bates2018paramotopy}
D.~Bates, D.~Brake, M.~Niemerg, Paramotopy: Parameter homotopies in parallel, in: International Congress on Mathematical Software, Springer, 2018, pp. 28--35.

\bibitem{caro2009sensitivity}
S.~Caro, N.~Binaud, P.~Wenger, Sensitivity analysis of 3-{RPR} planar parallel manipulators, Journal of Mechanical Design 131~(12) (2009).

\bibitem{goldsztejn2016three}
{A. Goldsztejn, S. Caro, G. Chabert}, A three-step methodology for dimensional tolerance synthesis of parallel manipulators, Mechanism and {M}achine Theory 105 (2016) 213--234.

\bibitem{chen1999approximation}
H.-Y. Chen, H.~Pottmann, Approximation by ruled surfaces, Journal of Computational and Applied Mathematics 102~(1) (1999) 143--156.

\bibitem{survey}
G.~Nawratil, Point-models for the set of oriented line-elements – a survey, Mechanism and Machine Theory 111 (2017) 118--134.

\bibitem{kapilavai2020homotopy}
A.~Kapilavai, G.~Nawratil, On homotopy continuation based singularity distance computations for 3-{RPR} manipulators, in: European Conference on Mechanism Science, Springer, 2020, pp. 56--64.

\bibitem{codes}
A.~Kapilavai, Supplementary materials: Singularity distance computations for 3-{RPR} manipulators using extrinsic metrics, Mendeley Data, (2023), V3.
\newblock \href {https://doi.org/10.17632/vwgxrjkhhm.3} {\path{doi:10.17632/vwgxrjkhhm.3}}.

\bibitem{manfred}
M.~Husty, On singularities of planar 3-{RPR} parallel manipulators, in: Proceedings of the $14^{th}$ IFToMM World Congress, Vol.~4, {IFToMM}, 2015, pp. 2325--2330.

\bibitem{sturmfels2005grobner}
B.~Sturmfels, What is... a grobner basis?, Notices-American Mathematical Society 52~(10) (2005) 1199.

\bibitem{sommese2005numerical}
A.~J. Sommese, C.~W. Wampler, The Numerical solution of systems of polynomials arising in engineering and science, World Scientific, 2005.

\bibitem{hauenstein2011regeneration}
J.~Hauenstein, A.~Sommese, C.~Wampler, Regeneration homotopies for solving systems of polynomials, Mathematics of Computation 80~(273) (2011) 345--377.

\bibitem{nasa2011trajectory}
{C. Nasa, S. Bandyopadhyay}, Trajectory-tracking control of a planar 3-{RRR} parallel manipulator with singularity avoidance, in: $13^{th}$ World Congress in Mechanism and {M}achine Science, 2011, pp. 19--25.

\end{thebibliography}

\appendix
\section{Derivation of the distance between triangular plates}\label{minterpretation}
The triangular plate $\blacktriangle_{ijk}$ in $\mathbb{R}^{2}$ can be captured by the affine combinations of its vertices $\mathbf{k}_{i}, \mathbf{k}_{j},\mathbf{k}_{k}$;  i.e.\

\begin{equation}\label{affhull}
 \blacktriangle_{ijk}: u\ \mathbf{k}_{i} + v\ \mathbf{k}_{j} + (1-u-v)\ \mathbf{k}_{k}  \quad \text{with}\quad \textnormal{u}\in [0,1],\ \textnormal{v}\in [0,1-u].
\end{equation}

 In order to define a distance function between two triangles $\blacktriangle_{ijk} = (\mathbf{k}_{i},\mathbf{k}_{j}, \mathbf{k}_{k})$ and $\blacktriangle_{ijk}^{'}=(\mathbf{k}_{i}^{'},\mathbf{k}_{j}{'}, \mathbf{k}_{k}{'} )$ illustrated in Fig.~\ref{fig:test1} we use the  parameterization of Eq.\ (\ref{affhull}) to view both triangles as  images of the unit triangle $T_0$ under the following two affine mappings: 
\begin{align}
    &\psi_{\blacktriangle_{ijk}} :\  \begin{pmatrix}
     u \\ v
     \end{pmatrix} \mapsto u \mathbf{k}_i + v \mathbf{k}_j + (1 - u - v)\mathbf{k}_k  \quad \text{with}\quad  u \in [0,1], v \in [0,1-u], \\
     &\psi_{\blacktriangle_{ijk}'} :\  \begin{pmatrix}
     u \\ v
    \end{pmatrix} \mapsto u \mathbf{k}_i' + v \mathbf{k}_j' + (1 - u - v)\mathbf{k}_k'  \quad \text{with}\quad  u \in [0,1], v \in [0,1-u].
\end{align}
{Now we define the distance as the integral over the pointwise distances of corresponding triangle points by:}
\begin{equation}\label{triangle}
 \begin{split}
 &\frac{1}{|T_0|} \int_{T_0} \left( \psi_{\blacktriangle_{ijk}} - \psi_{\blacktriangle_{ijk}^{'}}\right)^2\, dA  
  = \\
  &2 \int_{0}^{1} \int_{0}^{1-u}\left[u\ \mathbf{(k}_{i}-\mathbf{k}_{i}^{'})+v \ \mathbf{(k}_{j}-\mathbf{k}_{j}^{'})+(1-u-v) (\mathbf{k}_{k}-\mathbf{k}_{k}^{'})\right]^2 \,du \,dv 
  \end{split}
 \end{equation}
{where $A$ denotes the area.}
 By solving Eq.~(\ref{triangle}) we obtain the squared distance between two triangles $\blacktriangle_{ijk}$ 
and $\blacktriangle'_{ijk}$ given in Eq.~(\ref{metricdp1}).

\section{Lemma}\label{app:lem}
\begin{lem}\label{lem:app}
For a 3-RPR manipulator, which is not architecturally singular, there always exists pairwise distinct indices $i,j\in\left\{1,2,3\right\}$ such that $\mathbf{k}_i\neq \mathbf{k}_j$ and 
$\mathbf{k}_{i+3}\neq \mathbf{k}_{j+3}$ hold true. 
\end{lem}
\begin{proof}
From the geometric point of view, a 3-RPR manipulator is only architecturally singular if and only if either the base or platform collapses to a point. 

Suppose that $\mathbf{k}_{1}, \mathbf {k}_{2}, \mathbf{k}_{3}$ are pairwise distinct. As the manipulator is not architectural singular, there has to  exist a pair of indices  $i,j\in\left\{1,2,3\right\}$ such that $\mathbf{k}_{i+1}\neq\mathbf{k}_{j+1}$ holds, and we are done. 

Suppose now that a pair of points among $\mathbf{k}_{1}, \mathbf{k}_{2}, \mathbf{k}_{3}$ coincide; without loss of generality we can assume $\mathbf{k}_{1}=\mathbf{k}_{2}$. As the manipulator is not architecturally singular, $\mathbf{k}_{3}$ is distinct from $\mathbf{k}_{1}=\mathbf{k}_{2}$. Now either $\mathbf{k}_{4}\neq \mathbf{k}_{6}$ has to hold or $\mathbf{k}_{5}\neq \mathbf{k}_{6}$, as otherwise this would imply the architecture singularity  $\mathbf{k}_{4}=\mathbf{k}_{5}=\mathbf{k}_{6}$. 
\end{proof}

\section{Explicit expressions for the distances of Theorem \ref{regression1}}\label{explicit}
The explicit expressions for the distances
$D_{\star}^{\vartriangle}(\mathbf{K},\mathbf{K}')$ with $\star \in \{{\blacktriangle}, {\vartriangle}\}$ read as follows:
\begin{align}
  D_{\blacktriangle}^{\vartriangle}(\mathbf{K},\mathbf{K}') = &min\left[\frac{23}{630} (x_{5}^2-x_{5}x_{6}+x_{6}^2+y_{6}^{2})\pm\frac{23}{630}\left(\text{sign}(\gamma)\sqrt{\eta}\right)\right]\\
  D_{\vartriangle}^{\vartriangle}(\mathbf{K},\mathbf{K}') =&min\left[\frac{4}{15} (x_{5}^2-x_{5}x_{6}+x_{6}^2+y_{6}^{2})\pm\frac{4}{15}\left(\text{sign}(\gamma)\sqrt{\eta}\right)\right]
\end{align}
with
\begin{align}
\eta:=&x_{5}^4 - 2x_{5}^3x_{6} + 3x_{5}^2x_{6}^2 - x_{5}^2y_{6}^2 - 2x_{5}x_{6}^3 - 2x_{5}x_{6}y_{6}^2 +x_{6}^4 + 2x_{6}^2y_{6}^2 + y_{6}^4,\\
\gamma:=&e_{0}^2x_{5}^2y_{6}+2e_{0}^2x_5x_6y_6 - 2e_0^2x_6y_6-2e_{0}^2y_{6}^3+2e_{0}e_{1}x_{5}^3- 6e_0e_1x_5^{2}x_{6}+6e_0e_1x_5x_{6}^2+ \\ &2e_{0}e_{1}x_{5}y_{6}^2-4e_0e_1x_{6}^3 - 4e_0e_1x_6y_{6}^2-e_{1}^{2}x_{5}^2y_{6}- 2e_{1}^2x_{5}x_6y_{6}+2e_{1}^2x_{6}^2y_{6}+ 2e_{1}^2y_{6}^3.
\end{align}
This shows that the obtained values only depend on the geometry of the manipulator.

\end{document}